%% file: main.tex
\documentclass[12pt]{amsart}
\usepackage[foot]{amsaddr}
\usepackage{graphicx} % Required for inserting images
\usepackage{amsfonts}
\usepackage{amsmath}
\usepackage{amssymb}
\usepackage{amsthm}
\usepackage{enumitem}
\usepackage{geometry}
\usepackage{xcolor}
\usepackage{hyperref}
\hypersetup{
colorlinks=true,
linkcolor=cyan,
filecolor=blue,
urlcolor=red,
citecolor=green,
}
\usepackage{cleveref}
\usepackage{tikz-cd}
\usetikzlibrary{shapes.geometric}
\usetikzlibrary{arrows.meta}

\tikzset{
  boxed/.style={
    draw,
    rectangle,
    rounded corners=2pt,
    inner sep=3pt
  },
      plane/.style={
        draw,
        fill=#1,
        thick,
        opacity=0.8,
        general shadow={shadow scale=1, shadow xshift=2pt, shadow yshift=-2pt, opacity=0.3}
    }
}

\numberwithin{equation}{section}

\newtheorem{theorem}{Theorem}[section]
\newtheorem{proposition}{Proposition}[section]
\newtheorem{lemma}{Lemma}[section]
\newtheorem{remark}{Remark}[section]
\newtheorem{definition}{Definition}[section]

\newtheorem{example}{Example}[section]
\newtheorem*{theorem*}{Theorem}

\theoremstyle{remark}

\newcommand{\cjp}[1]{{\color{orange} CJ: \footnotesize #1}}

\title{Deep learning and the rate of approximation by flows}
\author{Jingpu Cheng$^{1}$}
\address{$^{1}$Department of Mathematics, National University of Singapore, 117543, Singapore}
\email{chengjingpu@u.nus.edu}

\author{Qianxiao Li$^{2}$}
\address{$^{2}$Department of Mathematics and Institute for Functional Intelligent Materials, National University of Singapore, 117543, Singapore}
\email{qianxiao@nus.edu.sg}

\author{Ting Lin$^{3}$}
\address{$^{3}$School of Mathematical Sciences, Peking University, 100871, China}
\email{lintingsms@pku.edu.cn}

\author{Zuowei Shen$^{4}$}
\address{$^{4}$Department of Mathematics, National University of Singapore, 117543, Singapore}
\email{matzuows@nus.edu.sg}

\date{\today}

\begin{document}

\begin{abstract}
We investigate the dependence of the approximation capacity of
deep residual networks on its depth in a continuous dynamical systems setting.
This can be formulated as the general problem of quantifying the
minimal time-horizon required to approximate a diffeomorphism by flows driven
by a given family $\mathcal F$ of vector fields.
We show that this minimal time can be identified as a geodesic
distance on a sub-Finsler manifold of diffeomorphisms,
where the local geometry is characterised by a
variational principle involving $\mathcal F$.
This connects the learning efficiency of target relationships
to their compatibility with the learning architectural choice.
Further, the results suggest that the key approximation mechanism
in deep learning, namely the approximation of functions
by composition or dynamics, differs in a fundamental
way from linear approximation theory,
where linear spaces and norm-based rate estimates
are replaced by manifolds and geodesic distances.

\end{abstract}

\maketitle

\tableofcontents

\input{intro_QL.tex}

\input{formulation_and_main_results.tex}

\input{results_1DReLU.tex}

\input{proofs_1DReLU.tex}

\input{general_formulations.tex}

\input{applications.tex}

\bibliographystyle{plain}
\bibliography{ref}

\end{document}

%% file: intro_QL.tex
\section{Introduction}

% \ql{I will write a new intro for clarity of thought. Please read through and combine if necessary}

Deep residual networks (ResNets) form the backbone of modern deep learning architectures,
ranging from convolution networks for image processing~\cite{he2016deep} to transformers
powering large language models \cite{vaswani2017attention}.
At the same time, their success also pose interesting new mathematical questions.
A deep residual network is built from repeated compositions of transformations
\begin{equation}\label{eq:resnet_discrete}
    x(n+1) = x(n) + f_n(x(n)), \qquad f_n \in \mathcal{F}, \qquad n=0, 1, 2,\dots,N-1,
\end{equation}
where $N$ is the depth of the network,
$x(n) \in \mathbb{R}^d$ is the hidden state of the network at layer $n$
and $f_n : \mathbb{R}^d\to\mathbb{R}^d$ is the trainable non-linear transformation applied at each layer,
with $\mathcal{F}$ being the set of possible transformations that is determined
by the architectural choice.
For example, in fully connected ResNets, $\mathcal F$ is the set of all functions that can be represented by a single-layer feedforward network~\cite{touvron2022resmlp,tolstikhin2021mlp,gorishniy2021revisiting}, whereas for transformers, $\mathcal F$ is the set of token-wise networks composed with self-attention maps~\cite{vaswani2017attention,devlin2019bert,dosovitskiy2020image}.
% \ql{give 2 examples, say fully connected, and also attention/transformers.}

With this compositional structure, one can build complex transformations $x(0)\mapsto x(N)$
by increasing the number of layers $N$,
and this is the main way of improving model learning capacity in modern artificial intelligence applications.
These fall into the following two complementary formulations.

\noindent \textbf{(i) Approximation of functions in supervised learning.}
In standard supervised learning tasks, we consider a set of
inputs $X \subset \mathbb R^d$ and a set of outputs $Y \subset \mathbb R^m$.
The goal of learning is to approximate an unknown target function $F:X\to Y$ that maps inputs to outputs.
For example, $X$ can be the set of natural images represented by pixel intensities, and
$Y$ is the corresponding classification of the image into types, such as cats, birds, cars, etc.
$F$ would then be the function that classify a given image into its type,
realising an image recognition tool.

Residual networks tackles this approximation problem by considering the composition
\begin{equation}
    \tilde F : = \beta\circ \varphi_N\circ \alpha,
\end{equation}
where $x(0) \mapsto \varphi_N(x(0)) = x(N)$ is the map defined by the $N$-layer residual dynamics in~\eqref{eq:resnet_discrete},
and $\alpha:X\to \mathbb R^d$ and $\beta:\mathbb R^d\to Y$ are simple maps (e.g., $\beta$ can be affine functions for regression problems and indicator functions
of a half-space for binary classification problems).
The function $\tilde F$ is then used to approximate the target $F$,
with its approximation capacity imparted by the large number of building blocks in $\varphi_N$.

\smallskip
\noindent\textbf{(ii) Approximations of distributions in generative modelling.}
Another highly relevant application relying on a similar dynamical approximation scheme is generative
modelling~\cite{rezende2015variational,ho2020denoising,song2019generative},
where the goal is to learn a model that can generate samples (e.g. images, text, crystal structure, etc.)
from a target data distribution. 
In this case, a predominant approach is to model a transport map that pushes forward a simple reference distribution
$\mu_0$ (e.g., a Gaussian) to a complicated target data distribution $\mu_1$.
More concretely, we aim to find a map $\varphi$ such that
\begin{equation}
    \varphi_{\#}\mu_0 \approx \mu_1.
\end{equation}
In practice, such transport maps are often modelled as the flow map of an ODE realized by the
deep residual structure~\cite{ho2020denoising,song2019generative,song2020score,rezende2015variational} (serving as a discrete approximation) in~\eqref{eq:resnet_discrete}.

In both scenarios discussed above, one requires the compositions of residual layers
to approximate a complicated transformation on $X$.
A primary mathematical question that arises is: how does such compositional structures
build complexity, and how is it different from classical methods --- such as finite elements, splines, and wavelets --- which typically
rely on linear combinations of simpler functions?
A convenient way to study this is the \emph{dynamical systems} approach,
which idealises the compositions in~\eqref{eq:resnet_discrete} as a continuous dynamics
\begin{equation}\label{eq:resnet_continuous}
    \dot{x}(t) = f_t(x(t)), \qquad f_t \in \mathcal{F}, \qquad t\in[0,T].
\end{equation}
This connects the problem of expressivity of deep ResNets to the study
of the approximation or controllability properties of flows of differential equations,
and this approach has yielded a number of insights in deep
learning~\cite{li2022deep,tabuada2022universal,agrachev2022control, ruiz2023neural, cheng2025interpolation},
Most notably, it was shown under fairly weak assumptions on $\mathcal{F}$ that
the family of flow maps of~\eqref{eq:resnet_continuous} is dense in appropriate
spaces of maps from $\mathbb{R}^d$ to itself~\cite{li2022deep,cheng2025interpolation}.
This implies in particular that as long as ResNets are sufficiently deep,
they can be used to solve arbitrary tasks, such as image recognition,
language modelling and reinforcement learning.

However, the follow-up (and perhaps the more important) question remains:
what maps from $\mathbb{R}^d$ to itself are more easily approximated by a
relatively shallow network as opposed to a very deep one?
This connects to the important practical problem of why and when
one should use deep learning over other approximation schemes.
A precise answer to this question is thus important in both theory and practice.

There are familiar analogues in classical approximation theory
for which the corresponding issues are well-understood.
For example, while Weierstrass~\cite{devore1993constructive} showed that polynomials
are dense in the space of continuous functions on the unit interval,
it does not tell us which continuous functions can be approximated well
with a low polynomial order.
The classical theorem due to Jackson~\cite{devore1993constructive} provides an answer to this question,
where it is shown that
\begin{equation}
    \inf_{p \in \mathcal{P}_n}
    \|F(x) - p(x)\|_{C([0,1])}
    \lesssim
    \| F^{(\alpha)} \|_{C([0,1])} n^{-\alpha},
    \qquad
    F \in C^{\alpha}([0,1]),
\end{equation}
where $\mathcal{P}_n$ is the set of polynomials with degree up to and including $n$,
$\alpha$ is a positive integer, and $F^{(\alpha)}$ is the order $\alpha$ derivative of $F$,
which we assume to be $\alpha$-times continuously differentiable.
Jackson's theorem identifies a proper subset of continuous functions ($C^{(\alpha)}$)
as target function spaces, on which we can establish a \emph{rate of approximation}.
In particular, we see that the smoother the target is, the faster the rate of approximation%
\footnote{A converse due to Bernstein~\cite{devore1993constructive} shows that an approximation error
decay rate of $n^{\alpha + \delta}$ ($\delta > 0$) implies $F \in C^{(\alpha)}$,
meaning that only smooth functions can be efficiently approximated.}.
Under the same smoothness condition, a function with a smaller $\|\cdot \|_{C([0,1])}$
admits a lower approximation error. This is a measure of the \emph{complexity}
of a target function under polynomial approximation.
This result has immediate practical consequences:
we now understand that it is precisely the continuously differentiable functions
that admit effective approximations by polynomials, and the complexity is governed
by the size of its derivatives -- the smaller the better.

The purpose of this paper is to study the parallel problem for approximation by flows, as in \cref{eq:resnet_continuous}. In our setting, the cost is not the number of terms in a linear combination of monomials (polynomial degree) but the depth of the residual architecture, idealised as the time horizon $T$ in \cref{eq:resnet_continuous}. We ask: for which target maps can the flow approximate arbitrarily well within a finite time $T$, and how does the minimal such $T$ depend on the target and our choice of architecture $\mathcal F$? We view this minimal time as a notion of complexity of the target.
This differs subtly from the classical Jackson-style question, which fixes a budget (e.g., degree $n$) and bounds the approximation error as a function of that budget. Here we invert the perspective: for a prescribed accuracy (arbitrarily small), we seek the minimal budget $T$ that makes the approximation possible. A simple linear analogy illustrates this distinction.
Consider a Hilbert space $H$ with an orthonormal basis $\{e_i\}_{i\ge 1}$. If one fixes a linear budget $S$ (say, an $\ell^1$ or $\ell^2$ constraint on coefficients), the set of functions that can be approximated arbitrarily well under that budget is just the corresponding ball in $H$ -- a trivial characterisation. By contrast, in the compositional/flow setting the hypothesis class is closed under composition, not addition, and the resulting minimal-budget characterisation is non-trivial.
This is precisely the object we will study.
The problem of minimal time of flows for approximation has been studied in specific settings in previous works.
For instance, in~\cite{ruiz2023neural}, the authors studied the approximation using continuous-time ResNets with ReLU type activations. They provided an upper bound on the time to achieve a certain accuracy for general $L^2$ functions. In~\cite{li2022deep}, the authors provide a quantitative estimate on the minimal reachable time in 1D case for ReLU activations. In~\cite{geshkovski2024measure}, the estimation of time for interpolating a set of measures using flow-based models with self-attention layers is studied.
These works mainly focus on specific architectures and are based on constructive approaches.
In contrast, our purpose is to develop a general framework for studying the minimal time problem for general flow-based models using a geometric viewpoint.
% \ql{I feel we are lacking some references on approx theory (in depth) somewhere, we should either mention some here
% (and talk about why they do not solve this problem)
% and then talk about the differences of our work vs theirs later after we introduce and prove he main theorem}

We now give a informal preview of the main results and insights.
First, we show that the right space of maps on which we can quantify flow approximation
should be a metric space respecting the compositional structure of the problem.
That is, instead of discussing the complexity of approximating \emph{one} map,
we should instead study the complexity of connecting \emph{two} maps by a flow.
The approximation complexity will then be identified with the metric on this space.
Next, this global picture is then supplemented by a local one,
where we show that the metric is realised as the geodesic distance
on this metric space viewed as a sub-Finsler manifold.
Most importantly, the local lengths (Finsler norms) are characterised
in a variational form involving the family of vector fields $\mathcal{F}$.
Together, this gives a concrete picture of the quantitative aspects of flow (ResNet) approximation:
the complexity of connecting two given maps by a continuous deep ResNet
is a distance on a manifold of target maps,
where the local geometry is generated by the expressiveness
of each shallow layer architecture of the deep network.
This framework addresses the basic question of quantitative measures of complexity
for approximation via deep architectures, and holds for in general dimensions
and architecture choices.
In~\cref{sec:applications} we give some examples of this approach for problems
where the precise manifold can be identified
and its associated geodesic distances can be computed or estimated.

%% file: formulation_and_main_results.tex
\section{Problem formulation and main results}
 In this section, we summarize our main results on the geometric framework for the complexity class of flow approximation. We begin by introducing the notations and the precise formulation of the problem, followed by the statement and explanation of the main results.
\subsection{Notations}

In the following, for a finite-dimensional manifold $M$, a vector field $f\in \operatorname{Vec}(M)$
and time $t\ge 0$, we denotes its flow map $\varphi_f^t$ as:
\begin{equation}
   \varphi_f^t: x(0)\to x(t), \quad \dot x(t) = f(x(t)).
\end{equation}
We adopt the following convention whenever possible:
\begin{enumerate}
\item We use $f, g$ to denote the vector fields and $\varphi_f^t$ as the flow map of $f$ at time horizon $t$.
\item We use $\psi, \xi,$ to denote the target mappings from a manifold $M$ to itself.
\item We use $u, v$ to denote tangent vectors.
\end{enumerate}

Let $X$ be a vector space. We write $\overline{X}^{\|\cdot \|}$ as the closure of $X$ under the norm $\|\cdot\|$. Moreover, if there is a natural linear inclusion $\overline{X}^{\|\cdot\|} \hookrightarrow Y$ to some classical space $Y$, we will identify both $X$ and $\overline{X}^{\|\cdot\|}$ as their image in $Y$ under this inclusion.

For $u:M\to\mathbb R^k$, we write $u\in W^{1,\infty}(M)$ if $u\in L^\infty(M)$ and its (weak) derivative satisfies $\nabla u\in L^\infty(M)$.
We equip it with the norm
\begin{equation}
\|u\|_{W^{1,\infty}(M)} := \|u\|_{L^\infty(M)} + \|\nabla u\|_{L^\infty(M)}.
\end{equation}
For a vector field $f\in \operatorname{Vec}(M)$, we interpret $f\in W^{1,\infty}(M)$ component-wise in local charts (equivalently, via any fixed smooth atlas on compact $M$).
We write $\operatorname{Diff}^{W^{1,\infty}}(M)$ for the group of homeomorphisms
$\psi:M\to M$ whose coordinate representations belong to $W^{1,\infty}$, i.e.
\begin{equation}
\psi,\ \psi^{-1} \in W^{1,\infty}(M).
\end{equation}

\subsection{Formulation and main results}
Let $M\subset \mathbb R^d$ be a smooth compact manifold (possibly with boundary). Consider the following control system:
\begin{equation}
    \label{eq:control_system}
    \dot x(t)  = f(x(t), \theta(t)), \quad x(0) = x_0 \in M,\quad \theta(t)\in \Theta,\; t\in [0,T],
\end{equation}
where for each parameter $\theta\in\Theta$ the map $x\mapsto f(x,\theta)$ is a vector field on $M$, i.e., $f(\cdot,\theta)\in \operatorname{Vec}(M)$. We refer to $x(\cdot)$ as the state trajectory and to $\theta(\cdot)$ as the control signal. The associated
family of vector fields (hereafter called a \emph{control family}) is
\begin{equation}
    \mathcal F :=\{\,x\mapsto f(x, \theta)\mid \theta \in \Theta\,\}\subset \operatorname{Vec}(M),
\end{equation}
which is the set of admissible control functions. We assume throughout that $\mathcal F$ is symmetric, i.e., $f\in \mathcal F$ implies $-f\in \mathcal F$, and functions in $\mathcal F$ are uniformly bounded and uniformly Lipschitz on $M$, i.e., there exists $L>0$ such that for all $f\in \mathcal F$,
\begin{equation}
    \|f\|_{L^\infty(M)}:=\sup_{x\in M} \|f(x)\|\le L,\quad \text{and}\quad \operatorname{Lip}(f):=\sup_{x,y\in M, x\neq y} \frac{\|f(x)-f(y)\|}{\|x-y\|}\le L.
\end{equation}
% \ql{We should assume that $\mathcal F$ is symmetric and sufficiently uniformly bounded here so we
% don't miss this later? Without the latter, the minimal time would not be well-defined.}

For a given horizon $T>0$, we denote by $\mathcal A_{\mathcal F}(T)$ the class of all flow maps generated by time-splitting the dynamics in $\mathcal F$ over
a total duration $T$:
\begin{equation}
        \mathcal A_{\mathcal F}(T):=\big\{\,\varphi_{f_m}^{t_m}\circ \cdots\circ \varphi_{f_1}^{t_1}\;\big|\; m\in\mathbb N,\ f_i\in\mathcal F,\ t_i>0,\ \sum_{i=1}^m t_i = T\,\big\},
\end{equation}

It is now known that under mild conditions on $\mathcal F$, the set
\begin{equation}
    \mathcal A_{\mathcal F}:=\bigcup_{T>0} \mathcal A_{\mathcal F}(T)
\end{equation}
is dense in $C(M,M)$ (the class of continuous maps $M\to M$)
with respect to natural topologies (e.g., uniform on compacta or $L^p_{\operatorname{loc}}$).
For example, in~\cite{li2022deep,tabuada2022universal,ruiz2023neural,agrachev2022control, cheng2025interpolation} it is shown that when $M=\mathbb R^{d}$ and $\mathcal F$ is affine-invariant and contains a non-linear vector field, then $\mathcal A_{\mathcal F}$ is dense in $L_{\operatorname{loc}}^p(\mathbb R^d, \mathbb R^d)$ for any $p\in[1,\infty)$.
In~\cite{duan2025minimal}, it is also shown that under similar condition,
$\mathcal A_{\mathcal F}$ is dense in $C(M,M)$ topology in the set of orientation preserving diffeomorphisms over $M$.
This means in particular that given sufficiently long time horizons,
maps generated by flows of reasonable control families $\mathcal F$ can
learn arbitrary relationships to any desired accuracy.

The density results give basic guarantees on the viability of using
deep learning for complex tasks.
However, quantitative rate estimates for such approximation schemes is much less explored.
Here, we focus on the minimal time needed to approximate a given target map.
Specifically, for $\psi\in \operatorname{Diff}(M)$ we define
\begin{equation}
    \label{eq:complexity_def}
    \mathbf C_{\mathcal F}(\psi):=\inf\{T>0\mid \psi\in \overline{\mathcal A_{\mathcal F}(T)}\}.
\end{equation}
Here, $\overline{\mathcal A_{\mathcal F}(T)}$ denotes the closure of $\mathcal A_{\mathcal F}(T)$ in the $C(M,M)$ topology.
Thus, $\mathbf C_{\mathcal F}(\psi)$ is the minimal time for which the flows-maps
generated by \eqref{eq:control_system} can approximate $\psi$ arbitrarily well.
In practice, this quantity corresponds to how deep a residual network with
layer architecture $\mathcal F$ needs to be to learn a relationship $\psi$.
We take as our target space the class of diffeomorphisms that are reachable in finite time:
\begin{equation}
    \label{eq:target_space}
    \mathcal T :=\{\psi\in \operatorname{Diff}(M)\mid \mathbf C_{\mathcal F}(\psi)<\infty\}.
\end{equation}
Our goal is to characterize, or at least provide estimates for, $\mathbf C_{\mathcal F}(\psi)$ for $\psi\in \mathcal T$.

The quantity $\mathbf{C}_{\mathcal F}(\psi)$ serves as a notion of complexity for maps in $\mathcal T$.
Approximation complexity in classical linear approximation theory is typically
characterized by certain norms (e.g. Sobolev norms, Besov norms),
where closure under linear combinations drives the analysis.
In contrast, our hypothesis class $\mathcal A_{\mathcal F}$ is generally
not closed under linear combinations.
Instead, it is closed under compositions.
That is, given $\psi,\xi$ in the class, we typically have $\psi\circ \xi\in\mathcal A_{\mathcal F}$, but not $\alpha \psi+\beta \xi$ for all $\alpha, \beta\in \mathbb R$.
This leads to fundamentally different approximation mechanisms.  For example, the approximation of the identity function is trivial using the hypothesis space of flows/ResNets.
However, the difference $0 = \operatorname{Id}-\operatorname{Id}$ is not easy to uniformly approximate using flows/ResNets, at least in 1D~\cite{li2022deep}.
In other words, the approximation error is not compatible with linear combinations of target functions.

A key observation is that approximation error/complexity should respect the algebraic structure of the hypothesis space. For a linear space $\mathcal H$
closed under linear combinations, one has the triangle inequality
\begin{equation}
    \inf_{h\in\mathcal H}\|h-(\psi+\xi)\|\le \inf_{h\in \mathcal H}\|h-\psi\|+\inf_{h\in \mathcal H} \|h-\xi\|.
\end{equation}
That is, the approximation error of $\psi+\xi$ is controlled by
the sum of the individual errors, compatible with linear structure.
In our compositional setting, an analogous triangle inequality holds with respect to composition:
\begin{equation}
    \label{eq:compositional_triangle_inequality}
    \mathbf C_{\mathcal F}(\psi\circ \xi)\le \mathbf C_{\mathcal F}(\psi)+\mathbf C_{\mathcal F}(\xi).
\end{equation}
In words, the minimal time to approximate $\psi\circ \xi$ is at most the sum of the times for $\psi$ and for $\xi$.

Here we arrive at an important observation:
$\mathbf C_{\mathcal F}(\psi)$ \emph{cannot} be any kind of a norm of $\psi$,
for otherwise the usual triangle inequality would imply the ease of approximation
of $\alpha \psi + \beta \xi$ provided the ease of approximation of $\psi,\xi$, which is false.
This then hints at the next natural possibility: an appropriate metric
satisfying the compositional triangle inequality should describe flow
approximation complexity.
This turns out to be the correct approach. Concretely, we extend $\mathbf C_{\mathcal F}(\cdot)$  to a distance on $\mathcal T$. For $\psi_1,\psi_2\in\mathcal T$ define
\begin{equation}
    \label{eq:d_f_def}
    d_{\mathcal F}(\psi_1, \psi_2): = \inf\{T>0\mid \psi_1\in \overline{\mathcal A_{\mathcal F}(T)\circ \psi_2}\},
\end{equation}
that is, $d_{\mathcal F}(\psi_1,\psi_2)$ is the minimal time to connect $\psi_2$ to $\psi_1$ up to arbitrary accuracy.
Recall that since $\mathcal F$ is symmetric ($f\in \mathcal F$ implies $-f\in \mathcal F$),
$d_{\mathcal F}$ is a metric on $\mathcal T$.
This metric can also be generalized to any pairs $\psi_1,\psi_2\in \operatorname{Diff}(M)$,
by allowing $d_{\mathcal F}(\psi_1,\psi_2)$ to take value $+\infty$ when $\psi_1$ is not reachable from $\psi_2$.
 With this metric, we have
\begin{equation*}
    \mathbf C_{\mathcal F}(\psi) = d_{\mathcal F}(\psi, \operatorname{Id}),\quad \text{where } \operatorname{Id} \text{ is the identity map.}
\end{equation*}

The critical question is then: how does $\mathcal{F}$
(the complexity of the control family, or in deep learning, the architectural choice of each layer)
determine this metric?
To answer this question, we develop a geometric viewpoint,
where we consider $\mathcal T\subseteq \mathcal M\subset \operatorname{Diff}^0(M)$ 
as a subset of some $\mathcal M$ with Banach manifold structure (e.g., $\operatorname{Diff}^{W^{1,\infty}}(M)$ for suitable $M$), and
endow $\mathcal M$ with a sub-Finsler structure --
a generalisation of a sub-Riemannian structure~\cite{sontag2013mathematical,agrachev2019comprehensive,arguillere2020sub}
where a norm on each tangent space replaces the sub-Riemannian inner product.
Crucially, as in sub-Riemannian manifolds one can still define a geodesic distance,
which we show is exactly $d_{\mathcal F}$.
This geodesic distance is called the Carnot-Carathéodory (CC) distance in the literature of sub-Riemannian geometry.
For simplicity, we adopt the name ``geodesic distance'' in the rest of this paper,
keeping in mind that in the most general case this should be understood as the CC distance.
Remarkably, the local sub-Finsler norm
admits a variational characterization tied to $\mathcal F$, yielding a new
lens on approximation complexity for flow-based models. 
In the following, for $s>0$, we define
\begin{equation}
    \label{eq:CH_def}
\mathbf{CH}_s(\mathcal F)
:= \Big\{\sum_{i=1}^N a_i f_i : f_i\in \mathcal F,\ \sum_{i=1}^N |a_i| \le s \Big\}.
\end{equation}
% This corresponds to the $s$-ball in the norm $\|\cdot\|_{\mathcal F}$
% defined in~\cref{eq:norm_F} below.

% --- Atomic norm generated by the control family ---
% For $v\in \operatorname{span}\mathcal F$, define the (atomic) norm
% \begin{equation}
%     \label{eq:norm_F}
%     \|v\|_{\mathcal F}:=\inf\Big\{\sum_{i=1}^N |a_i|:\ v=\sum_{i=1}^N a_i f_i,\ f_i\in\mathcal F\Big\},\qquad
%     X_{\mathcal F}:=\overline{\operatorname{span}\mathcal F}^{\ \|\cdot\|_{\mathcal F}}.
% \end{equation}
% Equivalently, $\mathbf{CH}_s(\mathcal F)$ in~\eqref{eq:CH_def} is the $s$-ball of $\|\cdot\|_{\mathcal F}$.

Then, the key intuition of the sub-Finsler structure is as follows.
An infinitesimal change of a map $\psi$ can be produced by composing it with a short-time flow:
for a vector field $v\in \operatorname{Vec}(M)$ and small $\varepsilon>0$,
\[
    (\varphi_v^{\varepsilon}\circ \psi)(x)=\psi(x)+\varepsilon\,v(\psi(x))+o(\varepsilon).
\]
Thus the instantaneous velocity at $\psi$ has the form $u=v\circ \psi$.
If $v$ belongs to the $s$-scaled convex hull $\mathbf{CH}_s(\mathcal F)$, then (by time-rescaling)
this short move can be realised using controls from $\mathcal F$ with time proportional to $s$.
This motivates measuring the ``size'' of an infinitesimal velocity $u$ at $\psi$ by the smallest such $s$:
% \cjp{changed the closure from $C^0$ sense to the atomic norm sense. Under our assumptions, if in $W^{1,\infty}$ base manifold, they should be the same. But in $C^1$ case maybe slightly different.}
\begin{equation}
    \label{eq:local_norm_variational}
    \|u\|_{\psi}=\inf\{\,s>0\mid u\in \overline{\mathbf{CH}_s(\mathcal F)}^{C^0}\circ \psi\,\},
\end{equation}
with the convention $\|u\|_{\psi}=+\infty$ when $u$ is not of the form $v\circ\psi$.
Here, the closure is taken in the $C^0(M, \mathbb R^d)$ topology.
\Cref{sec:banach_finsler_geometry,sec:main_results_general} provides the rigorous presentation of this notion,
where this local quantity is defined intrinsically (as a sub-Finsler fiber norm on $\mathcal M$)
and proved to be well-posed.

Integrating this local
norm along suitable paths in $\mathcal M$ provides computable upper bounds on
$d_{\mathcal F}$, and hence on $\mathbf C_{\mathcal F}$. 
This result is summarized in the following theorem, which is the main result of this paper.
The geometric framework built in the main theorem is also illustrated in the diagram~\Cref{fig:sub_finsler_geometry}.

\begin{theorem}
\label{thm:main}
Suppose $(\mathcal M, \mathcal F)$ is a compatible pair, as defined in Definition~\ref{def:compatible-pair}. Then, the flow-complexity metric $d_{\mathcal F}$
coincides with the geodesic distance induced by the local norm~\eqref{eq:local_norm_variational}.
In particular, for all $\psi_1,\psi_2\in\mathcal M$,
\begin{equation}
    \label{eq:geodesic_distance}
    d_{\mathcal F}(\psi_1,\psi_2)
    =\inf\Big\{\int_0^1 \|\dot\gamma(t)\|_{\gamma(t)}\,dt\ \Big|\ \gamma(0)=\psi_1,\ \gamma(1)=\psi_2\Big\}.
\end{equation}
The proof is given in~\Cref{sec:main_results_general}.
\end{theorem}

\begin{remark}
The norm in~\eqref{eq:local_norm_variational} is also called the atomic norm or Minkowski functional of the set $\mathbf{CH}_s(\mathcal F)\circ \psi$, and is a concept widely applied in convex analysis and inverse problems~\cite{Chandrasekaran2012-FoCM,Rockafellar1970-ConvexAnalysis}.
\end{remark}

\begin{figure}[htbp]
    \centering
    \includegraphics[width=\linewidth]{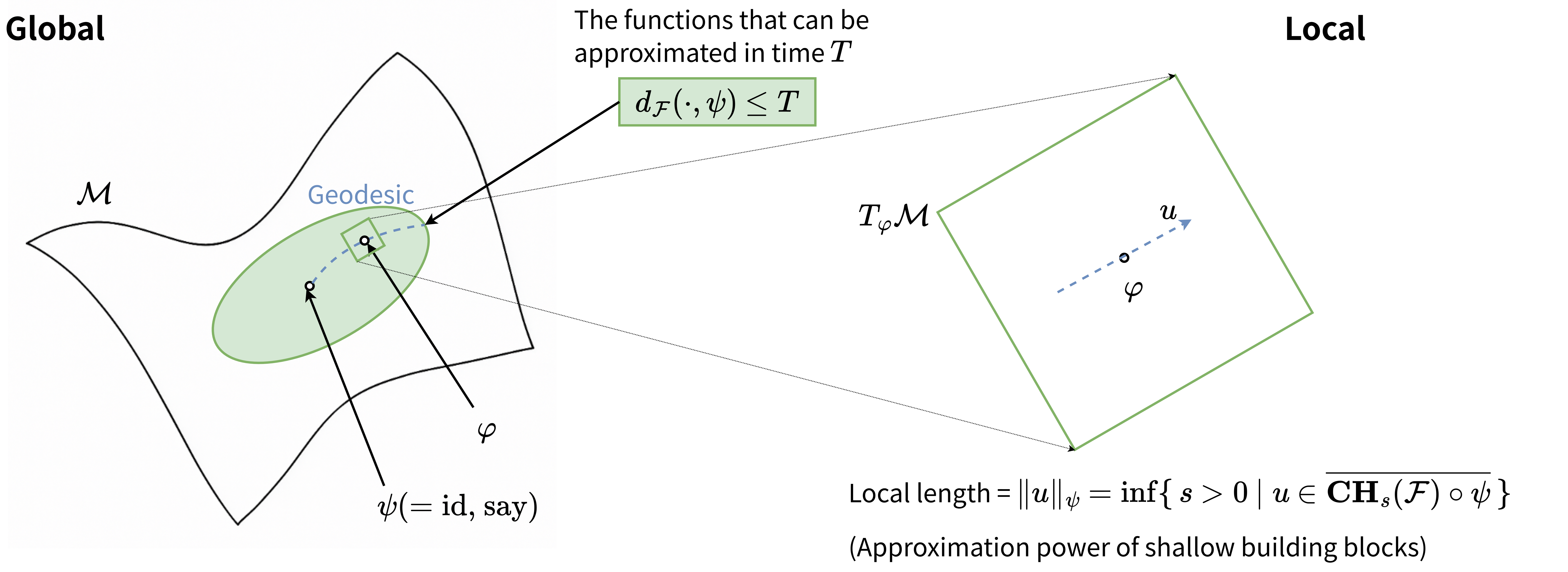}
    \caption{Diagram illustrating the connection built in the main result
    between the approximation complexity by flows and sub-Finsler geometry}
    \label{fig:sub_finsler_geometry}
\end{figure}

% \ql{3.13 is also sometimes called the ``atomic norm''. Some remarks about this should be discussed.}

The result above establishes a geometric framework for analyzing the complexity $\mathbf C_{\mathcal F}(\psi)$ of approximating a target map $\psi$ by flows of \eqref{eq:control_system}. Given $\psi\in\mathcal M$, one can estimate $\mathbf C_{\mathcal F}(\psi)$ by constructing a path $\gamma$ in $\mathcal M$ connecting $\operatorname{Id}$ to $\psi$, and integrating the local norm $\|\cdot\|_{\gamma(t)}$ 
% \ql{D still here, check notation throughout again}
along $\gamma$. The local norm itself is characterized by a variational principle involving the convex hull of the control family $\mathcal F$, which relates to the approximation results for shallow neural networks and is relatively well-studied~\cite{barron2002universal, ma2022barron,siegel2020approximation,siegel2023characterization}.

On the mathematical side, this result gives a quantitative
characterization of the rate of approximation of diffeomorphisms by flows,
where the vector field at each time is constrained to a family $\mathcal F$.
On the application side, it also addresses, in an idealised setting,
a key question of deep learning, namely which target relationships are best
learned using deep as opposed to shallow networks. Notably, this has implications in the two approximation problems mentioned in the introduction.

\subsection{Implications for approximation theory and compositional models}
\label{subsec:approx-theory-interpretation}
Approximation theory is fundamentally about metrics on functions. One needs a metric to quantify the approximation error (how close an approximant is to a target), and one also needs a notion of complexity of the target relative to a hypothesis class (how hard it is to approximate). A central feature that distinguishes deep neural networks from both shallow networks and classical linear approximation schemes is that complexity is built through function compositions, rather than linear combinations. From this viewpoint, a key question for an approximation theory of deep networks is: how does composition improve approximation, and how should we measure the corresponding complexity in a way that is compatible with composition? Without a good understanding of this question, it is difficult to understand the essential benefits of depth in deep learning.

Our first contribution to this question is the introduction of the distance $d_{\mathcal F}(\cdot,\cdot)$ on the target space $\mathcal T$ as a measure of compositional closeness. As shown in~\Cref{eq:compositional_triangle_inequality}, the complexity functional
$C_{\mathcal F}(\psi)=d_{\mathcal F}(\psi,\operatorname{Id})$ satisfies a triangle inequality with respect to composition. This is precisely the compatibility one expects for a compositional hypothesis space: if $\psi$ can be well-approximated by composing two intermediate maps, then the corresponding complexity should be bounded by the sum of the intermediate complexities.

Moreover, \Cref{thm:main} provides a quantitative characterization of $d_{\mathcal F}$ via a sub-Finsler geometry on $\mathcal T$ induced by the layer class $\mathcal F$. In particular, the local norm at each point $\psi$ is given by a variational principle involving the (scaled) convex hull of $\mathcal F$, which is closely related to approximation properties of shallow networks. This turns $d_{\mathcal F}$ from an abstract definition into a quantity that is, at least in principle, computable or estimable: one may identify the target class $\mathcal T$ and derive bounds on $d_{\mathcal F}$ by analyzing the shallow approximation power of $\mathcal F$.

This framework has several implications for approximation using deep neural networks. Recall that a residual network with depth $N$ builds an input--output map $\mathbb{R}^d\to\mathbb{R}^d$, $x_0\mapsto x_N$, via
\begin{equation}
x_{k+1}=x_k+\Delta t\, f_k(x_k),\qquad f_k\in \mathcal F,\qquad k=0,\dots,N-1,
\label{eq:resnet-discrete-interpretation}
\end{equation}
where $\mathcal F$ is the class of transformations realizable by one layer. Letting $N\to\infty$ while keeping the total horizon $T=N\Delta t$ fixed yields the continuous-time idealization
\begin{equation}
\dot x(t)=f_t(x(t)),\qquad f_t\in \mathcal F,\qquad t\in[0,T],
\label{eq:resnet-ode-interpretation}
\end{equation}
whose flow map $x(0)\mapsto x(T)$ represents the network input--output relation. In this setting, the horizon $T$ plays the role of an idealized notion of depth.

The condition $d_{\mathcal F}(\psi,\psi_0)<\infty$ has a direct approximation-theoretic meaning: it asserts that the target map $\psi$ can be approximated arbitrarily well by composing $\psi_0$ with a finite-time flow generated by controls in $\mathcal F$. More precisely,
\begin{equation}
\label{eq:finite-distance-epsilon}
d_{\mathcal F}(\psi,\psi_0)<\infty
\quad\Longleftrightarrow\quad
\exists\,T<\infty\ \text{such that}\ 
\forall \varepsilon>0,\ \exists\,\phi_\varepsilon\in\mathcal A_{\mathcal F}(T)\ \text{with}\ 
\|\phi_\varepsilon\circ\psi_0-\psi\|_{C^0(M)}<\varepsilon .
\end{equation}
Thus, within the continuous-time idealization, lying in the same $d_{\mathcal F}$-component is exactly the statement that one can reach (up to arbitrary accuracy) in finite depth by composing layers.

In supervised learning, one often models a target map $F:K\subset\mathbb{R}^\ell\to\mathbb{R}^k$ as
\begin{equation}
F \approx \beta \circ \psi\circ \alpha,
\label{eq:supervised-factorization}
\end{equation}
where $\psi\in\mathcal A_{\mathcal F}$ is a representation map generated by the deep dynamics, while
$\alpha$ and $\beta$ belong to comparatively simple input/output classes (e.g.\ linear maps or shallow readouts).
If a representation $\psi$ that makes \eqref{eq:supervised-factorization} feasible lies in $\mathcal T$, then
\eqref{eq:finite-distance-epsilon} guarantees that $\psi$ can be approximated within finite idealized depth.
In particular, applying our framework for higher-dimensional flows, we show in~\Cref{sec:finite_time_approximation}
that for the introduced ReLU-based control family considered there, for any compact $K$ and any sufficiently smooth target
$F$, there exist linear maps $\alpha,\beta$ and a horizon $T>0$ such that for every $\varepsilon>0$ one can find
$\phi_\varepsilon\in\mathcal A_{\mathcal F}(T)$ satisfying
\begin{equation}
\|\beta\circ \phi_\varepsilon\circ \alpha - F\|_{C^0(K)}<\varepsilon.
\end{equation}
Importantly, the same \emph{uniform} time horizon $T$ works for arbitrarily small $\varepsilon$.
This contrasts with approximation results in which the time needed to approximate a general target function may
diverge as the required accuracy increases such as in~\cite{cheng2025interpolation,li2022deep}. See
\Cref{sec:finite_time_approximation} for further discussion.

In flow-based generative modelling, given a reference distribution $\nu$ on $M$ (a prior) and a target
distribution $\mu$, one seeks a transport map $\psi$ such that $\psi_\#\nu=\mu$. In general there may be
infinitely many such maps, but their approximation complexity with respect to the model class
$\mathcal A_{\mathcal F}$ can be very different. Our framework provides a principled way to compare the
complexity of different constructions: one may compare $d_{\mathcal F}(\psi_1,\operatorname{Id})$ and
$d_{\mathcal F}(\psi_2,\operatorname{Id})$ for candidate transport maps $\psi_1,\psi_2$ realizing the same
pushforward. This is particularly relevant for flow matching~\cite{lipman2022flowmatching}, where multiple
vector fields (and hence multiple flow maps) can be used to connect a given prior to the same target
distribution. The metric $d_{\mathcal F}$ offers a theoretical tool to quantify which constructions are more
compatible with the chosen layer class, and therefore potentially easier to approximate in practice.

In the following sections, we demonstrate how to apply this framework to obtain estimates for $d_{\mathcal F}$
through explicit examples. In the one-dimensional case $M=[0,1]$, when $\mathcal F$ is generated by neural
networks with activation $\operatorname{ReLU}(x):=\max\{0,x\}$, we can compute $d_{\mathcal F}$ in closed form
for any pair of diffeomorphisms $\psi_1,\psi_2\in\mathcal T$ (see~\Cref{thm:1D_relu}):
\begin{equation}
d_{\mathcal F}(\psi_1,\psi_2)=\|\ln \psi_1^{\prime}-\ln \psi_2^\prime \|_{\operatorname{TV}([0,1])}.
\end{equation}
Notice that the right-hand side is not a norm of $\psi_1-\psi_2$, highlighting a fundamental difference from
classical linear approximation theory. Furthermore, if we regard $\psi_1,\psi_2$ as cumulative distribution
functions, then the right-hand side coincides with the $L^1$ distance between their score functions, a quantity
frequently appearing in flow-based generative models. We give more details in~\Cref{sec:1d_relu} and
\Cref{sec:applications}.

We also provide a two-dimensional example in~\Cref{sec:applications}, where $M=S^1\subset\mathbb{R}^2$ is the
unit circle and the control family is given by a two-dimensional ReLU control family equipped with a
layer-normalization-type constraint. In this case, while an exact closed form for $d_{\mathcal F}$ is difficult to
obtain, we can still derive explicit estimates for pairs $\psi_1,\psi_2$ with $C^3$ regularity.

Moreover, viewing approximation complexity through $d_{\mathcal F}$ suggests a practical guideline for model
design. While deep models are often initialized at the identity map $\operatorname{Id}$, the metric perspective
indicates that the initialization need not be $\operatorname{Id}$: the idealized depth required to reach a target
$\psi$ from an initial map $\psi_0$ is $d_{\mathcal F}(\psi,\psi_0)$. If domain knowledge can be used to select an
initialization $\psi_0$ lying in the same component as $\psi$ and closer in $d_{\mathcal F}$, then the
compositional effort needed for approximation can be substantially reduced. This perspective, together with
the connected-component obstruction, will be discussed further in~\Cref{sec:new_ways_build_model}.

Finally, although our framework is developed for the continuous-time idealization, the key insights also apply
to the discrete-time setting. The dynamical formulation preserves the compositional structure of the hypothesis
space, which is the key feature of deep neural networks. By applying suitable time-discretization schemes, one
can translate estimates for $d_{\mathcal F}$ into corresponding error and complexity bounds for the discrete-time
approximation complexity $\mathbf C_{\mathcal F}(\psi)$, which is closer to practical implementations.

%% file: results_1DReLU.tex
\section{The rate of approximation by flows: one dimensional ReLU networks}
\label{sec:1d_relu}
We begin by illustrating the basic ideas and general philosophies of our approach using an example,
where the control family is a set of one-dimensional neural network functions with the ReLU activation function.
This is a simple case where many computations have closed form.
At the same time, it extends the previous investigations of this example
in~\cite{li2022deep}.
Here we outline the main findings with the detailed proofs and computations deferred to~\cref{sec:1d_relu_proofs}.

Let $M: = [0, 1]$ and consider the following control family in $\operatorname{Vec}(M)$:
\begin{equation}
    \label{eq:1d_relu_family}
    \mathcal F:= \{f(x) =\sum_{i=1}^2 w_i \operatorname{ReLU}(a_i x+b_i)\mid  f(0)=f(1)=0, \sum_{i=1}^2 |w_i a_i| \le 1 \}\subset \operatorname{Vec}(M),
\end{equation}
% \ql{I prefer setting $n=2$ or say that $n\geq 2$ is fixed,
% here because it highlights taht $n$ is finite and $\mathcal F$ is not universal.}
which is the set of shallow neural networks with bounded weights sum and fixed values at $0$ and $1$, with the activation function $\operatorname{ReLU}(x): = \max\{0, x\}$.
In the language of machine learning, each layer in this deep ResNet
consists of a width-2 ReLU-activated fully connected neural network, operating in one hidden dimension.
Observe that for $f \in \mathcal F$, its flow map $\varphi_t^f$ has the following property:
\begin{center}
$\varphi_t^f$ is an increasing function from $[0, 1]$ to $[0, 1]$ with $\varphi_t^f(0) = 0$ and $\varphi_t^f(1) = 1$.
\end{center}

Now, we introduce our target function space as in~\eqref{eq:target_space}, i.e. the set of all functions that can be uniformly approximated by the flow maps in $\mathcal A_{\mathcal F}$ within finite time.
In this example, we have a clearer description of the target function space: any function in $\mathcal M$ is a non-decreasing function from $[0, 1]$ to $[0, 1]$ with fixed endpoints $0$ and $1$.

% where $\mathbf C_{\mathcal F}(\psi)$ is the minimal reachable time complexity for $\psi$ defined in~\eqref{eq:complexity_def}.

We are then interested in characterizing the complexity measure $\mathbf C_{\mathcal F}(\psi)$ for $\psi\in \mathcal M$, as well as a closed-form characterization of $\mathcal M$. In~\cite{li2022deep}, an estimation of $\mathbf{C}_{\mathcal F}(\psi)$ is provided as following:
\begin{proposition}[Proposition 4.8 in \cite{li2022deep}]
    \label{thm:1D_relu}
    If $\psi$ is a piece-wise smooth increasing function with $\psi(0)=0$, $\psi(1)=1$, and $\|\ln \psi^\prime\|_{\operatorname{TV}([0,1])}<\infty$, then $\psi \in \mathcal M$. In particular,
    \begin{equation}
    \mathbf{C}_{\mathcal F}(\psi) \le \left\|\ln \psi^{\prime}\right\|_{\mathrm{TV}([0,1])}+ |\ln \psi^\prime(0)|+|\ln \psi^\prime(1)|.
    \end{equation}
Here, $\|\cdot\|_{\mathrm{TV}([0,1])}$ denotes the total variation of a function on $[0,1]$, defined as
\begin{equation}
    \|g\|_{\mathrm{TV}([0,1])} := \sup \{\sum_{i=1}^{n-1} |g(x_{i+1}) - g(x_i)| \;\Big| \; n\in \mathbb N,\; 0\le x_1 < x_2 < \cdots < x_n \le 1\}.
\end{equation}
\end{proposition}

\Cref{thm:1D_relu} provides an upper bound of the complexity measure $\mathbf C_{\mathcal F}(\psi)$ using a constructive approach~\cite{li2022deep},
but it was not shown there whether this bound is tight, or if an exact
formula for $\mathbf C_{\mathcal F}(\psi)$ can be obtained.

A key observation is that $\mathcal M$ is closed under function composition.
That is, for any $\psi_1, \psi_2 \in \mathcal M$, we have $\psi_1 \circ \psi_2 \in \mathcal M$.
% \ql{We should mention it here.}
This inherent algebraic structure parallels
the linear structure in classical approximation theory, and thus motivates us to consider the compatibility of the complexity measure $\mathbf C_{\mathcal F}(\psi)$ with respect to function composition. Specifically, for any $\psi_1, \psi_2 \in \mathcal M$, we expect the compositional triangle inequality in~\eqref{eq:compositional_triangle_inequality} to hold.
From a dynamical system perspective, the complexity $\mathbf C_{\mathcal F}(\psi)$ can be interpreted as the minimal time needed to steer the system from the identity map $\operatorname{Id}$ to the target function $\psi$ using the ReLU control family~\eqref{eq:1d_relu_family}. With this interpretation, the compositional triangle inequality in~\eqref{eq:compositional_triangle_inequality} can be understood as follows: to reach $\psi_1 \circ \psi_2$ from $\operatorname{Id}$, one can first reach $\psi_2$ from $\operatorname{Id}$, and then reach $\psi_1 \circ \psi_2$ from $\psi_2$. The total time taken is the sum of the two individual times, which gives an upper bound for $\mathbf C_{\mathcal F}(\psi_1 \circ \psi_2)$.

This observation indicates that $\mathbf C_{\mathcal F}(\psi)$ should be considered in a pairwise manner, rather than individually for each $\psi$. In other words, instead of focusing solely on the complexity of reaching a single target function from the identity, we should investigate the complexity of transitioning between any two target functions in $\mathcal M$, leading to the definition of a distance function on $\mathcal M$:
\begin{equation}
        \label{eq:distance}
        d_{\mathcal F}(\psi_1, \psi_2): = \inf\{T>0\mid \psi_1\in \overline{\mathcal A_{\mathcal F}(T)\circ \psi_2}\}.
    \end{equation}
Here $$\overline{\mathcal A_{\mathcal F}(T) \circ \psi_2} := \{ \xi \in C([0,1]) ~|~ \inf_{\phi \in \mathcal A_{\mathcal F}(T)} \| \xi - \phi\circ \psi_2 \|_{C([0,1])} = 0\}.$$

We can easily verify that $d_{\mathcal F}(\cdot, \cdot)$ is indeed a metric on $\mathcal M$, i.e. it is positive definite, symmetric, and satisfies the triangle inequality. Furthermore, for any $\psi_1, \psi_2\in \mathcal M$, it can be directly checked that:
\begin{equation}
    \mathbf{C}_{\mathcal F}(f_2\circ f_1)= d_{\mathcal F}(f_2\circ f_1, \operatorname{Id})
    \le d_{\mathcal F}(f_2\circ f_1, f_1)+d_{\mathcal F}(f_1, \operatorname{Id})=\mathbf{C}_{\mathcal F}(f_2)+\mathbf{C}_{\mathcal F}(f_1),
\end{equation}
which is the triangular inequality with respect to function compositions.
With this additional structure identified,
$\mathcal M$ is not only a set of functions, but also a metric space equipped with the metric $d_{\mathcal F}(\cdot, \cdot)$ that is compatible with function compositions, an inherent algebraic structure of $\mathcal M$.

The next question is then: how does this metric depend on the control family $\mathcal F$?
Unravelling this relation is crucial for understanding
what maps on the unit interval are more amenable to approximation with a moderate time horizon.

We begin with a simple case by fixing a function $\psi \in \mathcal M$ and considering another function $\xi \in \mathcal M$ that is very close to $\psi$ in the metric $d_{\mathcal F}$.
How can we connect $\psi$ and $\xi$ with a flow?
For a very small scale $\tau>0$ and $ \alpha_1, \cdots, \alpha_n>0$,
consider the flow map:
\begin{equation}
    \varphi_{f_n}^{\alpha_n \tau}\circ \cdots \circ \varphi_{f_1}^{\alpha_1 \tau} \in \mathcal A_{\mathcal F}, \text{ with } f_i \in \mathcal F,
\end{equation}
which admits the first-order approximation:
\begin{equation}
    \varphi_{f_n}^{\alpha_n t}\circ \cdots \circ \varphi_{f_1}^{\alpha_1 t}(x) = x + \tau \sum_{i=1}^n \alpha_i f_i(x) + o(\tau).
\end{equation}
Therefore, the corresponding function in $\mathcal A_{\mathcal F}\circ \psi$ is close to
\begin{equation}
    \psi + \tau \sum_{i=1}^n \alpha_i f_i\circ \psi + o(\tau).
\end{equation}
If $u:=\frac{\xi-\psi}{\tau}$ is close to $\sum_{i=1}^n \alpha_i f_i\circ \psi$ for some $f_i \in \mathcal F$ and $\alpha_i>0$, then the map $\xi = \psi +\tau u$ is approximately reachable from $\psi$ within time $\tau \sum_{i=1}^n \alpha_i$.
Therefore, given a perturbation function $u$ and small $t>0$, the distance $d_{\mathcal F}(\psi, \psi+\tau u)$ is approximately given by the product of $\tau$ and the minimal weight sum needed to approximate $u$ using functions in $\mathcal F$ composed with $\xi$. For $s>0$, if we define
\begin{equation}
    \mathbf{CH}_s(\mathcal F):=
    \left\{\sum_{i=1}^n \alpha_i f_i \mid f_i \in \mathcal F, \alpha_i\ge 0, \sum_{i=1}^n \alpha_i \le s
    \right\}
\end{equation}
as the $s$-scaled convex hull of $\mathcal F$ in $\operatorname{Vec}(M)$. 
% \begin{equation}
%     \overline{\mathbf{CH}_s(\mathcal F)}^{\|\cdot \|_{\mathcal F}}: =     \left\{\sum_{i=1}^\infty \alpha_i f_i \mid f_i \in \mathcal F, \alpha_i\ge 0, \sum_{i=1}^\infty \alpha_i \le s
%     \right\}
% \end{equation}
% as the closure of $\mathbf{CH}_s(\mathcal F)$ under the atomic norm defined by $\mathcal F$ embeded in $\operatorname{Vec}(M)$.
Then, the local norm
\begin{equation}
    \label{eq:local_norm_relu}
    \|u\|_{\psi}:=\inf\{\,s>0\mid u\in \overline{\mathbf{CH}_s(\mathcal F)}^{C^0}\circ \psi\,\},
\end{equation}
represents the minimal weight sum needed to approximate $u$ using functions in $\mathcal F$ composed with $\psi$. 

Then, based the above local analysis, for a perturbation function $u$ and small $\tau>0$, the function $\psi + \tau u$ is approximately reachable from $\psi$ within time $\tau \|u\|_{\psi}$.

% More rigorously, we have the following relation:
% \begin{equation}
%     \label{eq:local_distance_relu}
%     \lim_{\tau\to 0} \frac{d_{\mathcal F}(\psi, \psi+\tau u)}{\tau} = \|u\|_{\psi}.
% \end{equation}
% Here, we set $d_{\mathcal F}(\psi, \psi+\tau u)$ as infinity if $\psi+\tau u \notin \mathcal M$.

% The expressions~\eqref{eq:local_norm_relu} and~\eqref{eq:local_distance_relu} reveals a relation between the distance $d_{\mathcal F}(\psi, \cdot)$ and the control family $\mathcal F$: the norm induced by the infinitesimal of $d_{\mathcal F}(\cdot, \cdot)$ at each point $\psi \in \mathcal M$ can be characterized by the complexity of approximating perturbation functions using the convex hull of the control family composed with $\psi$. This variational characterization relates to the approximation theory of shallow neural networks, which is relatively well-studied~\cite{barron2002universal, ma2022barron,siegel2020approximation,siegel2023characterization}.
% \ql{This last part loses focus, move this to the general theory. In fact, I think the 4.10 is not needed here.}

In the one-dimensional ReLU example, we can explicitly compute $\|\cdot\|_{\psi}$ by investigating the corresponding variational problem in~\eqref{eq:local_norm_relu}, which is essentially a spline approximation problem. Here, we introduce the space $\operatorname{BV}^2([0,1])$ defined as:
\begin{equation}
    \operatorname{BV}^2([0,1]) := \{u\in C([0,1]) \mid u' \text{ exists a.e., and } u' \in \operatorname{BV}([0,1])\}.
\end{equation}
We consider the perturbation $u$ to be in the space $\operatorname{BV}^2([0,1])$ and with zero boundary conditions, defined as:
\begin{equation}
    \operatorname{BV}^2_0([0,1]) := \{u\in \operatorname{BV}^2([0,1]) \mid  u(0)=u(1)=0\}.
\end{equation}

\begin{proposition}
    \label{prop:local_norm_relu}
        For $ u \in \operatorname{BV}_0^2([0,1])$ and $\psi\in \mathcal M$, the following identity holds
    \begin{equation}
       \|u\|_\psi =\inf\{\,s>0\mid u\in \overline{\mathbf{CH}_s(\mathcal F)}^{C^0}\circ \psi\,\}
        =\left\|\frac{u^\prime}{\psi^\prime}\right\|_{\operatorname{TV}[0,1]}.
    \end{equation}
\end{proposition}

With this local picture, we can imagine computing the global distance $d_{\mathcal F}(\cdot, \cdot)$ by summing up the local costs along a path connecting two target functions. Specifically, consider $\psi_1, \psi_2 \in \mathcal M$, we can consider a sequence of intermediate functions $\xi_0 = \psi_1, \xi_1, \cdots, \xi_n = \psi_2$ in $\mathcal M$, such that each pair $\xi_{i}, \xi_{i+1}$ are very close to each other, i.e. $d_{\mathcal F}(\xi_i, \xi_{i+1})$ is very small. The distance $d_{\mathcal F}(\psi_1, \psi_2)$ is then bounded by:
\begin{equation}
    d_{\mathcal F}(\psi_1, \psi_2) \le \sum_{i=0}^{n-1} d_{\mathcal F}(\xi_i, \xi_{i+1}) \approx \sum_{i=0}^{n-1} \|\xi_{i+1}-\xi_i\|_{\xi_i}.
\end{equation}
Taking the limit as the partition gets finer, we arrive at the integral form:
\begin{equation}
    d_{\mathcal F}(\psi_1, \psi_2)
    \le \int_0^1 \left\|\frac{d}{dt}\gamma(t)\right\|_{\gamma(t)}\, dt
\end{equation}
for any piece-wise smooth path $\gamma:[0,1]\to \mathcal M$ with $\gamma(0) = \psi_1$ and $\gamma(1) = \psi_2$. By taking the infimum over all such paths, we obtain a characterization of $d_{\mathcal F}(\cdot, \cdot)$ as the shortest length of the path connecting two points under the local norm $\|\cdot\|_{\cdot}$:
\begin{equation}
    \label{eq:global_distance_variational_relu}
    \begin{aligned}
    &d_{\mathcal F}(\psi_1, \psi_2) = \\
    &\inf
    \left\{ \int_0^1
    \left\|\frac{d}{dt}\gamma(t) \right\|_{\gamma(t)}\,dt
    \bigm| \gamma:[0,1]\to\mathcal M \text{ is piece-wise smooth },\; \gamma(0)=\psi_1,\; \gamma(1)=\psi_2 \right\}.
    \end{aligned}
\end{equation}
The smoothness of the path $\gamma$ depends on the manifold structure of $\mathcal M$, which will be rigorously introduced in the next section.
With this characterization and the explicit form of the local norm $\|\cdot\|_{\psi}$, we can explicitly derive a path that minimizing the distance integral between any two functions $\psi,\xi$ in $\mathcal M$ (derived in ~\eqref{eq:geodesic_1d}):
% \ql{REF where this is derived}
\begin{equation}
    \tilde \gamma_t(x)  = \frac{\int_0^x(\psi_1^\prime(x))^{1-t}(\psi_2^\prime(x))^t dx}{\int_0^1 (\psi_1^\prime(x))^{1-t}(\psi_2^\prime(x))^t dx}.
\end{equation}
Finally, integrating the local norm along this path, we can derive a closed form expression of the distance $d_{\mathcal F}(\cdot, \cdot)$, and an explicit characterization of the manifold $\mathcal T$. Summarizing the results, we have:
\begin{theorem}
    \label{thm:global_distance_1d_relu}
$\mathcal T$ is characterized as:
\begin{equation}
    \mathcal T = \{\psi\in \operatorname{BV}^2([0,1]) \mid \psi(0)=0, \psi_1 =1,  \psi^\prime \ge c > 0 \text{ a.e.}\text{ for some } c\}.
\end{equation}
For any $\psi_1, \psi_2\in \mathcal T$, we have
\begin{equation}
    \label{eq:global_distance_1d_relu}
    d_{\mathcal F}(\psi_1, \psi_2) =
    \left\|\ln \psi_1^{\prime}-\ln \psi_2^\prime \right\|_{\operatorname{TV}([0,1])}
\end{equation}
\end{theorem}

As a corollary, the complexity measure $\mathbf C_{\mathcal F}(\psi)$ for any $\psi\in \mathcal T$ is given by
\begin{equation}
    \mathbf C_{\mathcal F}(\psi) = d_{\mathcal F}(\psi, \operatorname{Id}) = \left\|\ln \psi^{\prime}\right\|_{\operatorname{TV}([0,1])}.
\end{equation}

Compared with the estimation in~\Cref{thm:1D_relu}~\cite{li2022deep},
\Cref{thm:global_distance_1d_relu} not only sharpens the upper bound, but also provides an exact characterization of the complexity measure $\mathbf C_{\mathcal F}(\psi)$ for all $\psi\in \mathcal T$. Moreover, we also have a closed-form characterization of the target function space $\mathcal T$.
% \ql{Add the part about local norm's variational relationship to global distance - with formulas}
% \ql{Where is this characterization? Part of Theorem 4.2?}
These improvements are achieved by the extension of $\mathbf C_{\mathcal F}(\cdot)$ to the distance function $d_{\mathcal F}(\cdot, \cdot)$, and the variational relationship between the local norm $\|\cdot\|_{\psi}$ and the global distance $d_{\mathcal F}(\cdot, \cdot)$ given in~\Cref{eq:global_distance_variational_relu}.
The local norm $\|\cdot\|_{\psi}$ connects the distance $d_{\mathcal F}(\cdot, \cdot)$ to the control family $\mathcal F$ through the variational problem in~\Cref{prop:local_norm_relu}, turning the difficult problem of computing
the minimal time-horizon to an optimization problem over paths,
which in this case is completely solvable.
This analysis paves way for the general theory we will present
in the next section.

%% file: proofs_1DReLU.tex
\subsection{Proofs of the results for 1D ReLU control family}
\label{sec:1d_relu_proofs}

We provide the proofs and computations
behind the results stated in \Cref{sec:1d_relu}.
The rigorous statement and proofs of the connections between the global distance $d_{\mathcal F}(\cdot, \cdot)$, the local norm $\|\cdot\|_{\mathcal F}$, together
with the variational characterization (\eqref{prop:local_norm_relu} and~\eqref{eq:global_distance_1d_relu}) are deferred to~\Cref{sec:main_results_general},
where the results are proved in a more general setting.
Here, we only focus on the computations of the closed-form formulae
relevant for the 1D ReLU control family.
\subsubsection{Proof of \Cref{prop:local_norm_relu}}
We first give the sketch of the proof of \Cref{prop:local_norm_relu}, and
then provide the detailed computations. The proof is separated into the following steps:
\begin{itemize}
    \item \textbf{Step 1: } We first consider the special case where $\psi$ is the identity mapping on $[0,1]$, and prove that
    \begin{equation}
        \|u^\prime\|_{\operatorname{TV}[0,1]} \ge \inf\{\,s>0\mid u\in \overline{\mathbf{CH}_s(\mathcal F)}^{C^0}\circ \psi\,\}
    \end{equation}
    for any $u\in \operatorname{BV}^2_0([0,1])$.
    This is done by first considering a discrete version with interpolation on equally distributed points. We study the weight $\ell^1$ minimization problem defined by~\eqref{eq:l1_minimization} and~\eqref{eq:interpolation_condition}, and provide an exact formula for the minimum value of $S_N(w,v)$ in \Cref{prop:min-Sn}. This provides a way of constructing the approximation $\tilde u_N$ with controlled $\ell^1$ norm. Then, we take the limit $N\to \infty$ to obtain the desired inequality.
    \item \textbf{Step 2: } Next, we prove the reverse inequality
    \begin{equation}
        \|u^\prime\|_{\operatorname{TV}[0,1]} \le \inf\{\,s>0\mid u\in \overline{\mathbf{CH}_s(\mathcal F)}^{C^0}\circ \psi\,\}
    \end{equation}
    in~\Cref{prop:interpolation_lower_bound}.
    This is proved by contradiction. Assume the opposite inequality holds, then one can construct an interpolation scheme with controlled $\ell^1$ norm, which contradicts the exact formula in \Cref{prop:min-Sn}.
    \item \textbf{Step 3: } Finally, we extend the result to general $\psi\in \mathcal M$ by a change of variable.
\end{itemize}

We now provide details for Step 1, assuming that $\psi$ is the identity mapping $[0,1] \to [0,1]$.
We consider a discrete version with interpolation on equally distributed points.
Specifically, for a given positive integer $N$,
we consider a shallow network with an additional constant bias term $C_N$:
\begin{equation}
   \tilde u_N(x): = \sum_{i=0}^{N} \left(w_i \sigma(x-\frac{i}{N})+ v_i\sigma(\frac i N-x)\right) + C_N,
\end{equation}
where $\sigma(x)=\max\{0, x\}$ is the ReLU activation function,  $w_i$, $v_i$ and $C_N$ are parameters to be determined later.

We study the following $\ell^1$ minimization problem with interpolation constraints:
\begin{align}
    \label{eq:l1_minimization}
   \min ~& S_N(w,v):=  \sum_{i=0}^{N}|w_i|+|v_i|, \\
   \label{eq:interpolation_condition}\text{s.t. }& \tilde u_N(\frac{i}{N})=u(\frac{i}{N}), \quad i=0,1,\cdots, N.
\end{align}
among all possible choices of $\tilde u_N$.

Let $\operatorname{dist}(x,A)$ be the distance between a point $x$ and a set (interval) $A$. The following lemma is useful in the following argument.

\begin{lemma}
    The following results hold:
    \begin{equation} \label{eq:lemma-dist-1} |a-b|+|b| = |a|+ 2\operatorname{dist}(b, [\min\{a, 0\}, \max\{a, 0\}]) \text{ for any }a,b\in \mathbb R,
    \end{equation}
    and
    \begin{equation} \label{eq:lemma-dist-2} \operatorname{dist}(\sum_{i=1}^n a_i, \sum_{i=1}^n A_i)\le \sum_{i=1}^n \operatorname{dist}(a_i, A_i)\text{ for any }a_i\in \mathbb R, \text{ and }A_i\subset \mathbb R,
    \end{equation}
where the sum of the intervals is the Minkowski addition.

Moreover, if each $A_i$ is a closed interval $A_i=[\ell_i,r_i]$, then
equality holds in \eqref{eq:lemma-dist-2} whenever one of the following happens:
\[
\text{(i) } a_i\in [\ell_i,r_i]\ \text{ for all } i,\qquad
\text{(ii) } a_i\le \ell_i\ \text{ for all } i,\qquad
\text{(iii) } a_i\ge r_i\ \text{ for all } i.
\]
\end{lemma}

\begin{proof}
We first prove \eqref{eq:lemma-dist-1}.
By symmetry it suffices to treat $a\ge 0$, so the interval is $[0,a]$.
\begin{itemize}
\item If $b\in[0,a]$, then
\[
|a-b|+|b|=(a-b)+b=a=|a|,
\qquad
\operatorname{dist}(b,[0,a])=0.
\]
\item If $b<0$, then
\[
|a-b|+|b|=(a-b)+(-b)=a-2b=a+2(-b)
=|a|+2\,\operatorname{dist}(b,[0,a]).
\]
\item If $b>a$, then
\[
|a-b|+|b|=(b-a)+b=2b-a=a+2(b-a)
=|a|+2\,\operatorname{dist}(b,[0,a]).
\]
\end{itemize}
This proves \eqref{eq:lemma-dist-1}.

Now we prove \eqref{eq:lemma-dist-2}.
Fix $\varepsilon>0$. For each $i$, choose $x_i\in A_i$ such that
$|a_i-x_i|\le \operatorname{dist}(a_i,A_i)+\varepsilon/n$.
Then $x:=\sum_{i=1}^n x_i\in \sum_{i=1}^n A_i$, hence
\begin{equation}
\operatorname{dist}\Bigl(\sum_{i=1}^n a_i,\ \sum_{i=1}^n A_i\Bigr)
\le \Bigl|\sum_{i=1}^n a_i - x\Bigr|
= \Bigl|\sum_{i=1}^n (a_i-x_i)\Bigr|
\le \sum_{i=1}^n |a_i-x_i|
\le \sum_{i=1}^n \operatorname{dist}(a_i,A_i)+\varepsilon.
\end{equation}
Letting $\varepsilon\to 0$ yields \eqref{eq:lemma-dist-2}.

Finally, we prove the equality conditions.
Assume $A_i=[\ell_i,r_i]$. Then $\sum_i A_i=[\sum_i \ell_i,\ \sum_i r_i]$.
If (i) holds, then $\sum_i a_i\in \sum_i A_i$, so both sides are $0$.
If (ii) holds, then $\operatorname{dist}(a_i,A_i)=\ell_i-a_i$ and
\begin{equation}
\operatorname{dist}\Bigl(\sum_i a_i,\ \sum_i A_i\Bigr)
= \Bigl(\sum_i \ell_i\Bigr)-\Bigl(\sum_i a_i\Bigr)
= \sum_i (\ell_i-a_i)
= \sum_i \operatorname{dist}(a_i,A_i).
\end{equation}
Case (iii) is similar.
\end{proof}

The following proposition provides the exact solution of $S_N(w,v)$ for the $\ell_1$ minimization problem with interpolation constraints.

\begin{proposition}
\label{prop:min-Sn}
Let \begin{equation}
    \label{eq:k_i}
    \begin{aligned}
        k_i &= N\Delta^2 u(\frac{i}{N}):= N(u(\frac {i+1}{N})-2u(\frac i N)+u(\frac{i-1}{N})), \qquad i=1, \cdots, N-1,\\
    \end{aligned}
\end{equation}
and
\begin{equation}
    K_+ = \sum_{i=1}^{N-1} \max\{k_i, 0\}, \qquad K_- = \sum_{i=1}^{N-1} \min\{k_i, 0\}.
\end{equation}

Then, we have
\begin{equation}
\label{eq:min-Sn}
    \min S_N(w,v) = \sum_{i=1}^{N-1} |k_i| + \operatorname{dist}\Big(-N(u(\frac 1 N)-u(0)), [K_-, K_+] \Big).
\end{equation}

\end{proposition}

\begin{proof}

First, notice that the terms $w_N\sigma(x-1)$ and $ v_0\sigma(x)$ do not contribute to the interpolation, and thus $w_N$ and $v_0$ must be zero in an optimal $\tilde u_N$. Now, take $x= 0, \frac 1 N, \cdots, 1$, by the interpolation condition~\eqref{eq:interpolation_condition},
we have:
\begin{equation} \label{eq:interpolation_conditions}
    \frac{1}{N}(\sum_{i=j+1}^N iv_i+\sum_{i=0}^{j-1} (j-i)w_i)=u(\frac j N) - C_N, \qquad j=0,1, \cdots,N.
\end{equation}
Here the summation is zero if the index is out of range.
Now, we denote $\tilde k_i=w_i+v_i$ for $i=0, \cdots, N$.
By taking the second order difference, it follows that the conditions in~\eqref{eq:interpolation_conditions} are equivalent to:
\begin{equation}
    \tilde k_0 =N(u(\frac 1 N)-u(0))+ \sum_{i= 1}^N v_i, \qquad \tilde k_i = k_i, \quad i = 1, \cdots, N-1,
\end{equation}
and
\begin{equation}
    \label{eq:0th_condition}
    \sum_{i=1}^N i v_i = u(0)-C_N.
\end{equation}
Since $C_N$ is a free variable, we can consider the problem without restriction~\eqref{eq:0th_condition}. Therefore, we have:
% By taking the second order differences of the sequence $u(0), u(\frac 1 N), \cdots, u(1)$, we have:
% \begin{equation}
%     \label{eq:k_i}
%     \begin{aligned}
%         k_0&=N(u(\frac 1 N)-u(0))+ g(i), \\
%         k_i &= N\Delta^2 u(\frac{i}{N}):= N(u(\frac {i+1}{N})-2u(\frac i N)+u(\frac{i-1}{N})), \qquad i=1, \cdots, N-1,\\
%     \end{aligned}
% \end{equation}
\begin{equation}
    \label{eq:S_estimate}
    \begin{aligned}
            S_N(w, v)&= \sum_{i=1}^{N-1} (|k_i-v_i| + |v_i|) + \left|N(u(\frac 1 N)- u(0))+\sum_{i=1}^{N} v_i\right|+ |v_N|\\
            & \ge \sum_{i=1}^{N-1} (|k_i-v_i| + |v_i|) + \left|N(u(\frac 1 N)- u(0))+\sum_{i=1}^{N-1} v_i\right|.
    \end{aligned}
\end{equation}
Here, we use the inequality that $|a+b|+|b|\ge |a|$ for any $a,b\in \mathbb R$.
% In the following, we we denote $\operatorname{dist}(x, A)$ as the infimum distance between a point $x$ and a set $A$.
Continuing with~\eqref{eq:S_estimate}, we have:
\begin{equation}
    \label{eq:S_estimate_2}
    \begin{aligned}
        S_N(w,v)&\ge \sum_{i=1}^{N-1} (|k_i-v_i| + |v_i|) + \left|N(u(\frac 1 N)- u(0))+\sum_{i=1}^{N-1} v_i\right|\\
        &= \sum_{i=1}^{N-1} |k| + 2\sum_{i=1}^{N-1} \operatorname{dist}(v_i, [\min\{k_i, 0\}, \max\{k_i, 0\}]) + \left|N(u(\frac 1 N)- u(0))+\sum_{i=1}^{N-1} v_i\right|\\
        & \ge \sum_{i=1}^{N-1} |k_i|+2\operatorname{dist}(\sum_{i=1}^{N-1} v_i, \sum_{i=1}^{N-1} [\min\{k_i, 0\}, \max\{k_i, 0\}]) + \left|N(u(\frac 1 N)- u(0))+\sum_{i=1}^{N-1} v_i\right|\\
        & = \sum_{i=1}^{N-1} |k_i|+2\operatorname{dist}(\sum_{i=1}^{N-1} v_i, [K_-, K_+]) + \left|N(u(\frac 1 N)- u(0))+\sum_{i=1}^{N-1} v_i\right|,
    \end{aligned}
\end{equation}
where the second line uses \eqref{eq:lemma-dist-1}, and the third line uses \eqref{eq:lemma-dist-2}.
% \ql{please check duplicate label warnings. There are many of them}

Let us denote $V = \sum_{i =1}^{N-1} v_i$.
Then, it can be readily checked that the minimum of $2\operatorname{dist}(V, [K_-, K_+]) + | N(u(\frac{1}{N}) - u(0)) + V|$ can be achieved at
\begin{equation}
V=
\begin{cases}
  - N(u(\frac{1}{N}) - u(0)) , & \text{ if } -N (u(\frac{1}{N}) - u(0))  \in [K_-, K_+] \\
 K_-, & \text{ if } -N(u(\frac 1 N)-u(0))<K_-, \\
K_+, & \text{ if }-N(u(\frac 1 N)-u(0))>K_+.
\end{cases}
\end{equation}

The minimum then can be calculated as $\operatorname{dist}(-N(u(\frac 1 N)-u(0)), [K_-, K_+])$.

% \begin{itemize}
%     \item $\sum_{i=1}^{N-1} v_i = -N(u(\frac 1 N)-u(0))$, if $N(u(\frac 1 N)-u(0))\in [K_-, K_+]$;
%     \item $\sum_{i=1}^{N-1} v_i = K_-$, if $N(u(\frac 1 N)-u(0))<K_-$;
%     \item $\sum_{i=1}^{N-1} v_i = K_+$, if $N(u(\frac 1 N)-u(0))>K_+$.
% \end{itemize}
% All these cases result in a minimum of $\operatorname{dist}(-N(u(\frac 1 N)-u(0)), [K_-, K_+])$.
% Finally, we have proved that
Therefore, we have
\begin{equation}
    S_N(w,v)\ge \sum_{i=1}^{N-1} |k_i| + \operatorname{dist}(-N(u(\frac 1 N)-u(0)), [K_-, K_+]).
\end{equation}
Moreover, this value can be achieved by taking
$v_N = 0$, and for $i < N$ we choose $v_i$ in $[\min\{k_i, 0\}, \max\{k_i, 0\}]$, and $\sum_{i=1}^{N-1} v_i = V$ to minimize the last equation in~\eqref{eq:S_estimate_2}. The value of $C_N$ is now given by:
\begin{equation}
    C_N= u(0)-\frac{1}{N} \sum_{i=1}^N iv_i.
\end{equation}
\end{proof}

By taking a continuous limit on \eqref{eq:min-Sn}, we are ready to show the upper bound.
The following lemma for the limit of second order difference of $u\in \operatorname{BV}^2([0,1])$ is useful.
\begin{lemma}
    \label{lem:limit-second-difference}
    For any $u\in \operatorname{BV}^2([0,1])$, we have
    \begin{equation}
        \lim_{N\to\infty} \sum_{i=1}^{N-1} |k_i| = \|u'\|_{\operatorname{TV}[0,1]}.
    \end{equation}
\end{lemma}
\begin{proof}
Set
\begin{equation}
    \delta_k := N\big(u(x_k)-u(x_{k-1})\big) = N\int_{x_{k-1}}^{x_k} g(x),dx,\qquad k=1,\dots,N.
\end{equation}
Then
\begin{equation}
    N\big(u(x_{k+1})-2u(x_k)+u(x_{k-1})\big)=\delta_{k+1}-\delta_k,
\quad\text{so}\quad
S_N=\sum_{k=1}^{N-1}|\delta_{k+1}-\delta_k|
\end{equation}
Let $\mathcal A_N$ be the set of continuous, piecewise linear $\phi:[0,1]\to\mathbb R$ with nodes $\{x_k\}$, $\phi(0)=\phi(1)=0$, and $|\phi|_\infty\le 1$. On each $[x_{k-1},x_k]$ we have $\phi'(x)=N(\phi(x_k)-\phi(x_{k-1}))=:N\Delta\phi_k$, hence
\begin{equation}
\int_0^1 g\phi'
=\sum_{k=1}^{N} N\Delta\phi_k\int_{x_{k-1}}^{x_k} g
=\sum_{k=1}^{N}\delta_k\Delta\phi_k
= -\sum_{k=1}^{N-1}(\delta_{k+1}-\delta_k)\phi(x_k).
\end{equation}
Therefore,
\begin{equation}
\bigg|\int_0^1 g\phi'\bigg|
\le \sum_{k=1}^{N-1}|\delta_{k+1}-\delta_k||\phi(x_k)|
\le S_N.
\end{equation}
Choosing $\phi\in\mathcal A_N$ with $\phi(x_k)=\operatorname{sgn}(\delta_{k+1}-\delta_k)$ (and linear interpolation with $\phi(0)=\phi(1)=0$) gives equality, hence
\begin{equation}
\quad S_N=\sup_{\phi\in\mathcal A_N}\int_0^1 g\phi'\,dx
\end{equation}

Using the dual representation of total variation,
% \ql{unify $\|\cdot\|$ vs $|\cdot|$ for TV.}
\begin{equation}
\|g\|_{\operatorname{TV}([0,1])}
=\sup\Big\{\int_0^1 g\varphi' : \varphi\in C_0^1((0,1)),\ \|\varphi\|_\infty\le 1\Big\},
\end{equation}
and the inclusion $\mathcal A_N\subset\{\varphi:\ \|\varphi\|_\infty\le1,\ \varphi(0)=\varphi(1)=0,\ \varphi\ \text{Lipschitz}\}$, we have:
\begin{equation}
\limsup_{N\to\infty} S_N \ \le\ \|g\|_{\operatorname{TV}([0,1])}.
\end{equation}

For any $\psi\in C_0^1((0,1))$ with $|\psi|_\infty\le1$, let $\phi_N\in\mathcal A_N$ be the piecewise linear interpolant of $\psi$ on $\{x_k\}$.
Then $\phi_N\to\psi$ in $W^{1,1}$, hence
\begin{equation}
\int_0^1 g\phi_N' \ \to\ \int_0^1 g\psi'.
\end{equation}
Therefore, $S_N\ge \int_0^1 g\phi_N'$. Taking $\liminf$ and then the supremum over such $\psi$ yields
\begin{equation}
\liminf_{N\to\infty} S_N \ \ge\ \|g\|_{\operatorname{TV}([0,1])}.
\end{equation}

Combining the two bounds gives $S_N\to \|g\|_{\operatorname{TV}([0,1])}$.
\end{proof}

Combining \Cref{prop:min-Sn} and \Cref{lem:limit-second-difference}, we can conclude Step 1 with the following upper bound.

\begin{proposition}
    For $ u \in \operatorname{BV}^2_0([0,1])$, the following identity holds
    \begin{equation}
        \inf \{ s \mid u \in \overline{\mathbf{CH}_s(\mathcal F)}^{C^0}\}
         \le \|{u^\prime}\|_{\operatorname{TV}[0,1]}
    \end{equation}
\end{proposition}
\begin{proof}
    From \eqref{eq:min-Sn}, the assumption that $u \in \operatorname{BV}^2_0([0,1])$ and~\Cref{lem:limit-second-difference}, we have
\begin{equation}
    \label{eq:limit_S_N}
    \lim_{N\to\infty} S_N(w,v) = \|u^\prime\|_{\operatorname{TV}[0,1]} + \operatorname{dist}\left(- u^\prime(0), \left[\int_{[0,1]}\min\{u^{\prime\prime}(x), 0\}dx, \int_{[0,1]}\max\{u^{\prime\prime}(x), 0\}dx\right]\right).
\end{equation}

Note that:
\begin{equation}
    u^\prime(0)- u^\prime(1) = \int_{[0,1]}\min\{u^{\prime\prime}(x), 0\}dx + \int_{[0,1]}\max\{u^{\prime\prime}(x), 0\}dx,
\end{equation}
the right hand side of equation~\eqref{eq:limit_S_N} can be rewritten in a symmetric form:
\begin{equation}
    \|u^\prime\|_{\operatorname{TV}[0,1]} + \frac{1}{2} \max\{|u^\prime(0)+u^\prime(1)| - \|u^\prime\|_{\operatorname{TV}[0,1]}, 0\}.
\end{equation}
According to the following lemma, we have that
\begin{equation}
    \lim_{N\to\infty} S_N(w,v) = \|u^\prime\|_{\operatorname{TV}[0,1]}.
\end{equation}
\begin{lemma}
For $u\in \operatorname{BV}^2_0([0,1])$, we have $|u^\prime(0)+u^\prime(1)| \le \|u^\prime\|_{\operatorname{TV}[0,1]}$.
\end{lemma}
\begin{proof}[Proof of the lemma]
Since $u\in \operatorname{BV}_0^2([0,1])$, we have $u'\in \operatorname{BV}([0,1])\subset L^1(0,1)$ and $u$ is absolutely continuous with
\begin{equation}
u(1)-u(0)=\int_0^1 u'(s)\,ds = 0 .
\end{equation}
If $u'(1)=0$, then
\begin{equation}
|u'(0)+u'(1)| = |u'(0)-u'(1)| \le \|u'\|_{\operatorname{TV}[0,1]} .
\end{equation}
Assume now $u'(1)>0$ (the case $u'(1)<0$ is analogous). From $\int_0^1 u'=0$ it follows that $u'$ cannot be nonnegative a.e.
unless $u'=0$ a.e.; hence there exists $t\in(0,1)$ such that $u'(t)$ exists and $u'(t)\le 0$. Then $u'(t)\,u'(1)\le 0$, so
\begin{equation}
|u'(t)+u'(1)| = \big||u'(1)|-|u'(t)|\big| \le |u'(1)|+|u'(t)| = |u'(1)-u'(t)|.
\end{equation}
Therefore,
\begin{equation}
\begin{aligned}
|u'(0)+u'(1)|
&\le |u'(0)-u'(t)| + |u'(t)+u'(1)| \\
&\le |u'(0)-u'(t)| + |u'(1)-u'(t)| \\
&\le \|u'\|_{\operatorname{TV}[0,1]},
\end{aligned}
\end{equation}
where the last inequality follows from the definition of total variation by taking the partition $\{0,t,1\}$:
$\|u'\|_{\operatorname{TV}[0,1]}\ge |u'(t)-u'(0)|+|u'(1)-u'(t)|$.
\end{proof}

Here, $u^\prime(0)$ and $u^\prime(1)$ are the one-sided limits of $u^\prime$ at the endpoints, which is well-defined since $u\in \operatorname{BV}_0^2([0,1])$.

For each fixed finite $N$, we define a function:
\begin{equation}
    \bar u_N: = \tilde u_N - C_N + \frac{1}{N}\sigma(x+N C_N).
\end{equation}
where $C_N$ is the value of the constant term in the previous proof. It is then clear that $\bar u_N$ is in $\operatorname{Span} \mathcal F$. It follows that
\begin{equation}
    |\bar u_N-\tilde u_N|_{C([0,1])}\le \frac{1}{N}.
\end{equation}
Also, we have that the sum of weights in $\bar u_N$ is no more than $\min S_N(w,v)+\frac{1}{N}$. Notice that $\tilde u_N$ is a piecewise linear interpolant of $u$. 
% Consider a subsequence $N_k = 2^k$ for $k=1,2,\cdots$. For any $k<l$, we have $u_{N_l} - u_{N_k}$ is a piecewise linear function which is zero at the breakpoints $\{0, \frac{1}{N_k}, \cdots, 1\}$.
% The maximum value of $|u_{N_l} - u_{N_k}|$ is attained at one of the breakpoints, and thus
Since $u \in \operatorname{BV}^2_0([0,1])$, when $N\to \infty$ we obtain $\bar u_N\to u$.
This implies
\begin{equation}
    \label{eq:optimal_cost}
    \begin{aligned}
            \operatorname{inf}\{ s \mid u\in \overline{\mathbf{CH}_s(\mathcal F)}\}^{C^0}&\le \lim_{N\to \infty} \left(\min S_N(w,v)+\frac1 N\right)\\
            &= \|u^\prime\|_{\operatorname{TV}[0,1]} + \frac{1}{2} \max\{|u^\prime(0)+u^\prime(1)| - \|u^\prime\|_{\operatorname{TV}[0,1]}, 0\}\\
            & = \|u^\prime\|_{\operatorname{TV}[0,1]}.
    \end{aligned}
\end{equation}
The last equality uses the following fact: that

\end{proof}

Now, we provide the details for Step 2, i.e., we show the cost function (defined as the right-hand side of \eqref{eq:optimal_cost}) is also the lower bound.

\begin{proposition}
    \label{prop:interpolation_lower_bound}
    For $ u \in \operatorname{BV}^2_0([0,1])$, the following identity holds
    \begin{equation}
        \inf \{ s \mid u \in \overline{\mathbf{CH}_s(\mathcal F)}^{C^0}\}
         \ge \|{u^\prime}\|_{\operatorname{TV}[0,1]}
    \end{equation}
\end{proposition}

\begin{proof}
% \lt{working/}
% \noindent\textbf{Step 1: optimal cost for $N$-grids interpolation}

% Since $h\in W^{2,1}([0,1])$, we have that

% Now we need to show that $\operatorname{Cost}(u)$ is also a lower bound for $\inf \{ s ~|~ u \in \overline{\mathbf{CH}_{s}(\mathcal F) \circ \psi}\}.$.

Suppose the opposite holds, then there exists a sequence
\begin{equation}
    g_k(x) = \sum_{i=1}^{N_k} \alpha_i\sigma(x-b_i)+\sum_{j=1}^{M_k} \beta_j\sigma(x-c_j), k=1,2, \cdots,
\end{equation}
such that
\begin{equation}
    \lim_{k\to \infty} g_k=u, \text{ and } \lim_{k\to \infty} \left(\sum_{i=1}^{N_k}|\alpha_i|+\sum_{j=1}^{M_k}|\beta_j|\right) = \|{u^\prime}\|_{\operatorname{TV}[0,1]}-\varepsilon,
\end{equation}
for some $\varepsilon >0$. We can assume that for each $k$, $\|g_k-u\|_{C([0,1])}\le \frac 1 k$ by choosing a proper subsequence, and $\left(\sum_{i=1}^{N_k}|\alpha_i|+\sum_{j=1}^{M_k}|\beta_j|\right) < \|{u^\prime}\|_{\operatorname{TV}[0,1]}.$

The key is to adjust the node points $b_i, c_i$ to some equidistributed points. To see this, we introduce the following modifications: For each $b_i\in [0,1]$, let $\tilde b_i$ be one of $\{0, \frac 1 k, \cdots, \frac{k-1}{k}, 1\}$ that is closest to $b_i$, for $i = 1,2,\cdots, N_k$.
Similarly, we define $\tilde c_i$.
For all $b_i<0$ and $c_j>1$, we rewrite $\alpha_i\sigma(x-b_i) = \alpha_i\sigma(x) - \alpha_i b_i$ and $\beta_j\sigma(c_j-x) = \beta_j\sigma(1-x) - \beta_j (c_i-1)$ for $x\in [0,1]$.
Subsequently, we define
\begin{equation}
    \begin{aligned}
            \tilde g_k(x) & := \sum_{i:b_i\in [0,1]} \alpha_i\sigma(x-\tilde b_i)+\sum_{j:c_j\in [0,1]} \beta_j\sigma(x-\tilde c_j)\\ &+\sum_{i:b_i<0} (\alpha_i\sigma(x) - \alpha_i b_i)  +\sum_{j:c_j>1} (\beta_j\sigma(1-x) - \beta_j (c_i-1))\\
            &= \sum_{i=0}^{k} \left(w_i \sigma(x-\frac{i}{k})+ v_i\sigma(\frac i k-x)\right) +C_k,
    \end{aligned}
\end{equation}
where $w_i$ and $v_i$ are the weights after merging, and $C_k$ is the constant term after merging.
% Notice that $\tilde g_k$ is linear in all the intervals $[\frac i k, \frac{i+1}{k}]$ for all $i=0,1 \cdots k-1$ and $\tilde g_k$ has a weight sum no more than that of $g_k$, i.e.
It then holds that
\begin{equation}
    \label{eq:weight_sum_upper_bound}
    \sum_{i=1}^k (|w_i|+|v_i|)\le \sum_{i=1}^{N_k}|\alpha_i|+\sum_{j=1}^{M_k}|\beta_j| \le \|{u^\prime}\|_{\operatorname{TV}[0,1]}-\varepsilon.
\end{equation}

Also, for sufficiently large $k$ we have:
\begin{equation}
    \|\tilde g_k-g_k\|_{C([0,1])}\le \frac{\|{u^\prime}\|_{\operatorname{TV}[0,1]}}{2} \cdot \frac{1}{k},
\end{equation}
thus
\begin{equation}
    \|\tilde g_k-u\|_{C([0,1])}\le \frac{1}{k}(\|{u^\prime}\|_{\operatorname{TV}[0,1]}+ 1).
\end{equation}

Now we fix $k$, and specify the error as
$\varepsilon_i = \tilde g_k(\frac i k)-u(\frac i k)$,
for $ i=0,1,\cdots, k,$ with  $|\varepsilon_j|\le  \|\tilde g_k-u\|_{C(K)}$.
Now we pivot at $\tilde g_k$, and the coefficients solve the following linear system:
\begin{equation}
    \label{eq:interpolation_perturbed}
    \frac{1}{N}(\sum_{i=j+1}^N iv_i+\sum_{i=0}^{j-1} (j-i)w_i)=u(\frac j N) -C_k+ \varepsilon_i, \qquad j=0,1, \cdots,N.
\end{equation}
In the following, we denote $u_i=u(\frac i k)$ for simplicity, and define:
\begin{equation}
    \begin{aligned}
            &\Delta^2 u_i: = u_{i+1}-2u_i+u_{i-1}, \quad i=1, \cdots, k-1, \quad \Delta^2 u_0: = u_1-u_0, \quad \Delta^2 u_k: = u_k-u_{k-1}.\\
            &\Delta^2 \varepsilon_i: = \varepsilon_{i+1}-2\varepsilon_i+\varepsilon_{i-1}, \quad i=1, \cdots, k-1, \quad\Delta^2 \varepsilon_0: = \varepsilon_1-\varepsilon_0, \quad \Delta^2 \varepsilon_k: = \varepsilon_k-\varepsilon_{k-1}.
    \end{aligned}
\end{equation}
Thanks to~\eqref{eq:interpolation_perturbed} and \Cref{prop:min-Sn}, we have:
\begin{equation}
    \label{eq:weight_sum_estimation_perturbed}
    \begin{aligned}
            \sum_{i=0}^{k} |w_i|+|v_i| \ge  k\sum_{i=1}^{k-1} |\Delta^2 u_i+\Delta^2 \varepsilon_i| + \operatorname{dist}\left( - k(\Delta^2 u_0+\Delta^2 \varepsilon_0),[\widetilde K_-, \widetilde K_+] \right)
    \end{aligned}
\end{equation}
where
\begin{equation}
    \widetilde K_+ = k \sum_{i=1}^{k-1} \max\{\Delta^2 u_i+\Delta^2 \varepsilon_i, 0\}, \qquad \widetilde K_- = k\sum_{i=1}^{k-1} \min\{\Delta^2 u_i+\Delta^2 \varepsilon_i, 0\}.
\end{equation}
The term $u_k$ and $\varepsilon_k$ do not appear explicitly in the terms of \eqref{eq:weight_sum_estimation_perturbed}. Nevertheless, we can perform a symmetrization for this discrete counterpart.
It follows from $k(\Delta^2 u_0+\Delta^2 \varepsilon_0)+\widetilde K_-+\widetilde K_{+} = k(\Delta^2 u_k+\Delta^2\varepsilon)$ that:
\begin{equation}
    \begin{aligned}
        &\operatorname{dist}\left( - k(\Delta^2 u_0+\Delta^2 \varepsilon_0),[\widetilde K_-, \widetilde K_+] \right) \\
        =&\operatorname{dist}\left( - k\left(\Delta^2 u_0+\Delta^2 \varepsilon_0\right)-\frac{\widetilde K_-+\widetilde K_+}{2},\left[\frac{\widetilde K_--\widetilde K_+}{2}, \frac{\widetilde K_+-\widetilde K_-}{2}\right] \right)\\
        = & \operatorname{dist}\left( -k\frac{\Delta^2 u_0+\Delta^2 \varepsilon_0+\Delta^2 u_k+\Delta^2\varepsilon_k}{2},\left[\frac{\widetilde K_--\widetilde K_+}{2}, \frac{\widetilde K_+-\widetilde K_-}{2}\right] \right)\\
        =& \frac{k}{2} \max\left\{|\Delta^2 u_0+\Delta^2 \varepsilon_0 + \Delta^2 u_k+\Delta^2 \varepsilon_k|-\sum_{i=1}^{k-1} |\Delta^2 u_i+\Delta^2 \varepsilon_i|, 0\right\}
    \end{aligned}
\end{equation}
Therefore, ~\eqref{eq:weight_sum_estimation_perturbed} can be rewritten in a symmetric form:
\begin{equation}
\label{eq:weight_sum_estimation_perturbed_symmetric}
 \sum_{i=0}^{k} (|w_i|+|v_i|)\ge \sum_{i=1}^{k-1} k|\Delta^2 u_i+\Delta^2 \varepsilon_i| + \frac{k}{2} \max\left\{|\Delta^2 u_0+\Delta^2 \varepsilon_0 + \Delta^2 u_k+\Delta^2 \varepsilon_k|-\sum_{i=1}^{k-1} |\Delta^2 u_i+\Delta^2 \varepsilon_i|, 0\right\},
\end{equation}

The rest of the proof is aimed at reducing the error term $\varepsilon_i$ in the right-hand side. In particular, we prove that for any $\delta > 0$
and sufficiently large $k$ we have

\begin{equation}
    \sum_{i=0}^k \left(|w_i|+|v_i|\right)\ge k\sum_{i=1}^{k-1} |\Delta^2 u_i|+ \frac 1 2 k\max \left\{ \left( |\Delta^2 u_0 + \Delta^2 u_k|-\sum_{i=1}^{k-1} |\Delta^2 u_i|\right), 0 \right\}- \delta.
\end{equation}

Now we start from \eqref{eq:weight_sum_estimation_perturbed_symmetric}. The first step is to relax $|x|$ and $\max\{x, 0\}$, namely,
\begin{equation}
    \label{eq:dual_representations}
    |x| = \max_{\phi\in [-1,1]} \phi x, \quad \max\{x, 0\} = \max_{\theta\in [0,1]} \theta x.
\end{equation}
We then have the following estimation:
\begin{equation}
    \begin{aligned}
        \text{RHS of }
        \eqref{eq:weight_sum_estimation_perturbed_symmetric}&\ge k\sum_{i=1}^{k-1} |\Delta^2 u_i+\Delta^2 \varepsilon_i|+ \frac \theta 2 k\left( |\Delta^2 u_0+\Delta^2 \varepsilon_0 + \Delta^2 u_k+\Delta^2 \varepsilon_k|-\sum_{i=1}^{k-1} |\Delta^2 u_i+\Delta^2 \varepsilon_i|\right)\\
        &=k\sum_{i=1}^{k-1} (1-\frac \theta 2) |\Delta^2 u_i+\Delta^2 \varepsilon_i| + \frac \theta 2 k|\Delta^2 u_0+\Delta^2 \varepsilon_0 + \Delta^2 u_k+\Delta^2 \varepsilon_k|\\
        & \ge k(1-\frac{\theta}{2})\sum_{i=1}^{k-1}  \phi_i(\Delta^2 u_i+\Delta^2 \varepsilon_i) + \frac \theta 2 k|\Delta^2 u_0+\Delta^2 \varepsilon_0 + \Delta^2 u_k+\Delta^2 \varepsilon_k|,
    \end{aligned}
\end{equation}
for all $\theta\in [0,1]$ and $\phi_i\in [-1,1]$ for $i=1, \cdots, k-1$.
For arbitrarily defined $\phi_0$ and $\phi_k$, the following discrete integration by parts hold:
\begin{equation}
    \sum_{i=1}^{k-1}\phi_i\Delta^2 \varepsilon_i =\sum_{i=1}^{k-1} \varepsilon_i \Delta^2\phi_i + \phi_{k-1} \varepsilon_k+\phi_1 \varepsilon_0-\varepsilon_{k-1}\phi_k-\varepsilon_1\phi_0,
\end{equation}
where $\Delta^2 \phi_i = \phi_{i+1}-2\phi_i+\phi_{i-1}$ is the discrete central difference. We then have the estimation:
\begin{equation}
    \begin{aligned}
            \sum_{i=1}^{k-1}  \phi_i(\Delta^2 u_i+\Delta^2 \varepsilon_i) &= \sum_{i=1}^{k-1} \phi_i\Delta^2 u_i+\sum_{i=1}^{k-1} \varepsilon_i \Delta^2\phi_i + \phi_{k-1} \varepsilon_k+\phi_1 \varepsilon_0-\varepsilon_{k-1}\phi_k-\varepsilon_1\phi_0\\
            &\ge \underbrace{\sum_{i=1}^{k-1} \phi_i\Delta^2 u_i- \frac{C}{k}\sum_{i=1}^{k-1} |\Delta^2\phi_i|}_{A} + \phi_{k-1} \varepsilon_k+\phi_1 \varepsilon_0-\varepsilon_{k-1}\phi_k-\varepsilon_1\phi_0.
    \end{aligned}
\end{equation}
Here, $C := \|{u^\prime}\|_{\operatorname{TV}[0,1]} + 1$ is independent with $k$.
Therefore, it follows that the right-hand side of \eqref{eq:weight_sum_estimation_perturbed_symmetric} is not less than
\begin{equation}
    \begin{aligned}
k(1 - \frac{\theta}{2}) A + \underbrace{k(1-\frac{\theta}{2}) \Big(\phi_{k-1} \varepsilon_k+\phi_1 \varepsilon_0-\varepsilon_{k-1}\phi_k-\varepsilon_1\phi_0 \Big) + \frac \theta 2 k|\Delta^2 u_0+\Delta^2 \varepsilon_0 + \Delta^2 u_k+\Delta^2 \varepsilon_k|.}_B
    \end{aligned}
\end{equation}

For a fixed $\delta > 0$, and $a = \operatorname{sign}(\Delta^2 \varepsilon_0+\Delta^2\varepsilon_1)$,  there exists a function $\phi \in C^2([0,1]),$ such that
\begin{enumerate}
        \item[(a)] $|\phi|_{C([0,1])}\le 1$;
        \item[(b)] $\phi(0) = \phi(1)=a$;
        \item[(c)] $\int_{[0,1]} \phi(x) u^{\prime\prime}(x)dx: = \int_{[0,1]}\phi d(Du^\prime) > \|u^\prime\|_{\operatorname{TV}([0,1])}-\delta.$
    \end{enumerate}

Fix $\delta$ and $a$, for each $k$, we choose $\phi_i = \phi(\frac i k)$ for $i=0,1,\cdots, k$. For the term $A$, we have
\begin{equation}
    \label{eq:main_term_estimate}
\begin{aligned}
    &\liminf_{k\to \infty}  k\sum_{i=1}^{k-1} \phi_i\Delta^2 u_i- C\sum_{i=1}^{k-1} |\Delta^2\phi_i| \\
    &=\liminf_{k\to \infty} \int_{[0,1]} \phi(x) u^{\prime\prime}(x)dx - \frac{C}{k}\int_{[0,1]} |\phi^{\prime\prime}(x)|dx\\
    & \ge \|u^\prime\|_{\operatorname{TV}([0,1])} -\delta.
\end{aligned}
\end{equation}
For the term $B$, we have
\begin{equation}
    \label{eq:remaining_term}
    \begin{aligned}
            &k(1-\frac \theta 2)(\phi_{k-1} \varepsilon_k+\phi_1 \varepsilon_0-\varepsilon_{k-1}\phi_k-\varepsilon_1\phi_0)+k\frac{\theta}{2}(|\Delta^2 u_0+\Delta^2 \varepsilon_0 + \Delta^2 u_k+\Delta^2 \varepsilon_k|)\\
            = & k(1-\frac \theta 2)\bigl(|\Delta^2 \varepsilon_0+\Delta^2\varepsilon_k|+ (a-\phi_{k-1})\varepsilon_k+(\phi_1-a)\varepsilon_0 \bigr)+k\frac{\theta}{2}(|\Delta^2 u_0+\Delta^2 \varepsilon_0 + \Delta^2 u_k+\Delta^2 \varepsilon_k|)\\
            \ge  & k(1-\frac \theta 2)|\Delta^2 \varepsilon_0+\Delta^2\varepsilon_k|+k\frac{\theta}{2}(|\Delta^2 u_0+ \Delta^2 u_k+\Delta^2 \varepsilon_0 +\Delta^2 \varepsilon_k|) - C(|\phi_{k-1}-a|+|\phi_1-a|)\\
            \ge & k\frac{\theta}{2}(|\Delta^2 \varepsilon_0+\Delta^2\varepsilon_k|+|\Delta^2 u_0+ \Delta^2 u_k+\Delta^2 \varepsilon_0 +\Delta^2 \varepsilon_k|)- C(|\phi_{k-1}-a|+|\phi_1-a|)\\
            \ge & k\frac \theta 2 |\Delta^2 u_0+\Delta^2 u_k|-C(|\phi_{k-1}-a|+|\phi_1-a|).
    \end{aligned}
\end{equation}
As $k\to \infty$, $\phi_{k-1}\to a$ and $\phi_1\to a$. Therefore, for sufficiently large $k$, we have that:
\begin{equation}
B \ge k\frac \theta 2 |\Delta^2 u_0+\Delta^2 u_k|-\delta.
\end{equation}
Combining all results above, we have shown that for sufficiently large $k$,
\begin{equation}
    \sum_{i=0}^k \left(|w_i|+|v_i|\right)\ge k\sum_{i=1}^{k-1} |\Delta^2 u_i|+ \frac \theta 2 k\left( |\Delta^2 u_0 + \Delta^2 u_k|-\sum_{i=1}^{k-1} |\Delta^2 u_i|\right)-\delta
\end{equation}
for all $\theta\in [0,1]$. Recall the dual form of $\max\{x, 0\}$ in~\eqref{eq:dual_representations},
we have shown that:
\begin{equation}
    \sum_{i=0}^k \left(|w_i|+|v_i|\right)\ge \|{u^\prime}\|_{\operatorname{TV}[0,1]} -3\delta.
\end{equation}
Since $\delta$ can arbitrarily small,  we then have a contradiction to~\eqref{eq:weight_sum_upper_bound}.
\end{proof}

Therefore, we show
    \begin{equation}
        \inf \{ s \mid u \in \overline{\mathbf{CH}_s(\mathcal F)}^{\,C^0}\}
        =\|{u^\prime}\|_{\operatorname{TV}[0,1]}
    \end{equation}
for all $u\in \operatorname{BV}^2_0([0,1])$. 
% \ql{I feel that this proof is way too long and no one can follow this.
% Is there a way to write it in several steps where each part has a sufficiently
% short proof. Alternatively, provide a sketch of the proof first?}

\subsubsection{Proof of~\Cref{thm:global_distance_1d_relu}}
Now, we can provide the proof of~\Cref{thm:global_distance_1d_relu}. In the proof, we will use some of the geometric concepts that will be introduced later in~\Cref{sec:main_results_general}.
\begin{proof}[Proof of~\Cref{thm:global_distance_1d_relu}]
    We first show that
    \begin{equation}
            \mathcal M \subseteq \{\psi\in \operatorname{BV}^2([0,1]) \mid \psi(0)=0, \psi(1) =1,  \psi^\prime \ge c > 0 \text{ a.e.}\text{ for some } c\}.
    \end{equation}
Notice that all the functions in $\mathcal F$ are uniformly Lipschitz
and uniformly bounded in the $\operatorname{BV}^2$ norm. By a Grownwall's inequality estimate, flows in $\mathcal A_{\mathcal F}$ are also uniformly bounded in the $\operatorname{BV}^2$ norm. Therefore, for any $\psi\in \mathcal M$ that is the limit of flows in $\mathcal A_{\mathcal F}(T)$ for some $T<\infty$, we have deduced that $\psi \in \operatorname{BV}^2([0,1])$ by the Helly selection principle.
Therefore, we have shown that $\mathcal M$ is a subset of $\operatorname{BV}^2([0,1])$.

We then consider a map:
\begin{equation}
    \label{eq:map_alpha}
    \alpha:\mathcal M\to \operatorname{BV}([0,1]): \psi\mapsto \log \psi^\prime.
\end{equation}

    The local norm on $\mathcal M$ then induces a local norm, i.e. the pushforward, on $\operatorname{Im}(\alpha)\subset \operatorname{BV}([0,1])$. For given $\xi = \alpha(\psi)\in \operatorname{Im}(\alpha)$ and $h\in \operatorname{BV}([0,1])$, this norm is given by:
\begin{equation}
           \|h\|_\xi^{\operatorname{Im}(\alpha)} = \|(d \alpha^{-1}(\xi))h\|_{\psi} = \|\int_0^x e^{g(t)}h(t)dt\|_{\psi}\\ = \|\int_0^x f^\prime(t) h(t)dt\|_{\psi}\\  = \|h\|_{\operatorname{TV}[0,1]}
\end{equation}
With this induced metric, $\alpha$ gives an isometric embedding from $\mathcal M$ onto $\operatorname{Im}(\alpha)\subset \operatorname{BV}([0,1])$. Notice that under this embedding, the metric is independent of $\xi\in \mathcal N$, which means the metric is "flat" in $\mathcal N$.

Then, it is easy to check that the modified straight line:
\begin{equation}
        \gamma_t(x)= (1-t)\xi_1(x)+t\xi_2(x) - \log \int_0^1 e^{(1-t)\xi_1(x)+t\xi_2(x)} dx,
\end{equation}
is a path with minimal length connecting $\xi_1=\alpha(\psi_1)$ and $\xi_2=\alpha(\psi_2)$ in $\mathcal N.$

In the original space $\mathcal M$, this path is given by:
\begin{equation}
    \label{eq:geodesic_1d}
    \tilde \gamma_t(x)  = \frac{\int_0^x(\psi_1^\prime(x))^{1-t}(\psi_2^\prime(x))^t dx}{\int_0^1 (\psi_1^\prime(x))^{1-t}(\psi_2^\prime(x))^t dx}.
\end{equation}

Integrating the local norm along $\gamma_t$, we achieve the distance:
\begin{equation}
    \begin{aligned}
            d_{\mathcal N}(\xi_1, \xi_2) &= \int_0^1 \|\xi_2^\prime - \xi_1^\prime\|_{\gamma_t} dt = \int_0^1 \|\xi_2^\prime - \xi_1^\prime\|_{\operatorname{TV}[0,1]} dt\\
            &= \|\xi_2 - \xi_1\|_{\operatorname{TV}[0,1]}.
    \end{aligned}
\end{equation}
Going back to $\mathcal M$, we have
\begin{equation}
    d_{\mathcal M}(\psi_1, \psi_2) = \|\log \psi_2^\prime - \log \psi_1^\prime\|_{\operatorname{TV}[0,1]}.
\end{equation}
As a corollary, this result shows that for any $\psi_1, \psi_2\in \{\psi\in \operatorname{BV}^2([0,1]) \mid \psi(0)=0, \psi_1 =1,  \psi^\prime \ge c > 0 \text{ a.e.}\text{ for some } c\}$, $d_{\mathcal F}(\psi_1, \psi_2) < \infty$. Therefore, we have shown that
\begin{equation}
    \mathcal M = \{\psi\in \operatorname{BV}^2([0,1]) \mid \psi(0)=0, \psi_1 =1,  \psi^\prime \ge c > 0 \text{ a.e.}\text{ for some } c\}.
\end{equation}
This completes the proof.
\end{proof}

%% file: general_formulations.tex
\section{The rate of approximation by flows: the general case}

We now generalize the approach motivated in the ReLU example to
analyze the rate of approximation by flows driven by general
control families $\mathcal F$ in $d$ dimensions.
Before we introduce the formulation of the general concepts,
let us first summarize the approach in the previous section.
We considered the class $\mathcal M$ of maps that can be approximated, starting from the identity,
by flows generated by $\mathcal F$ in finite time. This class is has a compositional structure rather than a linear one.
We quantified the approximation complexity by the minimal reachable time
$\mathbf C_{\mathcal F}(\psi)$ by extending it to a distance
$d_{\mathcal F}$ on $\mathcal M$, the minimal time to transport one map to another via flows induced by $\mathcal F$.
To connect this global quantity to $\mathcal F$, we examined the infinitesimal behavior of
$d_{\mathcal F}$ and obtained a local norm $\|\cdot\|_{\cdot}$ which characterizes the
 complexity of transporting a function to its nearby functions.
This norm captures the “shallow” approximation capability of $\mathcal F$,
while the global transport cost emerges by integrating this norm along a curve.

As what we will develop in the following,
the correspondence between global distance and local infinitesimal cost is naturally expressed
in a sub-Finsler framework on $\operatorname{Diff}(M)$, the group of diffeomorphisms on $M$:
% \ql{is Diff M introduced somewhere?}
$\mathcal F$ induces a horizontal distribution (the admissible directions) and a fiberwise norm on it;
curves that follow the distribution are horizontal, and their lengths are the time-integrals of this norm.
The resulting geodesic distance (the length of shortest horizontal paths)
turns out to coincide with the minimal reachable time $d_{\mathcal F}$.
Intuitively, directions outside the distribution have infinite local norm,
so only horizontal motion is allowed.

In what follows, we formalize this viewpoint in a Banach sub-Finsler setting on $\operatorname{Diff}(M)$.
\Cref{sec:banach_finsler_geometry} introduces the basic objects (horizontal distribution, local norm, horizontal curves and length),
\Cref{sec:main_results_general} defines the associated geodesic distance and relates it to flow approximation,
and \Cref{subsec:proof_thm_51} proves the main theorem that identifies $d_{\mathcal F}$ with this sub-Finsler geodesic distance
and gives the variational characterization of the local norm.
This yields a practical recipe for estimating the complexity of approximating maps by flows:
estimate the local norm (linked to shallow approximations by $\mathcal F$),
design horizontal paths, and integrate the norm to bound $d_{\mathcal F}$ and $\mathbf C_{\mathcal F}$.
In some cases, one can also design optimal horizontal paths that completely characterizes
$d_{\mathcal F}$, and hence $\mathbf C_{\mathcal F}$.

\subsection{Preliminaries on Banach sub-Finsler geometry}
\label{sec:banach_finsler_geometry}
Since there are sometimes differing conventions for infinite-dimensional manifolds
and Finsler geometry, we concretely describe our adopted settings.

Banach manifold generalizes the notion of manifold to infinite dimensions.
Locally, an $n$-dimensional manifold looks like an open set in $\mathbb R^n$,
while a Banach manifold looks likes an open set in a Banach space $(X, \|\cdot\|)$.
We adopt the definition of Banach manifolds as follows~\cite{deimling2013nonlinear}:
\begin{definition}[Banach Manifold]
    \label{def:banach_manifold}
Let $\mathcal M$ be a topological space.
We call $\mathcal M$ a $C^r$ Banach (where $r$ is a positive integer or $\infty$)
% \ql{what is $\mathbb N^*$?}
manifold modeled on a Banach space $(X, \|\cdot\|)$ with codimension $0\le k<\infty$,
if there exists a collection of \emph{charts} $(U_i, \beta_i)$, $i \in I$, such that:

\begin{enumerate}
    \item $U_i\subset \mathcal M$ is open and $\mathcal M=\bigcup_{i \in I} U_i$;

    \item Each $\beta_i$ is a homeomorphism from $U_i$ onto a subspace $X_i\subset X$ with codimension $k$;

    \item The crossover map
    \begin{equation}
        \beta_j \circ \beta_i^{-1}  : \beta_i(U_i \cap U_j) \subset X_i \to \beta_j(U_i \cap U_j)\subset X_j
    \end{equation}
    is a $C^r$ function for every $i,j \in I$, in the sense of Fr\'echet derivatives.
\end{enumerate}
The collection of charts $(U_i, \beta_i)$ is called an \emph{atlas} of $\mathcal M$.
\end{definition}
\begin{example}
The simplest example of a Banach manifold is an open subset $\Omega$ of $X$. In this case, it is an $X$-manifold with codimension $0$, where the atlas contains only one chart $(\Omega, \operatorname{Id})$.
\end{example}

\begin{example}
  \label{ex:diff1_banach_manifold}
Another important example is the space of diffeomorphisms on a compact manifold $M$.
Let $M\subset\mathbb R^d$ be a compact $C^\infty$ manifold, and fix a $C^\infty$ Riemannian metric on $M$ with exponential map $\exp_x:T_xM\to M$.
Define
\begin{equation}
\operatorname{Diff}^1(M)
:=\{\psi:M\to M:\ \psi \text{ is $C^1$ and bijective, and }\psi^{-1}\in C^1(M,M)\}.
\end{equation}
We now describe a canonical family of local charts on $\operatorname{Diff}^1(M)$ using short geodesics. 

Fix $\psi\in \operatorname{Diff}^1(M)$. Consider the Banach space 
\begin{equation}
  \operatorname{Vec}^1(M):= \{u\in C^1(M, \mathbb R^d)\mid u(x)\in T_xM, \text{ for all } x\in M\}
\end{equation}
to be the set of $C^1$ vector fields on $M$ equipped with the norm in $C^1(M,\mathbb R^d)$.
For $\varepsilon>0$ small enough, define

\begin{equation}\label{eq:chart-diff1-vf}
\beta_\psi:\ B_{\operatorname{Vec}^1(M)}(0,\varepsilon)\ \longrightarrow\ C^1(M,M),\qquad
\beta_\psi(u)(x):=\exp_{\psi(x)}\big(u(\psi(x))\big).
\end{equation}

For $\varepsilon$ sufficiently small, $\beta_\psi(u)$ is a $C^1$ diffeomorphism, and its image
$U_\psi:=\beta_\psi(B_{\operatorname{Vec}^1(M)}(0,\varepsilon))$ is an open neighborhood of $\psi$ in
$\operatorname{Diff}^1(M)$(\cite{wittmann2019banach}).
Moreover, for $\phi\in U_\psi$ the inverse chart is given pointwise by
$w(x):=\exp_{\psi(x)}^{-1}(\phi(x))$, and the corresponding vector field is $v:=w\circ\psi^{-1}$.
In words, for each $x\in M$, we start at the point $\psi(x)$ and move along the unique geodesic with initial velocity $u(x)$ (which is small), and declare the endpoint to be $\beta_\psi(u)(x)$.

The collection of charts $\{(U_\psi,\beta_\psi^{-1})\}_{\psi\in\operatorname{Diff}^1(M)}$ forms an atlas, thereby endowing $\operatorname{Diff}^1(M)$ with the structure of a Banach manifold modeled on $C^1$ maps (codimension $0$). Transition maps $\beta_{\psi_2}^{-1}\circ \beta_{\psi_1}$ are $C^1$ in the Fr\'echet sense (actually $C^\infty$ according to~\cite{wittmann2019banach}), since they are obtained by composing the smooth finite-dimensional maps $(p,v)\mapsto \exp_p(v)$ and $(p,q)\mapsto \exp_p^{-1}(q)$ pointwise.

\end{example}

Similar to finite-dimensional manifolds, we can now define the tangent space at a point $x\in \mathcal M$.

% \lt{In general, a compact subset of $X$ can be also regarded as a Banach manifold, which resembles the sphere in finite dimension. How to formulate this? This seems related to our case.}

% \cjp{In~\cite{deimling2013nonlinear}, the author requires each chart is a homeomorphism on to a subspace $X_i\subset X$, where each $X_i$ has the same co-dimension $k$ (not sure if need to be finite). In this sense, they call $\mathcal M$ an $X$-manifold with codimension $k$. This formulation should be sufficient for the case where $\mathcal M$ is defined via the level set of some functional on $X$.}

\begin{definition}[Tangent Spaces of Banach Manifolds]
Let $\mathcal M$ be a $C^1$ Banach manifold modeled on $X$ and $x_0\in \mathcal M$. Define
\begin{equation}
    W_{x_0}:=\{\gamma \in C^1( E, \mathcal M) \mid E \text{ is an open interval containing } 0, \gamma(0)= x_0  \}
\end{equation}
be the set of all $C^1$ curves passing through $x_0$.  Let
\begin{equation}
    K_{x_0}:=\{\phi\in C^1(N(x_0),\mathbb R)\mid N(x_0) \subset \mathcal M \text{ is a neighborhood of } x_0, \varphi(x_0)=0\}
\end{equation}
be the set of all smooth functions vanishing at $x_0$. Define an equivalent relation on $W_{x_0}$ by $\gamma_1 \sim \gamma_2$ iff $(\phi\gamma_1)^\prime (0) = (\phi\gamma_2)^\prime (0)$ for all $\phi \in K_{x_0}$. An equivalence class $[\gamma]$ of this relationship is called a tangent vector at $x_0$. The set of such tangent vectors
\begin{equation}
    T_{x_0}\mathcal M:=\{[\gamma]\mid \gamma\in W_{x_0}\}
\end{equation}
forms a linear space, which is called the tangent space at $x_0$.
\end{definition}

Note that for finite-dimensional manifold $\mathcal M$ with dimension $n$,
the tangent space is isomorphic to $\mathbb R^n$. A corresponding result holds for Banach manifolds.

\begin{proposition}
    \label{prop:tangent_space}
    Let $\mathcal M$ be an $X$-Banach manifold with co-dimension $k$.
    For any $x_0\in \mathcal M$, the tangent space $T_{x_0}\mathcal M$ is isomorphic as a linear space  to a subspace $Y\subset X$ with $ \operatorname{codim} Y= k$ Specifically, the map
    \begin{equation}
        \Phi_{x_0}: T_{x_0}\mathcal M \to Y\subset X, \quad [\gamma] \mapsto (\beta\circ \gamma)^\prime(0),
    \end{equation}
is a linear bijection, where $(U,\beta)$ is a chart such that $x_0\in U$,
and $Y$ is the subspace corresponding to $\beta$ with codimension $k$.
\end{proposition}
\begin{proof}
    It is straightforward to see  $\Phi_{x_0}$ is well-defined and linear by the definition of tangent space. We then directly check the following:
    \begin{itemize}
        \item $\operatorname{Ker} \Phi_{x_0}=\{0\}$:
        If
        \begin{equation}
            \Phi_{x_0}([\gamma])=(\beta\circ \gamma)^\prime(0)=0,
        \end{equation}
        then for any $\phi\in K_{x_0}$, we have:
        \begin{equation}
            (\phi\circ \gamma)^\prime(0) = [(\phi\circ \beta^{-1})\circ (\beta\circ \gamma)]^\prime(0) = (\phi\circ \beta^{-1})^\prime (\beta\circ\gamma(0))\cdot (\beta\circ \gamma)^\prime(0)=0.
        \end{equation}
        By the definition of tangent vectors, we have $[\gamma]=0$.
        \item $\operatorname{Im} \Phi_{x_0} = Y$: For any $x\in X$, consider the path $\gamma(t):= \beta^{-1}(\beta(x_0)+tx)$, which is a smooth curve passing through $x_0$ with $\gamma^\prime(0)=x$. It then follows that
        \begin{equation}
            \Phi_{x_0}([\gamma]) = (\beta\circ \gamma)^\prime(0) = (\beta\circ \beta^{-1}(\beta(x_0)+tx))^\prime(0) = x.
        \end{equation}
    \end{itemize}
\end{proof}

Our goal is to model the diffeomorphisms over $M \subset \mathbb R^d$ as a Banach manifold, and study the ``travelling'' on this manifold via flows generated by a control family $\mathcal F$. Starting from the identity map, as we increase the time-horizon, we will be able to reach an increasingly complex set of diffeomorphisms.
The allowed directions of this travel are of course determined by $\mathcal F$.
In particular, we motivated earlier that only the directions $u$ for which the local norm~\eqref{eq:local_norm_relu}
is finite are allowed.
In the most general case, the allowed directions at each point does not fill the whole tangent space.
Rather, it is a suitably defined subspace which we call a \emph{distribution},
following the language in sub-Riemannian geometry.
We now formally introduce this and associated notions.
Here, as in the finite-dimensional case, the union of all tangent spaces
\begin{equation}
    T\mathcal M:=\bigcup_{x\in \mathcal M} T_x \mathcal M
\end{equation}
is called the tangent bundle of $\mathcal M$.

\begin{definition}[$C^r$ submersion between Banach manifolds]
\label{def:banach_submersion}
Let $\mathcal N,\mathcal M$ be $C^r$ Banach manifolds. A $C^r$ map $F:\mathcal N\to\mathcal M$ is a \emph{$C^r$ submersion at $p\in \mathcal N$} if the differential
\begin{equation}
\mathrm dF_p:T_p\mathcal N\to T_{F(p)}\mathcal M
\end{equation}
is a split surjection, i.e.\ surjective and admits a bounded right inverse (equivalently, $\ker \mathrm dF_p$ is a complemented closed subspace of $T_p\mathcal N$). If this holds for all $p\in \mathcal N$, we call $F$ a $C^r$ submersion.
\end{definition}

The definition of vector bundles and distributions can be generalized to Banach manifolds as follows.
Here, we follow the definitions in~\cite{arguillere2020sub}.

\begin{definition}[Banach vector bundle]
A $C^r$ Banach vector bundle over $\mathcal M$ is a triple $(\mathcal E,\pi,F)$ where:
% \ql{Again, should we use mathcal font for all Banach manifolds?}
\begin{itemize}
\item $\mathcal E$ and $\mathcal M$ are $C^r$ Banach manifolds, and $\pi:\mathcal E\to\mathcal M$ is a $C^r$ surjective submersion in the sense of Definition~\ref{def:banach_submersion};
\item Each fiber $\mathcal E_x:=\pi^{-1}(x)$ is a Banach space linearly isomorphic to a fixed model Banach space $F$;
\item There exists an open cover $\{U_i\}_{i\in \alpha}$ of $\mathcal M$ and $C^r$ bundle charts (local trivializations)
\begin{equation}
\tau_i:\ \pi^{-1}(U_i)\xrightarrow{\ \cong\ } U_i\times F
\end{equation}
that are fiberwise linear and such that the \emph{transition maps}
\begin{equation}
g_{ij}:U_i\cap U_j\to \operatorname{GL}(F),\qquad
\tau_i\circ\tau_j^{-1}(x,v)=(x,\,g_{ij}(x)\,v),
\end{equation}
are $C^r$ (for the operator-norm topology) and satisfy the cocycle relations
$g_{ii}=\operatorname{Id}_F$, $g_{ij}=g_{ji}^{-1}$, $g_{ik}=g_{ij}g_{jk}$ on triple overlaps.
\end{itemize}
In any trivialization, $\pi$ has the local form $(x,v)\mapsto x$, and $\mathrm d\pi_{(x,v)}(\xi,\eta)=\xi$.
\end{definition}

Classically, a distribution is specified by a subbundle $\mathcal D\subset T\mathcal M$ via the inclusion map.
In Banach manifold setting, this concept can also be generalized. Here, we follow the notion of \emph{relative tangent space} or \emph{anchored bundle}~\cite{arguillere2020sub}, which covers the subbudle case in a more general way.

\begin{definition}[Relative tangent space / Anchored bundle]\label{def:relative_tangent_space}
Let $\mathcal M$ be a smooth Banach manifold. A \emph{relative tangent space} on $\mathcal M$ is a triple $(\mathcal E,\pi,\rho)$ where
\begin{itemize}
    \item $\pi:\mathcal E\to \mathcal M$ is a smooth Banach vector bundle morphism with typical fiber a Banach space (denoted again by $F$);
    \item $\rho:\mathcal E\to T\mathcal M$ is a smooth vector-bundle morphism (the \emph{anchor}).
\end{itemize}
The associated horizontal distribution is $\mathcal D:=\rho(\mathcal E)\subset T\mathcal M$.
\end{definition}

Intuitively, an anchored bundle $(\mathcal E,\pi,\rho)$ over $\mathcal M$ assigns to each point $x\in \mathcal M$ a subspace of the tangent space $T_x\mathcal M$ through the anchor map $\rho$, which varies smoothly with $x$. In the context of our minimal reachable time problem, the distribution $\mathcal D$ is determined by the control family $\mathcal F$: it represents the set of all locally admissible velocities that can be generated by $\mathcal F$ at each point.

We equip each fiber $\mathcal E_x$ with a norm that depends continuously on $x$ induces a fiberwise norm on the image
$\mathcal D:=\rho(\mathcal E)\subset T\mathcal M$.
This generalizes the Riemannian/sub-Riemannian case (inner products) to a sub-Finsler setting (arbitrary norms).
If $\rho_x$ is injective, we simply transport the norm to $\mathcal D_x$; otherwise we take the minimal-norm preimage.

\begin{definition}[sub-Finsler structure]
Let $(\mathcal E,\pi,\rho)$ be an anchored $C^r$ Banach vector bundle over $\mathcal M$.
Assume $\mathcal E$ is endowed with a \emph{fiberwise norm field}, i.e.\ for each $x\in\mathcal M$ a Banach norm
$\|\cdot\|_{x}$ on the fiber $\mathcal E_x$, depending continuously on $x$ (equivalently, the map
$(x,v)\mapsto\|v\|_{x}$ is continuous on $\mathcal E$).
The induced sub-Finsler structure on the anchored distribution $\mathcal D:=\rho(\mathcal E)\subset T\mathcal M$ is the family of fiberwise norms:
\begin{equation}
\|\xi\|_{\mathcal D,x}\ :=\ \inf\big\{\ \|v\|_{x}\ :\ v\in \mathcal E_x,\ \rho_x v=\xi\ \big\},\qquad \xi\in T_x\mathcal M,
\end{equation}
with the convention $\|\xi\|_{\mathcal D,x}=+\infty$ if $\xi\notin\mathcal D_x$.
If each $\rho_x$ is injective, then $\|\xi\|_{\mathcal D,x}=\|\rho_x^{-1}\xi\|_{x}$ for $\xi\in\mathcal D_x$.
\end{definition}

With an anchored normed bundle in place, a curve is horizontal if its velocity lies in $\mathcal D$ almost everywhere,
and its length is the time integral of the local norm.
Minimizing length over horizontal curves yields the geodesic distance, exactly as in finite-dimensional sub-Riemannian geometry, now in the Banach setting.

\begin{definition}[Horizontal curves, length and distance.]
A curve $\gamma:[0,T]\to\mathcal M$ is absolutely continuous if it is absolutely continuous in local charts; then $\dot\gamma(t)$ exists for a.e.\ $t$ and lies in $T_{\gamma(t)}\mathcal M$. We call $\gamma$ horizontal if
\begin{equation}
\dot\gamma(t)\in\mathcal D_{\gamma(t)}\quad\text{for a.e.\ }t\in[0,T],
\end{equation}
equivalently, there exists a measurable \emph{control} $v(t)\in E_{\gamma(t)}$ with
\begin{equation}
\dot\gamma(t)=\rho_{\gamma(t)}\,v(t)\quad\text{a.e.}
\end{equation}
The \emph{length} of a horizontal curve is
\begin{equation}
L(\gamma)\ :=\ \int_0^T \|\dot\gamma(t)\|_{\mathcal D,\gamma(t)}\,dt\
\end{equation}
and $L(\gamma):=+\infty$ if $\gamma$ is not horizontal. Length is invariant under absolutely continuous reparametrizations.

The associated geodesic distance (Carnot--Carathéodory distance) on $\mathcal M$ is
\begin{equation}
d(x_0,x_1)\ :=\ \inf\{\,L(\gamma)\ :\ \gamma\text{ horizontal},\ \gamma(0)=x_0,\ \gamma(T)=x_1\,\}\in[0,+\infty].
\end{equation}
This is an extended metric on $\mathcal M$; when any two points can be joined by a horizontal curve of finite length, it is a genuine metric.
\end{definition}

The geodesic distance $d$ encodes precisely the same infinitesimal cost prescribed by the sub-Finsler structure: along admissible (horizontal) directions, the metric has a first-order expansion whose slope equals the local fiber norm; along non-admissible directions, the local cost is infinite.
Formally, the metric derivative of $t\mapsto d(x,\gamma(t))$ at $t=0$ recovers the sub-Finsler norm at $x$.
This ensures the consistency between the global path-length distance and the local norm structure.

\begin{proposition}[Recovering the local sub-Finsler norm from the geodesic distance]
\label{prop:local_gauge_from_distance}
Let $(\mathcal E,\pi,\rho)$ be an anchored Banach vector bundle over a $C^1$ Banach manifold $\mathcal M$, endowed with a continuous fiberwise norm $\|\cdot\|_{x}$, and let $(\mathcal D,\|\cdot\|_{\mathcal D})$ and $d$ be the induced sub-Finsler structure and geodesic distance. Fix $x\in\mathcal M$ and $v\in T_x\mathcal M$.
\begin{enumerate}
\item If $v\in \mathcal D_x$, then for any $C^1$ \emph{horizontal} curve $\gamma$ with $\gamma(0)=x$ and $\gamma'(0)=v$,
\begin{equation}
\|v\|_{\mathcal D,x} \;=\; \lim_{t\to 0}\frac{d(x,\gamma(t))}{t}.
\end{equation}
In particular, the limit exists and is independent of the chosen horizontal $\gamma$ with $\gamma'(0)=v$.

\item If $v\notin \mathcal D_x$, then for any $C^1$ curve $\gamma$ with $\gamma(0)=x$ and $\gamma'(0)=v$,
\begin{equation}
\lim_{t\to 0^+}\frac{d(x,\gamma(t))}{t} \;=\; +\infty.
\end{equation}
\end{enumerate}
\end{proposition}

\begin{proof}
We work in a local $C^1$ chart around $x$ so that $\mathcal M$ is identified with an open set of a Banach space $X$ and the anchored bundle $(\mathcal E,\pi,\rho)$ becomes trivial with a continuous field of norms $\|\cdot\|_y$ on the fiber $F$.
In these coordinates the sub-Finsler speed is
$\|\dot\gamma(t)\|_{\mathcal D,\gamma(t)}=\inf\{\|u(t)\|_{\gamma(t)}:\ \rho_{\gamma(t)}u(t)=\dot\gamma(t)\}$.
Local triviality and continuity imply uniform equivalence of the involved norms on a neighborhood of $x$.

\medskip\noindent
\textbf{(i) Horizontal $v\in\mathcal D_x$. }

\emph{Upper bound.}
Let $\gamma$ be horizontal, $C^1$, with $\gamma(0)=x$ and $\gamma'(0)=v$.
By definition of length and continuity of the fiber norms,
\begin{equation}
d(x,\gamma(t))\ \le\ L(\gamma|_{[0,t]})
=\int_0^t \|\dot\gamma(s)\|_{\mathcal D,\gamma(s)}\,ds
= t\,\|v\|_{\mathcal D,x}+o(t),
\end{equation}
hence $\displaystyle \limsup_{t\to 0^+}\frac{d(x,\gamma(t))}{t}\le \|v\|_{\mathcal D,x}$.

\smallskip\noindent
\emph{Lower bound.}
Fix a sequence $\{t_k\}_{k=1}^\infty \to 0^+$.
For each $k$, let $\sigma_k$ be a horizontal curve joining $x$ to $\gamma(t_k)$ with
$L(\sigma_k)\le d(x,\gamma(t_k))+\tfrac1k$.
Reparameterize each $\sigma_k$ on $[0,t_k]$ so that $\dot\sigma_k=\rho_{\sigma_k} u_k$ with $u_k\in L^1([0,t_k];F)$ and
$\int_0^{t_k}\|u_k(s)\|_{\sigma_k(s)}\,ds=L(\sigma_k)$.
Set the averaged controls
\begin{equation}
\bar u_k:=\frac1{t_k}\int_0^{t_k} u_k(s)\,ds\in F.
\end{equation}
By continuity of $\rho$ and the $C^1$ expansion of $\gamma$,
\begin{equation}
\frac{\gamma(t_k)-x}{t_k}
=\frac1{t_k}\int_0^{t_k}\dot\sigma_k(s)\,ds
=\frac1{t_k}\int_0^{t_k}\rho_{\sigma_k(s)}u_k(s)\,ds
=\rho_x \bar u_k + o(1)\quad (k\to\infty).
\end{equation}
Hence $\rho_x \bar u_k\to v$. By lower semicontinuity and the convexity of the fiber norm,
\begin{equation}
\liminf_{k\to\infty}\frac{d(x,\gamma(t_k))}{t_k}
\ \ge\ \liminf_{k\to\infty}\frac1{t_k}\int_0^{t_k}\|u_k(s)\|_{\sigma_k(s)}\,ds
\ \ge\ \liminf_{k\to\infty}\|\bar u_k\|_{x}
\ \ge\ \inf_{\rho_x u=v}\|u\|_{x}
\ =\ \|v\|_{\mathcal D,x}.
\end{equation}
Combining with the upper bound gives the limit and its value, independent of the chosen horizontal $\gamma$.

\medskip\noindent
\emph{(ii) Non-horizontal $v\notin\mathcal D_x$.}
Suppose by contradiction that $\liminf_{t\to 0^+} d(x,\gamma(t))/t<+\infty$ for some $C^1$ curve $\gamma$ with $\gamma(0)=x$, $\gamma'(0)=v$.
Then there exist $t_k\to 0^+$ and horizontal $\sigma_k$ from $x$ to $\gamma(t_k)$ with $L(\sigma_k)\le C\,t_k$.
Repeat the averaging argument above: $\rho_x \bar u_k\to v$ and $\|\bar u_k\|_{x}\le C$ along a subsequence.
Passing to the limit yields $v\in \rho_x(F)=\mathcal D_x$, a contradiction.
Hence $\lim_{t\to 0^+} d(x,\gamma(t))/t=+\infty$.

\medskip
This proves both statements.
\end{proof}

We have now introduced the Banach sub‑Finsler toolkit needed to study approximation complexity: Banach manifolds and charts, tangent bundles, anchored bundles (distributions), fiber‑wise norms (sub‑Finsler structures), horizontal curves, and the associated geodesic distance. This framework gives a geometric interpretation of the distance \(d_{\mathcal F}(\cdot,\cdot)\): it coincides with the geodesic distance induced by a right‑invariant sub‑Finsler structure on the Banach manifold of \(W^{1,\infty}\)-diffeomorphisms.

In the next subsection we make this relation precise by identifying the relevant Banach manifold, the anchored bundle and its fiber‑wise norm, and by stating and proving the main theorem that characterizes \(d_{\mathcal F}\) in terms of the sub‑Finsler geometry developed above.

\subsection{General characterization of approximation complexity via Finsler geometry}
\label{sec:main_results_general}
We now adopt the Banach Finsler geometry framework to study the approximation complexity for
a general control family $\mathcal F$.
Specifically, let $M \subset \mathbb R^d$ be a compact smooth manifold with or without boundary,
and $\mathcal F \subset \operatorname{Vec}(M)$ be a control family.

We assume that $\mathcal F$ is uniformly bounded in the $W^{1,\infty}$ norm, i.e.
\begin{equation}
  \sup_{f\in\mathcal F}\|f\|_{W^{1,\infty}(M,\mathbb R^d)}<\infty.
\end{equation}
Define the $s$-scaled convex hull of $\mathcal F$ as
\begin{equation}
    \label{eq:CH_def}
\mathbf{CH}_s(\mathcal F)
:= \Big\{\sum_{i=1}^N a_i f_i : f_i\in \mathcal F,\ \sum_{i=1}^N |a_i| \le s \Big\}.
\end{equation}
Then, we consider the extended norm on $\operatorname{Vec}(M)$ defined as:
\begin{equation}
\|v\|_{\mathcal F} := \inf\{s\ge 0 : v \in \overline{\mathbf{CH}_s(\mathcal F)}^{C^0}\},
\end{equation}
with the convention $\|v\|_{\mathcal F}=+\infty$ if $v\notin \overline{\operatorname{span}\mathcal F}^{C^0}$. 
The norm $\|\cdot\|_{\mathcal F}$ is naturally finite on $\operatorname{span}\mathcal F$. Here, the closure is taken in the $C^0(M,\mathbb R^d)$ topology.
 We denote by $X_{\mathcal F}$ the closure of $\operatorname{span}\mathcal F$ under the $\|\cdot\|_{\mathcal F}$ norm.
% on $\operatorname{span}\mathcal F$ the atomic norm
% \begin{equation}
% \|v\|_{\mathcal F}:=\inf\Big\{\sum_{i=1}^N |a_i|:\ v=\sum_{i=1}^N a_i f_i,\ f_i\in\mathcal F\Big\},
% \qquad X_{\mathcal F}:=\overline{\operatorname{span}\mathcal F}^{\ \|\cdot\|_{\mathcal F}}.
% \end{equation}
% By definition, $( X_{\mathcal F}, \|\cdot \|_{\mathcal F})$ is a Banach space. 
Since $\mathcal F$ is uniformly bounded in $W^{1,\infty}$, $\mathbf{CH}_s(\mathcal F)$ is also uniformly bounded in $W^{1,\infty}$ for each $s$. Since the $C^0$ closure of a uniformly $W^{1,\infty}$-bounded set is still uniformly $W^{1,\infty}$-bounded, we have $X_{\mathcal F} \subset W^{1,\infty}(M,\mathbb R^d)$.
% there exists a continuous embedding $\iota: X_{\mathcal F}\hookrightarrow W^{1,\infty}(M,\mathbb R^d)$ such that $\iota(f)=f$ for all $f\in\operatorname{span}\mathcal F$. In the following, we identify $X_{\mathcal F}$ with its image $\iota(X_{\mathcal F})\subset W^{1,\infty}(M,\mathbb R^d)$, and view $X_{\mathcal F}$ as a Banach space of vector fields on $M$.

We make the following definition on the regularity of $\mathcal F$ and the choice of the base Banach manifold $\mathcal M$ on which we place the sub-Finsler structure.

\begin{definition}[Compatible pair]\label{def:compatible-pair}
Let $\mathcal M\subset \operatorname{Diff}(M)$ and $\mathcal F\subset \operatorname{Vec}(M)$.
We say that $(\mathcal M,\mathcal F)$ is a \emph{compatible pair} if the following hold:
\begin{enumerate}
\item \textup{($C^1$ Banach manifold structure via exponential charts)}
$\mathcal M$ carries a $C^1$ Banach manifold structure modeled on a subspace $X\subset \operatorname{Vec}(M)$ whose local charts are given by
Riemannian exponential maps (as in \Cref{ex:diff1_banach_manifold}).

\item \textup{(right translation is $C^1$)}
For each $\psi\in\mathcal M$, the right-translation map
$R_\psi:\mathcal M\to \mathcal M$, $R_\psi(\eta)=\eta\circ\psi$,
is of class $C^1$.

\item \textup{(admissible velocities)}
For every $\psi\in\mathcal M$ and $f\in X_{\mathcal F}$, $f\circ\psi\in T_\psi\mathcal M$.

\item \textup{(Invariance of $\mathcal M$ under Carath\'eodory controls)} For any measurable $u(\cdot):[0,T]\to X_{\mathcal F}$, the ODE
\begin{equation}
\dot\gamma(t)=u(t)\circ\gamma(t),\qquad \gamma(0)=\operatorname{Id},
\end{equation}
admits a unique absolutely continuous solution on $[0,T]$, and $\gamma(t)\in\mathcal M$ for all $t\in[0,T]$.
\end{enumerate}
\end{definition}
The conditions in the above definition are quite natural and are satisfied in many settings of interest, including the one considered in~\cref{sec:1d_relu}. In particular, the following example illustrates a typical and general class of compatible pairs.

\begin{example}\label{rem:good-pair-examples}
% \item If $\mathcal M=\operatorname{Diff}^1(M)$ and $\mathcal F\subset \operatorname{Vec}^1(M)$ is uniformly bounded in $C^1$,
% then the conditions 1--3 follow from the standard Banach manifold structure on $\operatorname{Diff}^1(M)$ and the $C^1$ dependence of flows.

% Moreover, \textup{(4)} holds as well. Indeed, for any measurable $u(\cdot):[0,T]\to X_{\mathcal F}$ we have
% $\|u(t)\|_{C^1}\lesssim \|u(t)\|_{\mathcal F}$ and hence $u\in L^1([0,T];C^1)$.
% The Carath\'eodory ODE $\dot\gamma(t)=u(t)\circ\gamma(t)$ therefore has a unique absolutely continuous solution,
% and differentiating in space yields the variational equation for $D\gamma$ with Gr\"onwall bounds.
% In particular each $\gamma(t)$ is a $C^1$ diffeomorphism with $C^1$ inverse, hence $\gamma(t)\in\operatorname{Diff}^1(M)$ for all $t$.
A typical type of compatible pairs is the following.

Suppose $M\subset \mathbb R^d$ is the closure of a bounded open set, $\mathcal M=\operatorname{Diff}^{W^{1,\infty}}(M)$
consists of $W^{1,\infty}$ diffeomorphisms fixing the boundary, and $\mathcal F\subset W^{1,\infty}(M,\mathbb R^d)$ such that $f$ vanishes on the boundary for all $f\in\mathcal F$. Then $(\mathcal M,\mathcal F)$ is a compatible pair.

Indeed, \textup{(1)} holds with model space
\begin{equation}
X:=\{u\in W^{1,\infty}(M,\mathbb R^d):\ u|_{\partial M}=0\},
\end{equation}
and the local charts around $\psi\in\mathcal M$ can be chosen as the affine maps
\begin{equation}
\beta_\psi:\ B_X(0,\varepsilon)\to \mathcal M,\qquad \beta_\psi(u):=(\operatorname{Id}+u)\circ\psi,
\end{equation}
for $\varepsilon>0$ sufficiently small.
Since $u|_{\partial M}=0$ and $\psi$ fixes $\partial M$, we have $\beta_\psi(u)$ also fixes $\partial M$; moreover for $\varepsilon$ small,
$\operatorname{Id}+u$ is bi-Lipschitz on $M$, hence $\beta_\psi(u)\in \operatorname{Diff}^{W^{1,\infty}}(M)$.
In these charts, all transition maps are smooth (in fact affine), so $\mathcal M$ is a $C^1$ Banach manifold.

\textup{(2)} is immediate: right translation is composition on the right,
$R_\varphi(\eta)=\eta\circ\varphi$, and in the above charts it corresponds to the affine map
$u\mapsto u\circ\varphi$, which is $C^1$ as a map $X\to X$ (composition by a fixed $W^{1,\infty}$ diffeomorphism is bounded on $W^{1,\infty}$).

\textup{(3)} holds because $f|_{\partial M}=0$ for all $f\in\mathcal F$ implies $X_{\mathcal F}\subset X$,
and hence for any $\psi\in\mathcal M$ we have $X_{\mathcal F}\circ\psi\subset X\circ\psi=T_\psi\mathcal M$ under the above charts.

Finally, \textup{(4)} follows from standard well-posedness and stability results for Carath\'eodory ODEs with Lipschitz vector fields.
Indeed, for a measurable $u(\cdot):[0,T]\to X_{\mathcal F}$ we have $\|Du(t)\|_{L^\infty}\lesssim \|u(t)\|_{\mathcal F}$, hence
$t\mapsto \|Du(t)\|_{L^\infty}$ is integrable on $[0,T]$.
Existence and uniqueness of an absolutely continuous solution to $\dot\gamma(t)=u(t)\circ\gamma(t)$ then follow from the Carath\'eodory theory.
Moreover, Gr\"onwall's inequality gives a uniform bi-Lipschitz bound on $\gamma(t)$ (and on its inverse by time-reversal),
and the condition $u(t)|_{\partial M}=0$ implies that boundary points are stationary trajectories.
Therefore $\gamma(t)$ fixes $\partial M$ and belongs to $\operatorname{Diff}^{W^{1,\infty}}(M)=\mathcal M$ for all $t\in[0,T]$.

In fact, the 1D example considered in~\cref{sec:1d_relu} is a special case of this setting, where $M=[0,1]$ and $\mathcal F$ consists of ReLU-type controls vanishing at the boundary.
\end{example}

% \begin{enumerate}
% \item[\textbf{(I)}] \textbf{$C^1$ control families.}
% If $\mathcal F\subset \operatorname{Vec}^1(M)$, then flows generated by piecewise-constant controls in $\mathcal F$ are $C^1$ diffeomorphisms.
% In this case we take as base manifold
% \begin{equation}
%   \mathcal M:=\operatorname{Diff}^1(M),
% \end{equation}
% endowed with its standard Banach manifold structure as discussed in~\Cref{ex:diff1_banach_manifold}.

% \item[\textbf{(II)}] \textbf{$W^{1,\infty}$ control families.}
% To cover architectures such as ReLU-type controls, we allow $\mathcal F\subset \operatorname{Vec}(M)\cap W^{1,\infty}(M,\mathbb R^d)$.
% In this case, the natural ambient diffeomorphism group is
% \begin{equation}
%   \mathcal M:=\operatorname{Diff}^{W^{1,\infty}}(M),
% \end{equation}
% and we assume that $\operatorname{Diff}^{W^{1,\infty}}(M)$ admits a $C^1$ Banach manifold structure  induced by local charts defined via the exponential map as defined in~\Cref{ex:diff1_banach_manifold}.\footnote{In the examples considered in this paper (e.g.\ on $[0,1]$ and on $S^1$), this $C^1$ manifold structure can be verified directly in the $W^{1,\infty}$ topology.}
% \end{enumerate}

% \medskip

% In both settings, we will denote by $\mathcal M$ the chosen base Banach manifold, and all geometric objects (charts, tangent directions, horizontal distributions, lengths, and geodesic distances) will be defined on $\mathcal M$. We denote $X$ as the Banach space that $\mathcal M$ modeled on.
% When needed, we will explicitly distinguish between cases \textbf{(I)} and \textbf{(II)}.

We now consider a target space $\mathcal T$ consisting of maps that can be approximated by the flows driven with $\mathcal F$ within a finite time:
\begin{equation} \mathcal T := \{\psi \in \mathcal M \mid \mathbf{C}_{\mathcal F}(\psi) < \infty\}.
\end{equation}
% Based on the property of the control family, we have $\mathcal M\subset \operatorname{Diff}^{W^{1,\infty}}(M)$, which is a Banach manifold modeled on $X:=W^{1,\infty}(M,\mathbb R^d)$.
The complexity measure $\mathbf C_{\mathcal F}$(defined in~\eqref{eq:complexity_def}) can be extended to a distance function $d_{\mathcal F}$ on $\mathcal T$: \begin{equation}
  d_{\mathcal F}(\psi_1, \psi_2): = \inf\{T>0\mid \psi_1\in \overline{\mathcal A_{\mathcal F}(T)\circ \psi_2}^{C^0}\},
\end{equation}
In fact, this distance can be generalized to an extended metric on $\mathcal M$, by defining $d_{\mathcal F}(\psi_1, \psi_2):= +\infty$ if $\psi_1\notin \overline{\mathcal A_{\mathcal F}(T)\circ \psi_2}^{C^0}$ for all $T>0$.
With this extended distance, $\mathcal T$ can be viewed as the connected component of the identity in the extended metric space $(\mathcal M, d_{\mathcal F})$. We then show that $\mathcal M$ can be endowed with a sub-Finsler structure such that the associated geodesic distance coincides with $d_{\mathcal F}$.

Consider the anchored Banach bundle $(\mathcal E, \pi, \rho)$ over $\mathcal M$ defined as:
\begin{equation}
  \mathcal E:=\mathcal M\times  X_{\mathcal F},\qquad \pi(\psi,v):=\psi,
\quad
\rho(\psi,v):=v\circ \psi,
\text{ for all }(\psi,v)\in \mathcal E.
\end{equation}
Its image defines the horizontal distribution
\begin{equation}
\mathcal D_\psi:=\rho_\psi( X_{\mathcal F})= X_{\mathcal F}\circ\psi\ \subset\ T_\psi\mathcal M,
\end{equation}
endowed with the fiberwise norm
\begin{equation}
  \label{eq:sub-finsler-norm-def}
\|u\|_{\psi}:=\|u\circ\psi^{-1}\|_{\mathcal F},\qquad u\in\mathcal D_\psi,
\quad
\|u\|_{\psi}=+\infty\ \text{if }u\notin\mathcal D_\psi.
\end{equation}
This gives a
sub-Finsler structure $(\mathcal D,\|\cdot\|_{\mathcal D})$ on $\mathcal M$.

With these preparations, we can give detailed explanations of the main theorem stated in~\Cref{thm:main}.
% \begin{theorem}[Precise statement of ~\Cref{thm:main}]\label{thm:general_results}
If we assume $(\mathcal M, \mathcal F)$ is a compatible pair of a control family $\mathcal F$ and a base manifold $\mathcal M$ as defined in Definition~\ref{def:compatible-pair}.
Then, the following statements hold:
\begin{enumerate}
\item The maps $\|\cdot\|_{ \psi}: T_\psi\mathcal M\to[0,\infty]$ defined in~\eqref{eq:sub-finsler-norm-def} gives a sub-Finsler structure  on the distribution $\mathcal D$ over $\operatorname{Diff}^{1}(M)$. 
% Moreover, the local sub-Finsler norm admits the variational characterization
% \begin{equation}
% \|u\|_{\psi} = \inf\{\,s>0\mid u\in \overline{\mathbf{CH}_s(\mathcal F)}\circ \psi\,\},
% \end{equation}
% with the convention that $\|u\|_{\psi}=+\infty$ if $u\notin\mathcal D_\psi$.
\item  For any $\psi_1,\psi_2\in\mathcal M$,
\begin{equation}
  \label{eq:CC_dist}
d_{\mathcal F}(\psi_1,\psi_2)=\inf\Big\{\ \int_0^1 \|\dot\gamma(t)\|_{\gamma(t)}\,dt\ :\ \gamma\ \text{horizontal},\ \gamma(0)=\psi_1,\ \gamma(1)=\psi_2\ \Big\}.
\end{equation}
That is, the geodesic distance associated with the sub-Finsler structure $(\mathcal D,\|\cdot\|_{\mathcal D})$ coincides with $d_{\mathcal F}$ on $\mathcal M$.
% \item
% For any $\psi\in\mathcal M$ and $u\in T_\psi\operatorname{Diff}(M)$,
% \[
% \|u\|_{\mathcal D,\psi}
% =
% \begin{cases}
% \displaystyle \lim_{t\downarrow 0}\dfrac{d_{\operatorname{CC}}(\psi,\gamma(t))}{t},
% & \text{if }u\in\mathcal D_\psi \text{ and }\gamma \text{ is any horizontal }C^1\text{ curve with }\gamma(0)=\psi,\ \gamma'(0)=u,\\[2.2ex]
% +\infty, & \text{if }u\notin\mathcal D_\psi.
% \end{cases}
% \]
% Thus the infinitesimal gauge recovered from $d_{\operatorname{CC}}$ coincides with the prescribed sub-Finsler fiber norm on $\mathcal D$.

% \item[(iii)] \textbf{Variational characterization on $\mathcal M$.}
% For all $\psi_1,\psi_2\in\mathcal M$,
% \[
% d_{\mathcal F}(\psi_1,\psi_2)
% =d_{\operatorname{CC}}(\psi_1,\psi_2)
% =\inf\Big\{\ \int_0^1 \|\dot\gamma(t)\|_{\mathcal D,\gamma(t)}\,dt\ :\ \gamma \text{ absolutely continuous, horizontal},\ \gamma(0)=\psi_1,\ \gamma(1)=\psi_2\ \Big\}.
% \]
\end{enumerate}

% \end{theorem}

\begin{remark}
  This distance defined by the right hand side in~\eqref{eq:CC_dist} is typically called the Carnot-Carathéodory in the literature of sub-Riemannian geometry. We adopt the name "geodesic distance" for simplicity here. 

  Moreover, when  $\mathcal D_\psi = T_\psi\mathcal M$ for all $\psi\in \mathcal M$, i.e. $\mathcal F$ spans the whole tangent space, the map $\|\cdot\|_\psi$ gives a Finsler structure on $\mathcal M$. In this case, $\mathcal M$ becomes a Finsler manifold, and $d_{\mathcal F}$ becomes an actual geodesic distance.
\end{remark}
% \ql{Add a remark saying: 1) recall here strictly it is CC distance, but we call it geodesic. 2) Say in which case does this become a actual manifold and actual geodesic.}

This theorem offers a general geometric perspective to understand the approximation complexity of diffeomorhisms
induced by a given control family $\mathcal F$.
In particular, the local norm $\|\cdot\|_{\psi}$
characterizes the local complexity of transporting a function to its nearby functions, while the global distance $d_{\mathcal F}(\cdot, \cdot)$ quantifies the overall complexity of transporting one function to another via the shortest path integral of the local norm along curves connecting the two functions.

In general, a closed-form characterization of the distance $d_{\mathcal F}(\cdot, \cdot)$ on $\mathcal M$ may not be available. However, this theorem still provides a feasible approach to estimate it. Specifically, we can follow the steps below:
\begin{itemize}
    \item First, we estimate the local norm $\|\cdot\|_{\psi}$ at different $\psi$ in $\mathcal M$.
    This relates to the approximation complexity of of each shallow layer of neural networks,
    which has been widely studied in the literature~\cite{ma2022barron,siegel2020approximation,siegel2023characterization}.
    \item Next, we construct proper horizontal paths connecting the two target functions in $\mathcal M$. This step may require insights into the structure of $\mathcal M$ and the dynamics induced by $\mathcal F$.
    \item Finally, we compute or estimate the path integral of the local norm along these paths to obtain upper bounds on $d_{\mathcal F}(\cdot, \cdot)$.
\end{itemize}
From the deep learning perspective, the distance $d_{\mathcal F}$ measures the idealized depth of network to approximate a target function with layer-wise architecture induced by $\mathcal F$. Therefore, this theorem provides a general framework to identify what kind of target functions are more efficiently approximated by deep networks compared to shallow ones, and how this efficiency depends on the choice of the control family $\mathcal F$.  Roughly speaking, if the target function can be connected to the identity map via a path passing through functions with small local norms, then it can be efficiently approximated by deep networks.
For some control families, as we discuss later in~\Cref{sec:applications}, explicit characterisations or estimates can be obtained.

Let us now place this geometric framework in the broader context of approximation theory for neural networks.
Classical results for shallow networks focus on universal approximation and Jackson-type estimates for one-hidden-layer architectures, where approximation complexity is typically quantified in terms of Sobolev, Besov, or Barron-type norms of the target function~\cite{cybenko1989approximation, hornik1991approximation, barron2002universal, ma2022barron,siegel2023characterization,siegel2020approximation}.
In contrast, the theory of deep networks must account for the compositional structure.
One line of work studies discrete–depth architectures directly, deriving explicit approximation bounds for specific architectures via constructive methods~\cite{lu2021deep,shen2019deep,shen2022optimal,shen2021neural, yarotsky2017error,yarotsky2018optimal}.
A second line of work, closer to our approach, idealizes residual networks~\cite{he2016deep} as continuous–time systems and analyses the associated flow maps, viewing depth as a time horizon for an ODE or control system~\cite{weinan2017proposal,ruthotto2020deep,li2018maximum}.
Within this flow-based perspective, the universal approximation results for a very broad class of architectures have been established~\cite{agrachev2022control,li2022deep,cheng2025interpolation,ruiz2023neural,tabuada2022universal,cheng2025unified}. Approximation complexity estimates have also been studied for specific architectures~\cite{geshkovski2024measure,ruiz2023neural, li2022deep}, often by explicit constructions.
Although the passage from discrete layers to flows is an idealization, it preserves the key distinguishing feature of
deep models, where complexity generated by function composition, while
making available tools from dynamical systems and control theory.
Our work contributes to this flow-based perspective by identifying a geometric structure on the complexity class defined by the flow approximation problem, and showing that concrete estimates can be derived within this structure.
This geometry is intrinsically non-linear and differs from the linear space setting of classical approximation theory, and we believe it provides a reasonable way to understand the complexity induced by function compositions.

In the rest of this part, we present the proof of ~\Cref{thm:main}.
\label{subsec:proof_thm_51}

We first show that the fiberwise norm in \eqref{eq:sub-finsler-norm-def} defines a sub-Finsler structure
on the right-invariant distribution induced by $X_{\mathcal F}$.
Recall the anchored bundle
\begin{equation}
\mathcal E:=\mathcal M\times X_{\mathcal F}
\xrightarrow{\ \pi\ }\mathcal M,
\qquad
\pi(\psi,v)=\psi,
\qquad
\rho(\psi,v)=\iota(v)\circ\psi .
\end{equation}
By compatibility condition~\textup{(3)} in Definition~\ref{def:compatible-pair}, the image of $\rho$ defines a distribution
\begin{equation}
\mathcal D_\psi:=\rho_\psi(X_{\mathcal F})=X_{\mathcal F}\circ\psi
\subset T_\psi\mathcal M,\qquad \psi\in\mathcal M.
\end{equation}
We endow $\mathcal D$ with the fiberwise (extended) norm
\begin{equation}\label{eq:fiber-norm}
\|\xi\|_{\psi}:=
\begin{cases}
\|\xi\circ\psi^{-1}\|_{\mathcal F}, & \xi\in\mathcal D_\psi,\\
+\infty, & \xi\notin\mathcal D_\psi.
\end{cases}
\end{equation}
Equivalently, if $\xi=v\circ\psi$ for some $v\in X_{\mathcal F}$, then $\|\xi\|_\psi:=\|v\|_{\mathcal F}$.
This is well-defined because right composition by $\psi$ is a bijection on its image.
For each fixed $\psi$, $\xi\mapsto \|\xi\|_\psi$ is a norm on $\mathcal D_\psi$, since
$v\mapsto v\circ\psi$ is a linear isomorphism $X_{\mathcal F}\to \mathcal D_\psi$.

It remains to check the local regularity required of a sub-Finsler structure in charts.
Let $(U_\psi,\beta_\psi)$ be a $C^1$ exponential chart of $\mathcal M$ around $\psi$, modeled on a Banach space $X$.
Since $(\mathcal M,\mathcal F)$ is a compatible pair, right translation
$R_\varphi:\eta\mapsto \eta\circ\varphi^{-1}$ is $C^1$ on $\mathcal M$ for $\varphi$ in a neighborhood of $\psi$.
In particular, in the chart $(U_\psi,\beta_\psi)$ the map
\begin{equation}
(\varphi,\xi)\ \longmapsto\ (dR_{\varphi})_{\varphi}[\xi]
\end{equation}
is continuous, and (by construction of the right-invariant distribution) this continuity is exactly what
is needed to ensure that $\|\cdot\|_\varphi$ varies continuously with $\varphi$ on the distribution $\mathcal D_\varphi$.
More concretely, for $\xi\in\mathcal D_\varphi$ we can write $\xi=v\circ\varphi$ with $v\in X_{\mathcal F}$, and then
\begin{equation}
\|\xi\|_\varphi=\|v\|_{\mathcal F},
\end{equation}
so the only dependence on $\varphi$ is through the identification of $\mathcal D_\varphi$ with $X_{\mathcal F}$ via right translation,
which is $C^1$ by condition~\textup{(2)}.
This establishes that $\{\|\cdot\|_\psi\}_{\psi\in\mathcal M}$ defines a sub-Finsler structure on $\mathcal D$.

\medskip

We then show the equality of $d_{\mathcal F}$ and the geodesic distance.
Let $d(\cdot,\cdot)$ denote the geodesic distance induced by the sub-Finsler norm, i.e.
\begin{equation}
d(\psi_1,\psi_2):=\inf\Big\{\int_0^1\|\dot\gamma(t)\|_{\gamma(t)}\,dt:\ \gamma \text{ horizontal},\ \gamma(0)=\psi_2,\ \gamma(1)=\psi_1\Big\}.
\end{equation}

We first show that $d\le d_{\mathcal F}$. We will use the following lemma, which is a Young measure-type compactness result~\cite{ball2005version} for measurable functions taking values in a compact set of vector fields. 

\begin{lemma}
\label{lem:relaxed-compactness}
Let $K$ be a compact metric space and let $u_n:[0,T]\to K$ be a sequence of measurable maps.
Then there exist a subsequence $u_{n_k}$  and a measurable family
$\nu_t\in\mathcal P(K)$ such that for every $\Phi\in C(K,\mathbb R)$ and every
$\varphi\in L^\infty([0,T], \mathbb R)$,
\begin{equation}
\label{eq:relaxed-compactness}
\int_0^T \varphi(t)\,\Phi(u_{n_k}(t))\,dt
\;\longrightarrow\;
\int_0^T \varphi(t)\Big(\int_K \Phi(v)\,d\nu_t(v)\Big)\,dt.
\end{equation}
Moreover, if $K$ is convex and closed in a locally convex vector space
(in particular in $C^0(M)$), then the barycenter
\begin{equation}
\label{eq:barycenter}
u(t):=\int_K v\,d\nu_t(v)
\end{equation}
belongs to $K$ for a.e.\ $t$.
\end{lemma}

\begin{proof}
Define probability measures $\mu_n\in\mathcal P([0,T]\times K)$ by
\[
\int \zeta(t,v)\,d\mu_n(t,v):=\frac1T\int_0^T \zeta\big(t,u_n(t)\big)\,dt,
\qquad \forall \zeta\in C([0,T]\times K).
\]
Since $[0,T]\times K$ is compact, $\{\mu_n\}$ is tight and hence relatively compact
in the weak topology; extract $\mu_n\Rightarrow \mu$.
The time-marginal of each $\mu_n$ is $dt/T$, hence the same holds for $\mu$.
By disintegration there exists a measurable family $\nu_t\in\mathcal P(K)$ such that
$d\mu(t,v)=\frac1T\,dt\,d\nu_t(v)$, which yields \eqref{eq:relaxed-compactness}.
If $K$ is convex and closed in a locally convex space, then \eqref{eq:barycenter}
lies in $K$ by Jensen/barycenter properties of convex closed sets.
\end{proof}

Fix $\psi_1,\psi_2\in\mathcal M$.
By applying a right translation,
it suffices to prove the claim for $(\psi_1\circ\psi_2^{-1},\mathrm{Id})$.
Thus, we assume $\psi_2=\mathrm{Id}$ below.

Let $T>d_{\mathcal F}(\psi_1,\mathrm{Id})$. By definition of $d_{\mathcal F}$,
there exists a sequence $\xi_n\in \mathcal A_{\mathcal F}(T)\circ \mathrm{Id}$ with
\begin{equation}
\label{eq:xi_n_to_psi1}
\xi_n\to \psi_1 \quad\text{in } C^0(M).
\end{equation}
For each $n$, choose a piecewise-constant control $u_n:[0,T]\to \mathbf{CH}_1(\mathcal F)$
driving a curve $\gamma_n:[0,T]\times \mathcal M\to \mathcal M$ such that
\begin{equation}
\label{eq:gamma_n_integral}
\gamma_n(t,x)=x+\int_0^t u_n(s,\gamma_n(s,x))\,ds,
\qquad \gamma_n(0,\cdot)=\mathrm{Id},\qquad \gamma_n(T,\cdot)=\xi_n.
\end{equation}
Reparametrize if necessary so that
\begin{equation}
\label{eq:unit-speed}
\|u_n(t)\|_{\mathcal F}\le 1 \quad\text{for a.e. }t\in[0,T].
\end{equation}
Set $K:=\overline{\mathbf{CH}_1(\mathcal F)}^{\,C^0}$.
Since $\mathcal F$ is uniformly bounded in $W^{1,\infty}$, $K$ is compact in $C^0(M)$ according to the Arzel\`a--Ascoli theorem.
By the Gr\"onwall inequality, each $\gamma_n(t,\cdot)$ is uniformly Lipschitz in $x$,
and the family $\{\gamma_n\}$ is equicontinuous in $t$ with respect to the $C^0$-metric.
By Arzel\`a--Ascoli, after extracting a subsequence,
\begin{equation}
\label{eq:gamma_n_to_gamma}
\gamma_n \to \gamma \quad\text{in } C^0([0,T]\times M),
\end{equation}
for some continuous $\gamma$ with $\gamma(0,\cdot)=\mathrm{Id}$ and
$\gamma(T,\cdot)=\psi_1$.

By the compactness of $K$, we can
apply Lemma~\ref{lem:relaxed-compactness} to the measurable maps $u_n:[0,T]\to K$.
to obtain a measurable map $t\to \nu_t\in\mathcal P(K)$ and define the barycenter
\begin{equation}
\label{eq:def_u}
u(t):=\int_K v\,d\nu_t(v)\in K,\qquad \|u(t)\|_{\mathcal F}\le 1\ \text{ for a.e. }t.
\end{equation}
The inclusion in $K$ uses that $K$ is convex.

Fix $x\in M$ and $t\in[0,T]$.
From \eqref{eq:gamma_n_integral} we write
\[
\gamma_n(t,x)-x=\int_0^t u_n(s,\gamma_n(s,x))\,ds.
\]
We claim that
\begin{equation}
\label{eq:integral_convergence}
\int_0^t u_n(s,\gamma_n(s,x))\,ds \;\longrightarrow\; \int_0^t u(s,\gamma(s,x))\,ds.
\end{equation}
Indeed, decompose
\[
\int_0^t \!\!\Big[u_n(s,\gamma_n(s,x))-u_n(s,\gamma(s,x))\Big]ds
+\int_0^t \!\!\Big[u_n(s,\gamma(s,x))-u(s,\gamma(s,x))\Big]ds
=:(I)_n+(II)_n.
\]
For $(I)_n$, use the uniform Lipschitz bound on $u_n(s,\cdot)$ and \eqref{eq:gamma_n_to_gamma}:
\[
|(I)_n|\le \int_0^t L\,\|\gamma_n(s,\cdot)-\gamma(s,\cdot)\|_{C^0}\,ds\;\longrightarrow\;0.
\]
For $(II)_n$,  by uniform continuity of $s\mapsto \gamma(s,x)$, choose a partition
$0=t_0<\dots<t_m=t$ such that
\begin{equation}
\sup_{s\in[t_{j-1},t_j]} \|\gamma(s,x)-\gamma(t_{j-1},x)\|\le \delta,
\quad j=1,\dots,m,
\end{equation}
with $\delta>0$ to be chosen momentarily. Set $y_j:=\gamma(t_{j-1},x)$.
Then, by the Lipschitz bound,
\begin{equation}
\left|\int_{t_{j-1}}^{t_j} u_n(s,\gamma(s,x))\,ds-\int_{t_{j-1}}^{t_j} u_n(s,y_j)\,ds\right|
\le L\,\delta\,(t_j-t_{j-1}),
\end{equation}
where $L$ is a uniform Lipschitz constant for $u_n(s,\cdot)$, and similarly with $u$ in place of $u_n$.
Summing over $j$ gives
\begin{equation}
|(II)_n|
\le \sum_{j=1}^m\left|\int_{t_{j-1}}^{t_j}\!\!\big[u_n(s,y_j)-u(s,y_j)\big]ds\right|
+2L\delta\,t.
\end{equation}
Now, for each fixed $y\in M$, the map $\Phi_y:K\to\mathbb R^d$, $\Phi_y(v):=v(y)$
is continuous.
Applying Lemma~\ref{lem:relaxed-compactness} with $\Phi=\Phi_{y_j}$ and
$\varphi=\mathbf 1_{[t_{j-1},t_j]}$ yields
\begin{equation}
\int_{t_{j-1}}^{t_j} u_n(s,y_j)\,ds \;\longrightarrow\; \int_{t_{j-1}}^{t_j} u(s,y_j)\,ds
\quad\text{for each }j.
\end{equation}
Hence $\limsup_{n\to\infty}|(II)_n|\le 2L\delta\,t$.
Since $\delta$ can be made arbitrarily small, we get $(II)_n\to 0$, proving
\eqref{eq:integral_convergence}. Passing to the limit in \eqref{eq:gamma_n_integral}
using \eqref{eq:gamma_n_to_gamma} gives, for all $x$ and $t$,
\begin{equation}
\label{eq:gamma_limit_integral}
\gamma(t,x)=x+\int_0^t u(s,\gamma(s,x))\,ds.
\end{equation}
Thus $\gamma$ is a Carath\'eodory solution driven by the control $u(\cdot)$. According to the condition \textup{(4)} in Definition~\ref{def:compatible-pair}, $\gamma(t,\cdot)\in\mathcal M$ for all $t$.
Moreover, we have $\|u(t)\|_{\mathcal F}\le 1$ for a.e.\ $t$,
so $\gamma$ is an admissible (horizontal) curve from $\mathrm{Id}$ to $\psi_1$ and
\begin{equation}
L(\gamma)=\int_0^T \|u(t)\|_{\mathcal F}\,dt \le T.
\end{equation}
Therefore $d(\psi_1,\mathrm{Id})\le T$.
Since $T>d_{\mathcal F}(\psi_1,\mathrm{Id})$ was arbitrary, we conclude
$d(\psi_1,\mathrm{Id})\le d_{\mathcal F}(\psi_1,\mathrm{Id})$, and by right-invariance
the same holds for general $(\psi_1,\psi_2)$:
\begin{equation}
d(\psi_1,\psi_2)\le d_{\mathcal F}(\psi_1,\psi_2).
\end{equation}

Finally, let us show that $d_{\mathcal F}\le d$.
Let $\gamma:[0,1]\to\mathcal M$ be horizontal with $\dot\gamma(t)=v(t)\circ \gamma(t)$ and $v(\cdot)$ measurable with $\int_0^1\|v(t)\|_{\mathcal F}\,dt<\infty$.
We approximate $t\mapsto v(t)$
in $L^1$ by simple functions taking values in $\mathbf{CH}_1(\mathcal F)$.
On each subinterval where the control function $f = \sum_{i=1}^N a_i f_i$ is the convex combination of finitely many vector fields $f_i\in\mathcal F$,
we apply the corresponding piecewise-constant controls with total time equal to the sum of the coefficients $\sum_{i=1}^N |a_i|$,
so that the endpoint converges to $\gamma(1)$ while the total time converges to $\int_0^1\|v(t)\|_{{\mathcal F}}\,dt$.
Taking infima over all horizontal $\gamma$ gives $d_{\mathcal F}(\psi_1,\psi_2)\le d(\psi_1,\psi_2)$.

Combining the two inequalities we obtain $d=d_{\mathcal F}$ on $\operatorname{Diff}^1(M)$. This completes the proof.

\subsection{Interpolation distance, variational formula and asymptotic properties}

% \ql{Need some introduction relating to
% \begin{itemize}
%     \item Sub-riemannian geometry
%     \item Interpolation and the SICON paper
% \end{itemize}
% Also, the current formulation is for the 1D problem.
% If this is the case, we should put it in the 1D ReLU case.
% }

Our geometric framework is closely related to the analysis of reachability and minimal-time problems
in classical control theory, related to finite-dimensional sub-Riemannian goemetry~\cite{montgomery2002tour,agrachev2019comprehensive}.
The key difference is that the manifold $\mathcal M$ we study is generally infinite-dimensional,
and the local metric may not be induced from an inner product.
Another related concept is the universal interpolation problem of flows studied in~\cite{cheng2025interpolation,cuchiero2020deep}. In~\cite{cheng2025interpolation}, it is shown that for a control family $\mathcal F$ and a target function $\psi$, if any set of finite samples $\{(x_i, \psi(x_i))\}_{i=1}^m$ can be interpolated by some flow in $\mathcal A_{\mathcal F}(T)$ within a uniform time $T$ independent of $m$ and the samples, then the control family $\mathcal F$ can also approximate $\psi$ uniformly within a finite time.

Given these connections, it is thus natural to consider what the finite-dimensional geometry induced by the corresponding finite-point interpolation problem is, and how it relates to the infinite dimensional geometry we studied before.

In the following, we will use $x, y, z$ to denote points in $M$ or $T_xM$, and use $\mathbf x, \mathbf y, \mathbf z$ to denote points in $M^m$ or $T_{\mathbf x} M^m$ (the $m$-fold product). We still use $\varphi,\psi$ to denote flows in $\mathcal A_{\mathcal F}$.

We first show how the control function induces a metric in $M^m$.
We begin with the case when $m = 1$. In this case, we recover the classical metric induced by the minimal time control~\cite{sontag2013mathematical}.
That is, for $x, y \in M$, we define
\begin{equation}
  d_{\mathcal F}^{[1]}(x,y) = \inf\{ T > 0 ~|~ \exists ~\varphi \in \overline{\mathcal A_{\mathcal F}(T)}, \varphi(x) = y \}.
\end{equation}
For the general case, consider $\mathbf{x} = (x_1,x_2,\cdots, x_m) \in M^m$ and $\mathbf{y} = (y_1,y_2,\cdots, y_m) \in M^m$. Then we can define similarly
\begin{equation}
  d_{\mathcal F}^{[m]}(\mathbf{x},\mathbf{y}) = \inf\{ T > 0 ~|~ \exists~ \varphi \in \overline{\mathcal A_{\mathcal F}(T)}, \varphi(x_i) = y_i,
  \,i=1,\dots,m
  \}.
\end{equation}

In fact, the minimum can be achieved as shown in the following.
\begin{proposition}
  \label{prop:distance-attained}
There exists $\varphi \in \overline{\mathcal A_{\mathcal F}(d_{\mathcal F}^{[m]}(\mathbf x, \mathbf y))}$, such that $\varphi(x_i)=y_i$
for $i=1,\dots,m$.
\end{proposition}
\begin{proof}
First, we notice that for any $\varphi_1\in \mathcal A_{\mathcal F}(t_1+t_2)$, there exists $\varphi_2\in \mathcal A_{\mathcal F}(t_1)$ such that
\begin{equation}
    \|\varphi_2-\varphi_1\|_{C(M)}\le (e^{t_2}-1)\|\varphi_1\|_{C(M)}.
\end{equation}
This follows by truncating the dynamics of $\varphi_1$.

For any positive integer $N>0$, there exists $ \varphi_N\in {\mathcal A_{\mathcal F}(d_{\mathcal F}^{[m]}(\mathbf x,\mathbf y)+\frac{1}{N})}$ such that
\begin{equation}
    |\varphi_N(x_i)-y_i|<\frac{1}{N}.
\end{equation}
According to the above fact, there exists $\tilde \varphi_N\in \mathcal A_{\mathcal F}(d_{\mathcal F}^{[m]}(\mathbf x,\mathbf y))$ such that
\begin{equation}
    \|\tilde \varphi_N(x_i) - y_i\|\le \frac{e^{\frac 1 N}}{N}.
\end{equation}
Since $\mathcal F$ is uniformly bounded and uniformly Lipschitz, the sequence $\{\tilde \varphi_N\}_{N=1}^\infty$ is uniformly bounded and uniformly equicontinuous.
By the Arzela-Ascoli theorem, the sequence $\{\tilde \varphi_N\}_{i=1}^\infty$ has a convergece subsequence. Denote this limit as $\tilde \varphi \in \overline{\mathcal A_{\mathcal F}(d_{\mathcal F}^{[m]}(\mathbf x, \mathbf y))}$, then we have $\tilde \varphi(x_i)= y_i$.
\end{proof}

% \begin{proposition}
% \label{prop:distance-finite}
%     We have
%     \begin{equation}
%       d_{\mathcal F}^{[m]}(\mathbf x, \mathbf y) = \min_{I^{[m]} \psi_1 = \mathbf x, I^{[m]}\psi_2=\mathbf y} d_{\mathcal F} (\psi_1,\psi_2).
%     \end{equation}
% \end{proposition}
% \begin{proof}
% We first fix $\psi_1$ and $\psi_2$, since $d^{[m]}_{\mathcal F}(I^{[m]}\psi_1, I^{[m]}\psi_2) \le d_{\mathcal F}(\psi_1,\psi_2).$ We then have $d_{\mathcal F}^{[m]}(\mathbf x, \mathbf y) \le d_{\mathcal F}(\psi_1,\psi_2)$ whenever $I^{[m]}\psi_1 = \mathbf x $ and $I^{[m]}\psi_2 = \mathbf y$.

% Conversely, let $\delta>0$, for $\mathbf x, \mathbf y \in M^m$, we take $T = d^{[m]}_{\mathcal F}(\mathbf x,\mathbf y)$ and consider flow maps in $\mathcal A_{\mathcal F}(T)$. For any $\varepsilon > 0$, there exists $\varphi \in \overline{\mathcal A_{\mathcal F}(T)}$ such that $\varphi(x_i) = y_i$. For any $\psi_1$ such that $I^{[m]}\psi_1 = \mathbf x$, we take $\psi_2 = \varphi \circ \psi_1 \in \overline{\mathcal A_{\mathcal F}(T) \circ \psi_1}$. Therefore, $d_{\mathcal F}(\psi_1, \psi_2) \le T \le d_{\mathcal F}(\psi_1, \psi_2) + \delta$. Take $\delta$ to be arbitrarily small, we conclude the result.
% \end{proof}

Our goal is now to connect this minimal-time distance to $d_{\mathcal F}$.
We fix a sequence of sampling sets $Z^{[m]}=\{z^{[m]}_i\}_{i=1}^m\subset M$ that asymptotically covers $M$, that is,
\begin{equation}
  \label{eq:dense_samples}
    \lim_{n\to \infty} \min_{i\in [n]}|y-z_i| =0, \text{ for all } y\in M.
\end{equation}
We next define the sampling map
\begin{equation}
I^{[m]}:\ \operatorname{Diff}(M)\to M^m,\qquad
I^{[m]}(\psi):=\big(\psi(z^{[m]}_1),\dots,\psi(z^{[m]}_m)\big).
\end{equation}
For $\psi_1,\psi_2\in \operatorname{Diff}(M)$ write $\mathbf x=I^{[m]}(\psi_1)$, $\mathbf y=I^{[m]}(\psi_2)\in M^m$.
Now, we can show that for each $m$, the minimal-time distance $d_{\mathcal F}^{[m]}$ admits
a variational characterization in terms of the geodesic distance $d_{\mathcal F}$,
linking finite-dimensional control theory to our formulation.
\begin{proposition}
\label{prop:distance-finite}
    We have
    $$d_{\mathcal F}^{[m]}(\mathbf x, \mathbf y) = \min_{I^{[m]} \psi_1 = \mathbf x, I^{[m]}\psi_2=\mathbf y} d_{\mathcal F} (\psi_1,\psi_2).$$
\end{proposition}
\begin{proof}
We first fix $\psi_1$ and $\psi_2$, since $d^{[m]}_{\mathcal F}(I^{[m]}\psi_1, I^{[m]}\psi_2) \le d_{\mathcal F}(\psi_1,\psi_2).$ We then have $d_{\mathcal F}^{[m]}(\mathbf x, \mathbf y) \le d_{\mathcal F}(\psi_1,\psi_2)$ whenever $I^{[m]}\psi_1 = \mathbf x $ and $I^{[m]}\psi_2 = \mathbf y$.

Conversely, let $\delta>0$, for $\mathbf x, \mathbf y \in M^m$, we take $T = d^{[m]}_{\mathcal F}(\mathbf x,\mathbf y)$ and consider flow maps in $\mathcal A_{\mathcal F}(T)$. For any $\varepsilon > 0$, there exists $\varphi \in \overline{\mathcal A_{\mathcal F}(T)}$ such that $\varphi(x_i) = y_i$ by~\Cref{prop:distance-attained}.
For any $\psi_1$ such that $I^{[m]}\psi_1 = \mathbf x$, we take $\psi_2 = \varphi \circ \psi_1 \in \overline{\mathcal A_{\mathcal F}(T) \circ \psi_1}$. Therefore, $d_{\mathcal F}(\psi_1, \psi_2) \le T \le d_{\mathcal F}(\psi_1, \psi_2) + \delta$. Take $\delta$ to be arbitrarily small, we conclude the result.
\end{proof}

Conversely, we can show that the infinite-dimensional distance $d_{\mathcal F}$ is actually
the limit of the finite-dimensional Carnot-Carathéodory distance $d_{\mathcal F}^{[m]}$ as $m\to\infty$.

\begin{proposition}
\label{prop:discrete_to_continuous_distance}
Fix $\psi_1,\psi_2\in \mathcal M$, then
\begin{equation}
\lim_{m\to\infty} d_{\mathcal F}^{[m]}\!\big(I^{[m]}(\psi_1),I^{[m]}(\psi_2)\big)= d_{\mathcal F}(\psi_1,\psi_2).
\end{equation}
\end{proposition}

\begin{proof}
The upper bound $d_{\mathcal F}^{[m]}\le d_{\mathcal F}$ follows from Proposition~\ref{prop:distance-finite}.
For lower bound, if for some $\eta>0$,
infinitely many $m$ satisfy $d_{\mathcal F}^{[m]}\le d_{\mathcal F}-\eta$,
then by Proposition~\ref{prop:distance-attained} we obtain $\varphi_m\in \mathcal A_{\mathcal F}(T)$ with $ T\le d_{\mathcal F}-\eta/2$ matching the samples. Equicontinuity of functions in $\mathcal A_{\mathcal F}(T)$ and density of $Z^{[m]}$ give a subsequence $\varphi_m\to \varphi$ uniformly with $\varphi\circ\psi_2=\psi_1$, contradicting the minimality of $d_{\mathcal F}$.
\end{proof}

Recall
the relation between local sub-Finsler norms and the global metric $d_{\mathcal F}$ established in~\Cref{thm:main} for the infinite-dimensional case.
We now establish the analogous relations for the finite-dimensional manifold \(M^m\) equipped with the distance
\(d_{\mathcal F}^{[m]}\).

Write \(\psi^{[m]}:=I^{[m]}(\psi)\in M^m\).
For \(\mathbf u\in T_{\psi^{[m]}}M^m\simeq (\mathbb R^d)^m\), define the corresponding local norm by
\begin{equation}
\label{eq:discrete_local_norm_variational}
\|\mathbf u\|_{\psi^{[m]}}
:=\inf\Big\{\, s>0\ \Big|\ \exists\, g\in \overline{\mathbf{CH}_s(\mathcal F)}\ \text{ s.t. }\ g\circ\psi\big(z^{[m]}_i\big)=u_i,\ i=1,\dots,m\Big\},
\end{equation}
with the convention that the infimum over an empty set equals infinity.
We then have the following analogue of ~\eqref{eq:local_norm_variational} in finite interpolation problems.

\begin{proposition}
\label{prop:discrete_local_derivative}
Let \(\psi\in\mathcal M\) and \(\mathbf u\in T_{\psi^{[m]}}M^m\).
For any \(C^1\) curve \(\gamma:[0,\varepsilon]\to M^m\) with \(\gamma(0)=\psi^{[m]}\) and \(\gamma'(0)=\mathbf u\),
\begin{equation}
\label{eq:discrete_local_norm_derivative}
\|\mathbf u\|_{\psi^{[m]}}
=\lim_{t\to 0}\frac{d_{\mathcal F}^{[m]}\!\big(\psi^{[m]},\, \gamma(t)\big)}{t}.
\end{equation}
\end{proposition}

\begin{proof}
  The argument is similar to the proof~\eqref{eq:local_norm_variational} in \Cref{thm:main}.

\emph{Upper bound}: If for some $s>0$, \(\mathbf u\) is realized at the samples by some \(g\in \overline{\mathbf{CH}_s(\mathcal F)}\) (i.e.\ \(g\circ\psi(z^{[m]}_i)=u_i\)), concatenate short flows of generators with total time \(s\,t+o(t)\)
to produce a horizontal curve in \(M^m\) starting at \(\psi^{[m]}\) whose endpoint at time \(t\) matches \(\gamma(t)\) to first order. Therefore,
\(d_{\mathcal F}^{[m]}(\psi^{[m]},\gamma(t))\le s\,t+o(t)\). Taking the limit \(t\to 0^+\) and then the infimum over \(s\) gives the upper bound.
\textbf{Lower bound}: the uniformly Lipschitz property of admissible flows and sub-additivity of \(d_{\mathcal F}^{[m]}\) imply that any admissible flow
of time \(o(t)\) realizing \(\gamma(t)\) forces \(\mathbf u\) to be attained by some \(g\in \overline{\mathbf{CH}_s(\mathcal F)}\) with \(s\) arbitrarily close to the metric slope \(\liminf_{t\to 0^+} \frac{d_{\mathcal F}^{[m]}(\psi^{[m]},\gamma(t))}{t}\).
\end{proof}

Similar to the infinite-dimensional case, we can also characterize the distance \(d_{\mathcal F}^{[m]}\) in terms of lengths of curves in \(M^m\) induced by the local norms.
Specifically, let \(\gamma:[0,1]\to M^m\) be an absolutely continuous curve, and define its discrete length by
\begin{equation}
L^{[m]}(\gamma):=\int_0^1 \big\|\dot\gamma(t)\big\|_{\gamma(t)}\,dt.
\end{equation}
We then have the following, whose proof is similar to that of~\eqref{eq:geodesic_distance} in \Cref{thm:main}.

\begin{proposition}
\label{prop:discrete_geodesic}
For all \(\mathbf x,\mathbf y\in M^m\),
\begin{equation}
\label{eq:discrete_cc}
d_{\mathcal F}^{[m]}(\mathbf x,\mathbf y)
=\inf\Big\{\, L^{[m]}(\gamma)\ \Big|\ \gamma:[0,1]\to M^m\ \text{absolutely continuous},\ \gamma(0)=\mathbf x,\ \gamma(1)=\mathbf y\Big\}.
\end{equation}
\end{proposition}

Finally, recall the variational characterization of the local norm in~\eqref{eq:local_norm_variational} for the infinite-dimensional manifold \(\mathcal M\):
\begin{equation}
\label{eq:cont_local_norm}
\|u\|_{\psi}=\inf\Big\{\, s>0\ \Big|\ u\in \overline{\mathbf{CH}_s(\mathcal F)\circ\psi}\Big\}.
\end{equation}
Write \(\mathbf u^{[m]}:=I^{[m]}(u)\in T_{\psi^{[m]}}M^m\).

The following proposition shows that the discrete local norms converge to the continuous local norm as the number of samples increases.

\begin{proposition}
\label{prop:local_norm_convergence}
Let \(\psi\in\mathcal M\) and \(u\in T_\psi\mathcal M\). We then have:
\begin{equation}
\lim_{m\to\infty}\ \big\|\, \mathbf u^{[m]}\,\big\|_{\,\psi^{[m]}}\ =\ \|u\|_{\psi}.
\end{equation}
\end{proposition}

\begin{proof}
\emph{Upper bound for left-hand side:}
Fix $\varepsilon>0$. By definition of $\|u\|_\psi$, choose $g_\varepsilon\in \overline{\mathbf{CH}_{\|u\|_\psi+\varepsilon}(\mathcal F)}$ such that
$u=g_\varepsilon\circ\psi$ on $M$.
Evaluating at the samples gives, for every $m$ and $i=1,\dots,m$,
\begin{equation}
g_\varepsilon\big(\psi(z^{[m]}_i)\big)\ =\ u(z^{[m]}_i).
\end{equation}
Hence, by the definition of the finite-sample local norm,
\begin{equation}
\|u^{[m]}\|_{\psi^{[m]}}\ \le\ \|u\|_\psi+\varepsilon\qquad\text{for all }m.
\end{equation}
Taking $\limsup_{m\to\infty}$ and then $\varepsilon\to  0^+$ yields
\begin{equation}
\limsup_{m\to\infty}\ \|\mathbf u^{[m]}\|_{\psi^{[m]}}\ \le\ \|u\|_\psi.
\end{equation}

\medskip\noindent
\emph{Lower bound for left-hand side:}
Assume, by contradiction, that there exists $\eta>0$ and an infinite subsequence $(m_k)$ such that
\begin{equation}
\|\mathbf u^{[m_k]}\|_{\psi^{[m_k]}}\ \le\ \|u\|_\psi-\eta\qquad\text{for all }k.
\end{equation}
By the discrete variational definition, for each $k$ there exists
$g_k\in \overline{\mathbf{CH}_{\|u\|_\psi-\eta}(\mathcal F)}$ such that
\begin{equation}\label{eq:sample_match}
g_k\big(\psi(z^{[m_k]}_i)\big)\ =\ u(z^{[m_k]}_i),\qquad i=1,\dots,m_k.
\end{equation}
The family $\{g_k\}$ is uniformly bounded and uniformly equicontinuous on $M$;
by the Arzelà-Ascoli theorem, passing to a further subsequence (not relabeled) we may assume
\begin{equation}
g_k\ \to\ g\quad\text{uniformly on }M,
\end{equation}
for some $g\in \overline{\mathbf{CH}_{\|u\|_\psi-\eta}(\mathcal F)}$ (closedness of the hull).

We claim that $g\circ\psi=u$ on $M$, which contradicts the definition of $\|u\|_\psi$.
Fix any $y\in M$. Since $Z^{[m]}$ asymptotically covers $M$ (defined in~\eqref{eq:dense_samples}), there exists a choice of indices $i(k)\in\{1,\dots,m_k\}$ such that
\begin{equation}
z^{[m_k]}_{i(k)}\ \to\ y\qquad (k\to\infty).
\end{equation}
By continuity of $\psi$ and $u$ and uniform convergence $g_k\to g$,
\begin{equation}
u\big(z^{[m_k]}_{i(k)}\big)\ \to\ u(y),\qquad
g_k\big(\psi(z^{[m_k]}_{i(k)})\big)\ \to\ g\big(\psi(y)\big).
\end{equation}
Using the matching property \eqref{eq:sample_match} at the indices $i(k)$, we conclude
\begin{equation}
g\big(\psi(y)\big)\ =\ u(y).
\end{equation}
Since $y\in M$ was arbitrary, $g\circ\psi=u$ on all of $M$. But $g\in \overline{\mathbf{CH}_{\|u\|_\psi-\eta}(\mathcal F)}$, hence the definition of $\|u\|_\psi$ would force $\|u\|_\psi\le \|u\|_\psi-\eta$, a contradiction.

Therefore no such subsequence can exist, and we must have
\begin{equation}
\liminf_{m\to\infty}\ \|\mathbf u^{[m]}\|_{\psi^{[m]}}\ \ge\ \|u\|_\psi.
\end{equation}

\medskip
Combining the two steps gives
\(
\displaystyle \lim_{m\to\infty}\|\mathbf u^{[m]}\|_{\psi^{[m]}}=\|u\|_\psi.
\)
\end{proof}

We have shown that the infinite-dimensional sub-Finsler geometry on $\operatorname{Diff}(M)$
(both its \emph{local} norm and its \emph{global} minimal-time distance $d_{\mathcal F}$)
arises as the limit of the corresponding finite-sample interpolation geometry as the sampling sets densify.
Specifically, as $Z^{[m]}$ becomes dense we have the convergence of distances
\begin{equation}
\lim_{m\to\infty} d_{\mathcal F}^{[m]}\!\big(I^{[m]}(\psi_1),\,I^{[m]}(\psi_2)\big)
\;=\; d_{\mathcal F}(\psi_1,\psi_2),
\end{equation}
and the convergence of local norms
\begin{equation}
\lim_{m\to\infty}\ \big\|\, I^{[m]}(u)\,\big\|_{\,I^{[m]}(\psi)}
\;=\; \|u\|_{\psi},
\end{equation}
Thus the finite-sample sub-Finsler structure is a consistent discretization of the infinite dimensional one.
The relations established in this subsection can be summarized in the following commutative diagram.
\begin{equation}
\begin{tikzcd}[row sep=3.5em, column sep=5em]
d_{\mathcal F}(\cdot, \cdot)
  \arrow[r, shift left=0.7ex, "\text{derivative}"]
  \arrow[d, shift left=0.7ex, "\text{sampling}"]
&
\|\cdot \|_{\psi}
  \arrow[l, shift left=0.7ex, "\text{integral}"]
  \arrow[d, shift left=0.7ex, "\text{sampling}"]
\\
d_{\mathcal F}^{[m]}(\cdot, \cdot)
  \arrow[u, shift left=0.7ex, "m\to \infty"]
  \arrow[r, shift left=0.7ex, "\text{integral}"]
&
\|\cdot \|_{\psi^{[m]}}
  \arrow[l, shift left=0.7ex, "\text{integration}"]
  \arrow[u, shift left=0.7ex, "m\to \infty"]
\end{tikzcd}
\end{equation}

This is fully consistent with the relation between universal interpolation and approximation studied in \cite{cheng2025interpolation}:
there, the existence of a \emph{uniform} time $T$ that interpolates any finite sample of a target $\psi$ is shown to be equivalent to $C^0$ approximation of $\psi$ in finite time.
The results here provide a geometric refinement of that equivalence:
uniform control of the discrete distances $d_{\mathcal F}^{[m]}$ over dense samples is equivalent to control of the global distance $d_{\mathcal F}$, and the discrete local norms converge to the continuous sub-Finsler norm that determines $d_{\mathcal F}$ via the geodesic (path-length) formula.
Practically, this links finite-sample interpolation complexity to the intrinsic (sub-Finsler) complexity of flow approximation.

%% file: applications.tex
\section{Applications}
\label{sec:applications}

The geometric picture we presented holds generally for any control family $\mathcal F$ over a smooth compact manifold $M \subset \mathbb R^d$  satisfying our assumptions. For a given control family $\mathcal F$, our framework provides a way to characterize or estimate the approximation complexity $C_{\mathcal F}(\psi)$ for given target function $\psi
\in \mathcal M$, by studying the induced sub-Finsler norm and geodesics. Moreover, our framework also generalizes the approximation complexity $\mathrm C_{\mathcal F}$ to a distance $d_{\mathcal F}(\psi_1, \psi_2)$ between any two diffeomorphisms $\psi_1, \psi_2\in \mathcal M$, which may provide new insights to function/distribution approximation and model building.
In the following, we will demonstrate the applicability and discuss the insights of our framework using specific examples.

\subsection{1D ReLU revisited}
Let us revisit the results for 1D ReLU control family
studied in~\Cref{sec:1d_relu} following the geometric framework established in the previous section.
In particular, we will identify all the spaces involved precisely,
and show how the manifold viewpoint reveal new insights to ReLU flow approximation.

\subsubsection{A complete summary for the ReLU case}
We start from the target function space $\mathcal T$ associated with the 1D ReLU control family $\mathcal F$ defined in~\eqref{eq:1d_relu_family}.
In our geometric viewpoint, $\mathcal T$ is the connected component of the identity in $\mathcal M =\operatorname{Diff}^{W^{1,\infty}}([0,1])$ under the topology induced by the geodesic distance induced by the sub-Finsler norm defined by the control family $\mathcal F$.
According to the general results, the horizontal (tangent) space at each $\psi\in\mathcal M$ can be identified as:
\begin{equation}
    X_{\mathcal F}\circ \psi \;:=\;\{\, f\circ \psi \mid f\in X_{\mathcal F}\,\}.
\end{equation}
With the local sub-Finsler norm given by
% \begin{equation}
%     \|u\|_{\psi} \;:=\; \lim_{t\to 0^+} \frac{d_{\mathcal F}(\psi, \gamma(t))}{t},
%     \quad \text{for any horizontal $C^1$ curve $\gamma$ with }\gamma(0)=\psi,\ \gamma'(0)=u.
% \end{equation}
\begin{equation}
    \|u\|_{\psi}=\inf\{\,s>0\mid u\in \overline{\mathbf{CH}_s(\mathcal F)}^{C^0}\circ \psi\,\},
\end{equation}
Moreover, $d_{\mathcal F}$ can be identified as the geodesic distance given by:
\begin{equation}
d_{\mathcal F}(\psi_1,\psi_2)=\inf\Big\{\ \int_0^1 \|\dot\gamma(t)\|_{\gamma(t)}\,dt\ :\ \gamma\ \text{horizontal},\ \gamma(0)=\psi_1,\ \gamma(1)=\psi_2\ \Big\}.
\end{equation}

In the 1D ReLU case, fortunately, the local norm $\|\cdot\|_\psi$ admits a closed-form characterization, as given in~\Cref{prop:local_norm_relu}.
This in turn gives a closed-form characterization of the space $X_{\mathcal F}$: it is just all the $\operatorname{BV}^2$ functions on $[0,1]$ that vanishes at the boundary.
This also results in a closed-form characterization of $\mathcal M$.
Furthermore, after a global transformation in~\eqref{eq:map_alpha}, we can also characterize a geodesic connecting any two target functions in $\mathcal M$.
With these characterizations, we can compute the distance $d_{\mathcal F}$ exactly as~\Cref{thm:1D_relu} by the geodesic characterization. Finally, we have the explicit expression for $d_{\mathcal F}$:
\begin{equation}
    d_{\mathcal F}(\psi_1, \psi_2) =\|\ln \psi_1^\prime -\ln \psi_2^\prime \|_{\operatorname{TV}([0,1])} \int_0^1 \left| \frac{\psi_1^{\prime\prime}(x)}{\psi_1^\prime(x)} - \frac{\psi_2^{\prime\prime}(x)}{\psi_2^\prime(x)} \right|\, dx.
\end{equation}

\subsubsection{Discussions on the closed-form results for 1D ReLU}
Now, let us take a closer look at the closed-form representation of the complexity $C_{\mathcal F}$:
\begin{equation}
    C_{\mathcal F}(\psi) = d_{\mathcal F}(\operatorname{Id}, \psi) = \int_0^1 \left|\frac{\psi^{\prime\prime}(x)}{\psi^\prime(x)}\right|\, dx.
\end{equation}
We can see that the complexity depends on the ratio between the second derivative and the first derivative of the target function $\psi$. Therefore, if a map $\psi\in \mathcal M$ has first order derivative bounded below by some $c>0$, and second order derivative bounded above by some $C>0$, then it has a complexity bounded by $C/c$. If $C$ is small and $c$ is relatively large, then the function is easy to approximate with ReLU flows.

On the other hand, if the first order derivative $\psi^\prime$ is close to zero at some point, or the second order derivative $\psi^{\prime\prime}$ is large at some point, then the complexity becomes large. For example, consider $\psi_{\varepsilon}(x):= \varepsilon x+ x^2$ for small $\varepsilon>0$. Then, we have
\begin{equation}
    C_{\mathcal F}(\psi_\varepsilon) = \int_0^1 \left|\frac{2}{\varepsilon + 2x}\right|\, dx = 2\ln\left(1+\frac{2}{\varepsilon}\right),
\end{equation}
which goes to infinity as $\varepsilon\to 0$. This is consistent with the intuition that when $\varepsilon$ is small, the function $\psi_\varepsilon$ has a very small first order derivative near $x=0$, which makes approximation by flow maps hard.
Also, if we consider $\xi_n: = x+\frac{1}{2n\pi}\sin(n\pi x)$ for positive integer $n$, then we have
\begin{equation}
    \xi_n^\prime(x) = 1+\frac{1}{2}\cos(n\pi x)\in [\frac{1}{2}, \frac{3}{2}]
\end{equation}
We have
\begin{equation}
    C_{\mathcal F}(\xi_n) = \int_0^1 \left|\frac{-\frac{n\pi}{2}\sin(n\pi x)}{1+\frac{1}{2}\cos(n\pi x)}\right|\, dx = \int_0^1 \left|\frac{n\pi \sin(n\pi x)}{2+\cos(n\pi x)}\right|\, dx\ge \frac{1}{3} \int_0^1 n\pi |\sin(n\pi x)|\, dx = \frac{2}{3} n,
\end{equation}
which goes to infinity as $n\to \infty$. This is also consistent with the intuition that when $n$ is large, the derivative of $\xi_n$ oscillates rapidly, making flow approximation hard.

We also notice that this complexity measure is very different from the classical function space norms,
such as Sobolev norms or Besov norms, which mainly depend on the absolute values of the derivatives of different orders.
Instead, our complexity measure depends on the first order derivative relative to the second order derivative,
which is more related to the geometry of the function graph.

For a more intuitive understanding of this difference,
we provide a visualization of the distance $d_{\mathcal F}(\cdot, \cdot)$
compared to the classical $L^2$ distance with respect to the Lebesgue measure
on a set of functions in $\mathcal M$. Specifically, we consider three functions in $\mathcal M$ defined as
\begin{equation}
    \psi_1(x) = x+\sin (2\pi x)/7\pi, \quad \psi_2(x) = x+\sin (4\pi x)/7\pi, \quad \psi_3(x) = x-\sin (2\pi x)/7\pi - \sin(4\pi x)/7\pi.,
\end{equation}

By the characterization of $\mathcal M$, their convex combinations $a\psi_1 + b\psi_2 + c\psi_3$, with $a,b,c\ge 0$ and $a+b+c=1$, also belong to $\mathcal M$. We then parametrize the function $\psi = a\psi_1 + b\psi_2 + c\psi_3$ by the barycentric coordinates $(a,b,c)$ in the triangle with vertices $(1,0,0), (0,1,0), (0,0,1)$.
Notice that the point $(\frac{1}{3}, \frac{1}{3}, \frac{1}{3})$ in the triangle corresponds to the identity map $\operatorname{Id}$.
We then visualize the contours of the distances $d_{\mathcal F}(\psi_0, \psi)$ and $\|\psi_0 - \psi\|_{L^2([0,1])}$ to the identity in the triangle. The results are shown in~\Cref{fig:distance-comparison}.
\begin{figure}[htbp]
    \centering
    \begin{minipage}{0.48\textwidth}
        \centering
        \includegraphics[width=\linewidth]{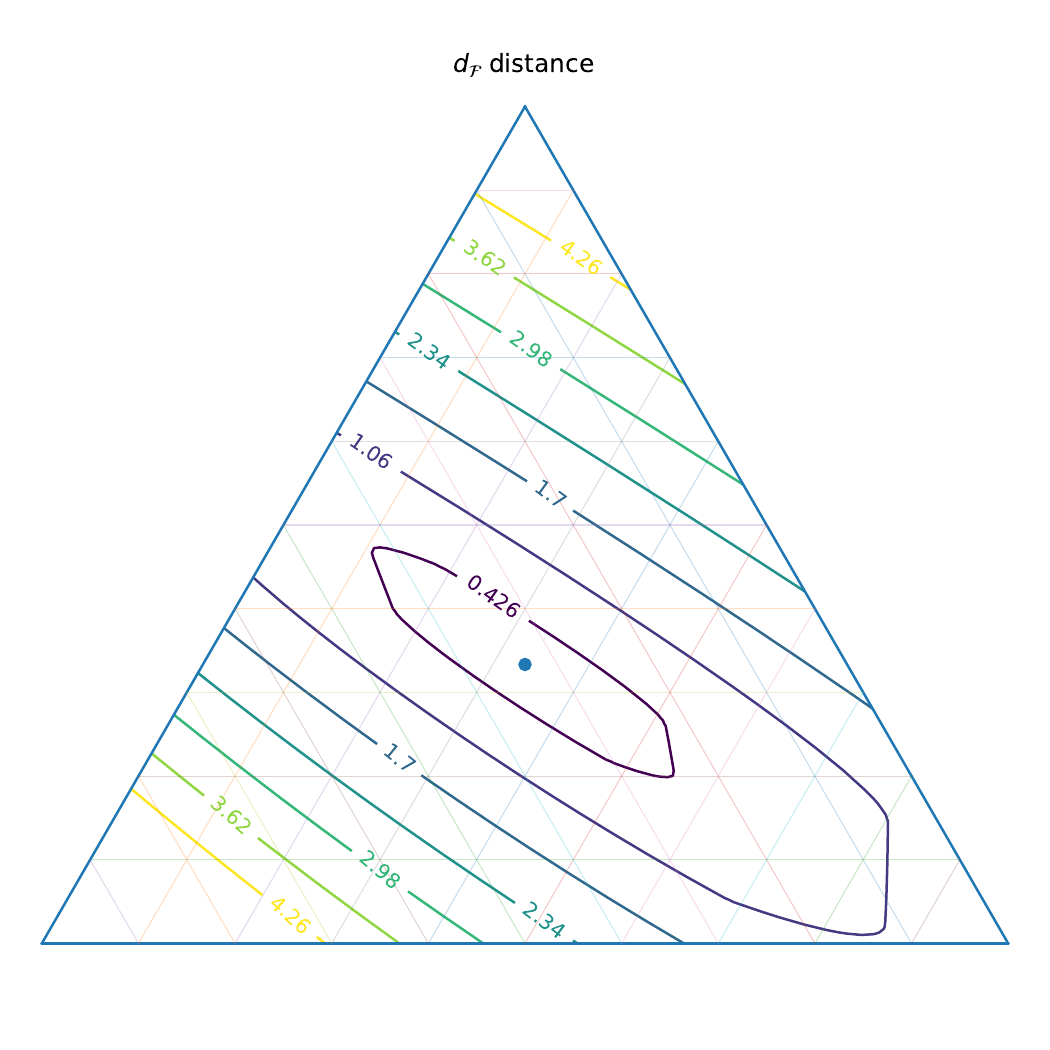}
        \par\medskip\small (a) Countours of $d_{\mathcal F}$
    \end{minipage}\hfill
    \begin{minipage}{0.48\textwidth}
        \centering
        \includegraphics[width=\linewidth]{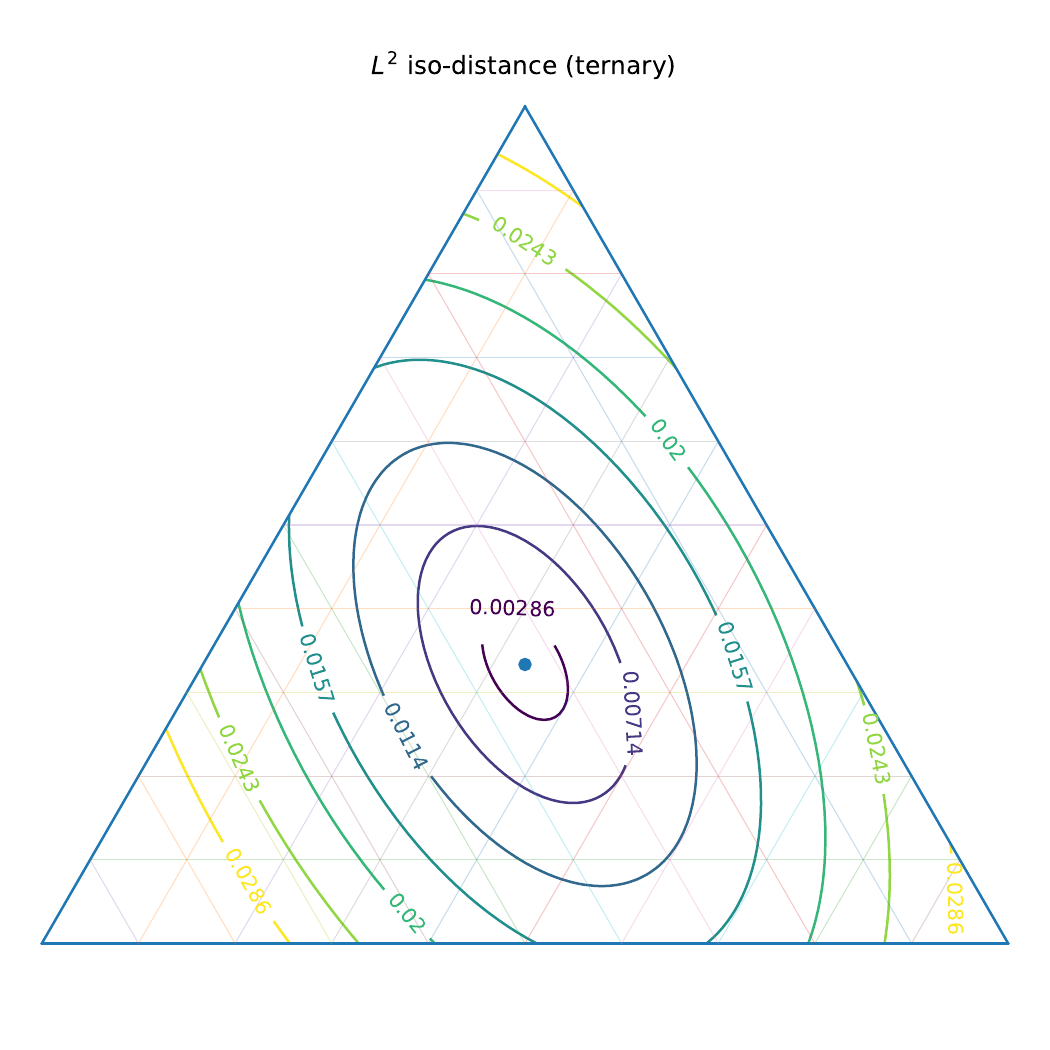}
        \par\medskip\small (b) Countours of $L^2$ distance
    \end{minipage}
    \caption{Comparison of distance contours between $d_{\mathcal F}$ and $L^2$ on the convex hull of three functions in $\mathcal M$}
    \label{fig:distance-comparison}
\end{figure}
From the figure, the contours of $L^2$ distance are elliptical, which is consistent with the fact that it is the affine transformation of the Euclidean distance in $\mathbb R^3$. However, the contours of $d_{\mathcal F}$ are highly non-elliptical, which reflects the difference of the complexity measure from classical norms.

% \cjp{Visualizations to be added here. However, if we want to compare with the orthonormal basis in $L^2$, it is impossible to find two basis function in $\mathcal M$ such that they are orthogonal in $L^2$ sense(because the image are all between $[0,1]$). So how do we give the visualization here? }
% \ql{We discuss this}

\subsubsection{Connections with flow-based generative models}
Another interesting observation is that the distance $d_{\mathcal F}(\psi_1, \psi_2) = |\ln \psi_1^\prime-\ln \psi_2^\prime|_{\operatorname{TV}([0,1])}$ has a connection to the learning of distributions. Specifically, functions in $\mathcal M$ can be naturlly identified as the cumulative distribution functions (CDFs) of some probability distributions supported on $[0,1]$.
For a given $\psi\in \mathcal M$, we denote its corresponding probability distribution as $\mu_\psi$, which has the density function $p_\psi = \psi^\prime$.
Then, the metric $d_{\mathcal F}(\cdot, \cdot)$ naturally induces a metric between probability distributions supported on $[0,1]$ as:
\begin{equation}
    d_{\mathcal F}(\mu_{\psi_1}, \mu_{\psi_2}) := d_{\mathcal F}(\psi_1, \psi_2) = \int_{[0,1]}|(\ln p_{\psi_1})^\prime - (\ln p_{\psi_2})^\prime|_{\operatorname{TV}([0,1])}dx,
\end{equation}
where $\mu_{\psi_1}, \mu_{\psi_2}$ are measures whose positive density functions that are bounded and bounded away from zero,  and have bounded variance.
Notice that the term $(\ln p_\psi)^\prime$ is known as the score function of the distribution $\mu_\psi$ in statistics~\cite{lehmann1998theory}, which is widely used in modern generative models, such as score-based diffusion models~\cite{song2019generative,ho2020denoising,song2020improved}. The metric $d_{\mathcal F}(\cdot, \cdot)$ actually measures $L^1$ distance between the score functions of two distributions.
Therefore, the results indicates that the complexity of transforming one distribution to another using flow maps of ReLU networks is closely related to the distance between their score functions. This provides a new perspective to understand the learning of distributions via flow-based models.
Similar connections in higher dimensional distributions are worthy of further investigations in the future.

\subsubsection{New ways to build models: insights from connected components}
\label{sec:new_ways_build_model}
From the geometric perspective, the distance $d_{\mathcal F}(\cdot, \cdot)$ offers a more general understanding of the approximation complexity $C_{\mathcal F}$: this distance can be defined between any two diffeomorphisms on $M$. Our discussion in~\Cref{sec:1d_relu} just focuses on the target set $\mathcal T$, which is the connected components of the identity map under the topology induced by $d_{\mathcal F}$.

More generally, for any $\psi \in \operatorname{Diff}([0,1])$, we can consider its connected component:
\begin{equation}
    \mathcal T_\psi := \{\tilde \psi \in \operatorname{Diff}([0,1]) \mid d_{\mathcal F}(\psi, \tilde \psi) < \infty\}.
\end{equation}
Then, according to our general results, $\mathcal T_\psi$ is also a Finsler Banach manifold modeled on $X_{\mathcal F}$, with similar characterizations of the local norm and geodesic distance.

This generalization offers us new understandings on function approximation: it is not necessarily to approximate a target function starting from the identity map.
A simple example is, if we choose $\psi = -\operatorname{Id}$ as the opposite of identity map, then $\mathcal T_{-\operatorname{Id}}$ contains orientation-reversing diffeomorphisms. Notice that it is well-known that it is impossible to uniformly approximate a orientation-reversing diffeomorphism by flow maps (i.e. from the identity map). However, if we start from the opposite map $-\operatorname{Id}$, the problem becomes feasible.

% \ql{Say something about ALL ResNets, transformers, etc start with the identity map}
In most of applications of deep learning, including ResNets~\cite{he2016deep}, transformers~\cite{vaswani2017attention}, and flow-based generative models~\cite{rezende2015variational,ho2020denoising,song2021score}, the networks are usually initialized as the identity map. This is of course a natural choice, since the identity map is the simplest function that preserves all information of the input data. However, an insight from our geometric framework is that approximation may be easier if we start from a proper initial function that is closer to the target function, at least in a topological sense.
With this perspective, our geometric framework provides a way of understanding the approximation complexity between any two functions, instead of only from the identity map to the target function.
This and its algorithmic consequences will be explored in future work.

\subsection{Other applications}
We further consider some other examples of control families and investigate the corresponding target function spaces and complexity measures using the geometric framework.

\subsubsection{Transport time on \(\operatorname{SO}(3)\)}
Although the goal of our framework is to study flows generated by neural networks for 
deep learning and artificial intelligence applications, it can be applied to general flows on manifolds.
Here, we consider a case where such a flow gives very familiar objects in geometry,
and discuss what our distance correspond to in these settings.
% \ql{Start with the goal here: motivation is ML/DL, but this geometry can apply to general flows, here we talk about
% a case where such a flow gives very familiar objects in geometry, and discuss what our distance correspond to in these settings.}

We consider a finite-dimensional example where the reachable diffeomorphisms form a Lie subgroup of \(\operatorname{Diff}(M)\). Let \(M=S^2\) (unit sphere) and consider the family of constant fields
\begin{equation}
x\ \mapsto\ A x,\qquad A\in \mathfrak{so}(3),
\end{equation}
i.e.\ \(A\) is skew-symmetric. Writing the \(\hat{\cdot}\)-map
\begin{equation}
\widehat{(a_x,a_y,a_z)}\ :=\
\begin{pmatrix}
0&-a_z&a_y\\
a_z&0&-a_x\\
-a_y&a_x&0
\end{pmatrix},
\qquad \operatorname{vee}(\hat a)=a,
\end{equation}
the vector field is \(x\mapsto \hat a\,x\) with body angular velocity \(a\in\mathbb R^3\).
Flows are rotations \(x\mapsto R x\) with \(R=\exp(t\hat a)\in \operatorname{SO}(3)\).
Hence the reachable set of diffeomorphisms is the finite-dimensional submanifold
\[
\mathcal M\ =\ \operatorname{SO}(3)\ \subset \operatorname{Diff}(S^2).
\]
For \(\psi\in\operatorname{SO}(3)\), define the principal rotation angle and principal logarithm
% \ql{don't use M for the matrix? I think better to use $\varphi$ for consistency, or use some other letter since M is the manifold S2.}
\begin{equation}
\theta(\psi)\ :=\ \arccos\!\Big(\frac{\operatorname{tr}(\psi)-1}{2}\Big)\in[0,\pi],\qquad
\operatorname{Log}(\psi):=\frac{\theta(\psi)}{2\sin\theta(\psi)}\,(\psi-\psi^\top)\in\mathfrak{so}(3),
\end{equation}
so that \(\exp(\operatorname{Log}\psi)=\psi\) and \(\operatorname{Log}(\psi)=\hat{\omega}\) with \(\|\omega\|_2=\theta(\psi)\).

We may identify \(T_R\operatorname{SO}(3) \equiv R\,\mathfrak{so}(3)\),
and define the anchored bundle \((\mathcal E,\pi,\rho)\) by
\begin{equation}
\mathcal E:=\operatorname{SO}(3)\times \mathbb R^3,\qquad
\pi(R,a)=R,\qquad
\rho_R(a)=R\,\hat a\ \in T_R\operatorname{SO}(3).
\end{equation}
Thus \(\rho\) is onto at each \(R\) (full actuation), so \(\mathcal D=\rho(\mathcal E)=T\operatorname{SO}(3)\) and the induced sub-Finsler structure is in fact a \emph{Finsler} structure.
Choosing a fiber norm on \(a\in\mathbb R^3\) yields a left-invariant metric with length
\begin{equation}
L(\gamma)=\int_0^1 \|a(t)\|\,dt,\qquad \dot R(t)=R(t)\,\hat{a}(t).
\end{equation}

We discuss two choices for the control family to realize rotations that will lead
to different approximation rates.

\medskip
\noindent\textbf{(A) \(\ell^2\)-fiber norm (Riemannian).}
Let
\begin{equation}
\mathcal F_2:=\big\{\,x\mapsto \hat a\,x\ \big|\ \|a\|_2\le 1\,\big\}.
\end{equation}
By convexity, the atomic norm coincides with \(\|a\|_2\):
\begin{equation}
\|R\hat a\|_{R}\ =\ \|a\|_2.
\end{equation}
This norm is induced by the inner product
\(\langle \hat a,\hat b\rangle=a\cdot b\) on \(\mathfrak{so}(3)\), hence we obtain the standard bi-invariant Riemannian metric on \(\operatorname{SO}(3)\)~\cite{huynh2009metrics}.
Geodesics are one-parameter subgroups \(R(t)=R_0\exp(t\,\hat \omega)\) with constant \(\omega\), and the geodesic (minimal-time) distance is
\begin{equation}
\label{eq:l2_equals_angle}
d_{\mathcal F_2}(\psi_1,\psi_2)\ =\ \big\|\operatorname{vee}\big(\operatorname{Log}(\psi_2\psi_1^{\top})\big)\big\|_2
\ =\ \theta(\psi_2\psi_1^\top),
\end{equation}
in particular \(d_{\mathcal F_2}(I,\psi)=\theta(\psi)\).
This distance is also called the angle metric on \(\operatorname{SO}(3)\) in the literature~\cite{hartley2013rotation}, since it measures the scale of the angle between two rotations.

\medskip
\noindent\textbf{(B) \(\ell^1\)-fiber norm (Finsler).}
Let the control family be the six axial generators
\[
\mathcal F_1:=\big\{ \pm\hat e_1 x, \pm\hat e_2 x, \pm\hat e_3 x \big\},
\]
so the atomic norm on \(a\in\mathbb R^3\) is \(\|a\|_1\):
\begin{equation}
\|R\hat a\|_{R}\ =\ \|a\|_1.
\end{equation}
The induced distance \(d_{\mathcal F_1}\) is the Finsler distance associated with \(\ell^1\) on the Lie algebra.
Using the constant-velocity path \(R(t)=\exp(t\,\operatorname{Log}\psi)\) gives the explicit \emph{upper bound}
\begin{equation}
\label{eq:l1_upper}
d_{\mathcal F_1}(I,\psi)\ \le\ \big\|\operatorname{vee}(\operatorname{Log}\psi)\big\|_1,\qquad
d_{\mathcal F_1}(\psi_1,\psi_2)\ \le\ \big\|\operatorname{vee}\!\big(\operatorname{Log}(\psi_2 \psi_1^{\top})\big)\big\|_1.
\end{equation}
Moreover, since \(\|a\|_1\ge \|a\|_2\) for all \(a\in\mathbb R^3\), the \(\ell^1\)-Finsler norm dominates the \(\ell^2\)-Riemannian norm pointwise; hence the \emph{lower bound}
\begin{equation}
\label{eq:l1_lower}
d_{\mathcal F_1}(\psi_1,\psi_2)\ \ge\ d_{\mathcal F_2}(\psi_1,\psi_2)\ =\ \theta(\psi_2\psi_1^{\top}).
\end{equation}
Combining \eqref{eq:l1_upper}-\eqref{eq:l1_lower} and the norm equivalence \(\|v\|_1\le \sqrt{3}\,\|v\|_2\) yields the two-sided estimate
\begin{equation}
\theta(\psi_2\psi_1^{\top})\ \le\ d_{\mathcal F_1}(\psi_1,\psi_2)\ \le\ \sqrt{3}\,\theta(\psi_2\psi_1^{\top}).
\end{equation}
Thus the \(\ell^1\)-Finsler transport time is quantitatively comparable to the canonical geodesic distance, with anisotropy encoded by the polyhedral \(\ell^1\) unit ball on the Lie algebra.

\subsubsection{Finite time approximation of smooth functions using additional dimensions}
\label{sec:finite_time_approximation}
In this example, we will consider an application of our framework to the approximation of general smooth functions $f:\mathbb{R}^d\to \mathbb{R}^d$ over compact sets, by using flows in higher dimensional spaces. 

For $R>0$, let $B_R\subset \mathbb{R}^{2d}$ denote the open Euclidean ball of radius $R$ centered at the origin. 
We will consider the following control family $\mathcal F$ on $\mathbb{R}^{2d}$:
\begin{equation}
\mathcal F := \{x\to \rho(x)\sigma(A x+b) \mid A\in \mathbb R^{2d\times 2d}, b\in \mathbb R^{2d}, \|A\|_2 \le 1, \|b\|_2 \le 1 \}
\end{equation}
where $\sigma$ is the ReLU activation function which is applied elementwise, and $\rho\in C_c^\infty(\mathbb{R}^{2d})$ is a smooth cutoff function such that $\rho\equiv 1$ on $B_1$ and $\rho\equiv 0$ outside $B_{7/4}$, and is positive on $B_{7/4}\setminus B_1$. 

Applying our results together with Theorem 2.1 in~\cite{yang2025optimal}, we have the following result:
\begin{proposition}
    \label{prop:finite_time_approximation}
Let $r\ge d+2$ be a positive integer.
For any given compact set $K\subset \mathbb R^d$ and any $C^r$ function $f:\mathbb R^d\to \mathbb R^d$, there exists $T>0$ and two linear maps $\alpha: \mathbb R^d\to \mathbb R^{2d}$, $\beta: \mathbb R^{2d}\to \mathbb R^d$ such that 
\begin{equation}
    f \in \overline{\{\beta \circ \varphi \circ \alpha \mid \varphi\in \mathcal A_{\mathcal F}(T)\}}^{C^0}.
\end{equation}
That is, one can approximate a function with enough smoothness in $d$-dimension arbitrarily well using a  continuous-time neural network with layer-wise architecture given by $\mathcal F$ in $2d$-dimension in a \emph{uniform} time $T$.
\end{proposition}

The application of our framework uses the following proposition:
\begin{proposition}
    \label{prop:lift_projection}
    Given a compact set $K\subset \mathbb R^d$. For any $C^r$ function $f:\mathbb R^d\to \mathbb R^d$, there exists linear maps $\alpha: \mathbb R^d\to \mathbb R^{2d}$, $\beta: \mathbb R^{2d}\to \mathbb R^d$ and a $C^r$ diffeomorphism $\Phi:\mathbb R^{2d}\to \mathbb R^{2d}$ such that
    \begin{equation}
        \beta\circ \Phi \circ \alpha|_K = f|_K,
    \end{equation}
    $\alpha(K)\subset B_1$,  $\Phi$ is equal to the identity on $B_2\setminus B_{3/2}$, and $\Phi$ maps $B_2$ onto itself, and maps $\alpha(K)$ to $B_1$.
\end{proposition}
\begin{proof}
First, let $\alpha:\mathbb R^d\to \mathbb R^{2d}$ be the embedding $\alpha_0(u)=(\lambda u,0)$, and $\beta_0:\mathbb R^{2d}\to \mathbb R^d$ be the projection $\beta(u,v)=v/\kappa$, where $\lambda, \kappa>0$. Since $K$ and $f(K)$ are compact, there exists $\lambda, \kappa>0$ such that $\alpha_0(K)\subset B_1\subset B_2\subset \mathbb R^{2d}$ and $K\times f(K)\subset \mathbb R^{2d}$. 
Choose such $\lambda, \kappa>0$, and consider the map
$\Psi:\mathbb R^{2d}\to \mathbb R^{2d}$ by $\Psi(u,v)=(u,v+\kappa f(1/\lambda u))$. Then, we have $\beta\circ \Psi \circ \alpha|_K = f|_K$. 
Moreover, $\Psi$ is a $C^r$ diffeomorphism of $\mathbb R^{2d}$, with inverse $\Psi^{-1}(u,v)=(u,v-\kappa f(1/\lambda u))$. Also, $\Psi$ maps $\alpha(K)\subset B_1$ to $B_1$.

Next, we will modify $\Psi$ to obtain a diffeomorphism $\Phi$ that is equal to the identity on $B_2\setminus B_{3/2}$, and is equal to $\Psi$ on $\alpha(K)$. This is done by the following lemma:

\begin{lemma}
    \label{lemma:extension}
There exists a family $(\Phi_t)_{t\in[0,1]}$ such that, for every $t\in[0,1]$,
\begin{equation}
\Phi_t:B_2\to B_2
\end{equation}
is a $C^r$ diffeomorphism, and the following properties hold:
\begin{align*}
\Phi_0 &= \operatorname{id}_{B_2},\\
\Phi_1|_K &= \Psi|_K,\\
\Phi_t &= \operatorname{id}
\qquad \text{on } B_2\setminus B_{3/2}.
\end{align*}
In particular, $(\Phi_t)_{t\in[0,1]}$ is an isotopy of $B_2$ from $\operatorname{id}_{B_2}$ to $\Phi_1$, and every $\Phi_t$ is equal to the identity in a neighborhood of $\partial B_2$.
\end{lemma}

\begin{proof}[Proof of the lemma]
    Denote $\tilde f(x): = \kappa f(x/\lambda)$.
Choose $\eta\in C_c^r(B_{3/2})$ such that
\begin{equation}
0\le \eta\le 1,
\qquad
\eta\equiv 1 \quad \text{on } B_1.
\end{equation}
Extend $\eta$ by zero outside $B_{3/2}$, and define a $C^r$ vector field on $\mathbb{R}^{2d}$ by
\begin{equation}
X(u,v):=\eta(u,v)\,(0,\tilde f(u)).
\end{equation}
Since $X$ has compact support, its flow $(\Phi_t)_{t\in\mathbb{R}}$ is globally defined, and each $\Phi_t$ is a $C^r$ diffeomorphism of $\mathbb{R}^{2d}$. Moreover,
\begin{equation}
X\equiv 0 \quad \text{on } \mathbb{R}^{2d}\setminus B_{3/2},
\end{equation}
hence
\begin{equation}
\Phi_t=\operatorname{id}
\quad\text{on } \mathbb{R}^{2d}\setminus B_{3/2}
\end{equation}
for every $t\in\mathbb{R}$. In particular,
\begin{equation}
\Phi_t=\operatorname{id}
\quad\text{on } B_2\setminus B_{3/2}.
\end{equation}

We next show that each $\Phi_t$ maps $B_2$ onto itself. Since $\Phi_t$ fixes every point of $\mathbb{R}^{2d}\setminus B_{3/2}$, it fixes in particular every point of $\mathbb{R}^{2d}\setminus B_2$. If there were $x\in B_2$ such that $\Phi_t(x)\notin B_2$, then $\Phi_t(\Phi_t(x))=\Phi_t(x)$ as well, contradicting injectivity of $\Phi_t$. Thus $\Phi_t(B_2)\subset B_2$. Applying the same argument to $\Phi_t^{-1}$ gives $\Phi_t(B_2)=B_2$.

Finally, let $x=(u,v)\in K$, and define
\begin{equation}
\gamma_x(t):=(u,v+t \tilde f(u))=(1-t)x+t\Psi(x), \qquad t\in[0,1].
\end{equation}
Since $x\in B_1$ and $\Psi(x)\in B_1$, and $B_1$ is convex, we have $\gamma_x(t)\in B_1$ for all $t\in[0,1]$. Hence $\eta(\gamma_x(t))=1$ for all $t$, and therefore
\begin{equation}
\dot\gamma_x(t)=(0,\tilde f(u))=X(\gamma_x(t)).
\end{equation}
Also $\gamma_x(0)=x$. By uniqueness of solutions to the ODE generated by $X$, it follows that
\begin{equation}
\Phi_t(x)=\gamma_x(t)
\qquad\text{for all } t\in[0,1].
\end{equation}
In particular,
\begin{equation}
\Phi_1(x)=\gamma_x(1)=\Psi(x).
\end{equation}
Since $x\in K$ was arbitrary, we conclude that
\begin{equation}
\Phi_1|_K=\Psi|_K.
\end{equation}
This proves the result.
\end{proof}
The lemma implies that we can extend $\psi_{|\alpha(K)}$ to a $C^r$ diffeomorphism $\Phi$ of $\mathbb R^{2d}$ that is equal to the identity on $B_2\setminus B_{3/2}$, and maps $B_1$, $B_2$ onto themselves. Then, we have $\beta\circ \Phi \circ \alpha|_K = f|_K$, which completes the proof of the proposition.

\end{proof}

Now, we can provide the proof of Proposition~\ref{prop:finite_time_approximation}.

\begin{proof}[Proof of Proposition~\ref{prop:finite_time_approximation}]

With Proposition~\ref{prop:lift_projection}, we can apply our framework on the diffeomorphism $\Phi$ from $B_2$ to itself in $\mathbb R^{2d}$. We choose $\mathcal M$ as the set of $W^{1,\infty}$ diffeomorphisms of $B_2$ to itself that fixes the boundary.

For any map $g:\mathbb R^{2d} \to \mathbb R^{2d}$ that is supported on $B_{3/2}$ and of class $C^r$, there exists a $C^r$ smooth function $h$ supported on $B_2\subset \mathbb R^{2d}$ such that $g(x) = h(x)\rho(x)$ for all $x\in \mathbb{R}^d$. 
According to the discussion in Example~\ref{rem:good-pair-examples}, the pair $(\mathcal M, \mathcal F)$ is a compatible pair. According to Theorem 2.1 in~\cite{yang2025optimal}, the $C^0$ closure of the convex hull of the following family
\begin{equation}
\tilde{\mathcal F}:  \{x\to \sigma(A x+b) \mid A\in \mathbb R^{2d\times 2d}, b\in \mathbb R^{2d}, \|A\|_2 \le 1, \|b\|_2 \le 1 \}
\end{equation}
contains a ball in the space $C^r(B_2)$. Therefore, the closure of the convex hull of $\mathcal F$ contains a ball in $C^r(B_2)$ intersected with the set of functions supported on $B_{3/2}$. 
Furthermore, by the proof of Lemma~\ref{lemma:extension} above, the diffeomorphism $\Phi$ can be connected to the identity map by a path of diffeomorphisms that keeps to be the identity map on $B_2\setminus B_{3/2}$.
This means that the path $\gamma(t)$ from $\operatorname{Id}$ to $\Phi$ can be generated by a flow of $C^r$ vector fields supported on $B_{3/2}$, and therefore $\|\dot\gamma(t)\|_{\gamma(t)}$ is finite for all $t$. This implies that
\begin{equation}
    d_{\mathcal F}(\operatorname{Id}, \Phi) = \int_0^1 \|\dot\gamma(t)\|_{\gamma(t)} dt < \infty.
\end{equation}

\end{proof}

\subsubsection{Deep ResNet with layer normalisation}
We now return to learning problems, where we demonstrate how our framework
can be used as a recipe to estimate approximation complexities even when exact computations are difficult.

Let us consider a two-dimensional control family with normalization, i.e., a projection onto the unit circle.
Our base manifold is the unit circle \(M=S^1\subset \mathbb R^2\).
Fix \(b\in[-1,1]\) and introduce a control family \(\mathcal G_b\subset \operatorname{Vec}(\mathbb R^2)\) by
\begin{equation}
    \mathcal G_b:=\big\{\,x\mapsto W\,\sigma(Ax+b\mathbf 1)\ \big|\ W,A\in\mathbb R^{2\times 2},\ W\ \text{diagonal},\ \|W\|_F\le 1,\ \|A_i\|_2=1,\ i=1,2 \big\},
\end{equation}
where \(\sigma\) is the ReLU activation and \(\mathbf 1=(1,1)^\top\).
Let \(\operatorname{Proj}: \operatorname{Vec}(\mathbb R^2)\to \operatorname{Vec}(S^1)\) be the orthogonal projection onto the tangent bundle:
\begin{equation}
\operatorname{Proj}(V)(x):=(I-xx^\top)V(x),\qquad x\in S^1,\ V\in \operatorname{Vec}(\mathbb R^2).
\end{equation}
Define the projected control family on \(S^1\) by
\begin{equation}
\mathcal F_b:=\{\,\operatorname{Proj}(g)\mid g\in \mathcal G_b\,\}\ \subset\ \operatorname{Vec}(S^1).
\end{equation}
We introduce the manifold $\mathcal M := \operatorname{Diff}^{W^{1,\infty}}(S^1)$, the set of diffeomorphisms of $S^1$ that are bi-Lipschitz and have Lipschitz inverses, which is a Banach manifold modeled on $W^{1,\infty}(S^1)$.

This construction is motivated by layer normalization in deep learning~\cite{ba2016layer}: in a continuous-time idealization, normalization corresponds to projecting the ambient vector field onto the tangent space of a sphere; here the projection \(\operatorname{Proj}\) plays exactly that role.
Layer-normalised networks are core building blocks for deep ResNets (of the transformer type) 
used in practical applications, especially in large language models~\cite{ba2016layer,vaswani2017attention,xiong2020layer}.

Now, we check that $(\mathcal M,\mathcal F_b)$ satisfies Definition~\ref{def:compatible-pair}.

\emph{(1) Exponential charts / $C^1$ Banach manifold structure.}
Identify $S^1$ with $\mathbb R/2\pi\mathbb Z$ via the angle coordinate $\theta$ and write diffeomorphisms as
$\psi(\theta)=\theta+h(\theta)$ with $h\in W^{1,\infty}(S^1)$ and $\|h'\|_{L^\infty}$ sufficiently small (so that $\psi$ is bi-Lipschitz).
In these coordinates, local charts are given by addition,
\[
\beta_\psi(u):=\psi+u,\qquad u\in W^{1,\infty}(S^1)\ \text{small},
\]
which is the exponential chart on the circle (geodesics in $\theta$ are affine). Hence $\mathcal M=\operatorname{Diff}^{W^{1,\infty}}(S^1)$
admits a $C^1$ Banach manifold structure modeled on $W^{1,\infty}(S^1)$.

\emph{(2) Right translation is $C^1$.}
For fixed $\varphi\in\mathcal M$, right translation is $R_\varphi(\eta)=\eta\circ\varphi$.
In the angle coordinate, composition by a fixed bi-Lipschitz map acts boundedly on $W^{1,\infty}(S^1)$, and the induced map on charts is
$u\mapsto u\circ\varphi$, which is $C^1$ (indeed affine in these coordinates).

\emph{(3) Admissible velocities.}
Each $g\in\mathcal G_b$ is globally Lipschitz on $\mathbb R^2$,
and $\operatorname{Proj}$ is smooth on $S^1$ with uniformly bounded derivative. Thus every $f\in\mathcal F_b$ belongs to
$W^{1,\infty}\operatorname{Vec}(S^1)$ with a uniform $W^{1,\infty}$ bound, and consequently $X_{\mathcal F_b}$
embeds continuously into $W^{1,\infty}(S^1)$. Therefore, for any $\psi\in\mathcal M$ and $f\in X_{\mathcal F_b}$,
the composition $f\circ\psi$ gives a tangent direction at $\psi$ in the $W^{1,\infty}$ manifold structure.

\emph{(4) Invariance under Carath\'eodory controls.}
Let $u(\cdot):[0,T]\to X_{\mathcal F_b}$ be measurable. Since $X_{\mathcal F_b}\subset W^{1,\infty}(S^1)$
and the uniform $W^{1,\infty}$ bound imply that $u(t,\cdot)$ is Lipschitz in space with an integrable Lipschitz modulus.
Hence the Carath\'eodory ODE $\dot\gamma(t)=u(t)\circ\gamma(t)$ admits a unique absolutely continuous solution.
Moreover, Gr\"onwall's inequality yields a bi-Lipschitz bound on $\gamma(t)$, so $\gamma(t)\in\mathcal M$
for all $t\in[0,T]$.

\medskip
The family \(\mathcal F_b\) induces an anchored bundle \((\mathcal \mathcal E,\pi,\rho)\) over \(\mathcal M\), with horizontal distribution \(\mathcal D=\rho(\mathcal E)\subset T\mathcal M\) and fiberwise atomic norms yielding a sub-Finsler structure \(\|\cdot\|_{\mathcal D,x}\).
Identifying the \emph{full} norm \(\|\cdot\|_{\mathcal D,x}\) on all of \(\mathcal D_x\) is difficult in this example.
Instead, we estimate the norm on a fiberwise linear subspace \(\widetilde{\mathcal D}_x\subset \mathcal D_x\) on which the explicit bound is easy to compute. This collection \(\{\widetilde{\mathcal D}_x\}_{x\in\mathcal M}\) actually defines a subbundle of the anchored bundle \((E,\pi,\rho)\).

Specifically, we work on the subset
\[
\bar{\mathcal M}:=\big\{\psi\in \operatorname{Diff}(S^1)\ \big|\ \psi,\psi^{-1}\in C^3,\ \psi \text{ orientation preserving}\big\},
\]
composed of smooth orientation-preserving diffeomorphisms whose inverses are also smooth.
We parametrize \(S^1\) by the angle \(\theta\in\mathbb R\) and identify \(\psi\in\bar{\mathcal M}\) with its \(2\pi\)-periodic lift \(\bar\psi:\mathbb R\to\mathbb R\) satisfying
\begin{equation}
\bar\psi(\theta+2\pi)=\bar\psi(\theta)+2\pi,\qquad \bar\psi'(\theta)>0,\ \forall\,\theta\in\mathbb R,\qquad \bar\psi(0)\in[0,2\pi).
\end{equation}
We use the identification
\begin{equation}
C_{\operatorname{per}}^3([0,2\pi])=\{\,u\in C^3(\mathbb R)\mid u(\theta+2\pi)=u(\theta)\ \text{for all }\theta\,\}
\end{equation}
to represent a \emph{chosen} fiberwise subspace \(\widetilde{\mathcal D}_\psi\subset \mathcal D_\psi\) on which we can compute explicit bounds; in what follows \(u\in C_{\operatorname{per}}^3([0,2\pi])\) should be read as \(u\in \widetilde{\mathcal D}_\psi\).

We then have the following estimate on the local norm over \(\widetilde{\mathcal D}_\psi\) for each \(\psi\in\bar{\mathcal M}\):
\begin{proposition}
\label{prop:S1_local_bound}
Suppose \(b=\cos\beta\) with \(\beta/\pi\notin\mathbb Q\).
There exist constants \(C_{b,1},C_{b,2}>0\) such that for any \(u\in C_{\operatorname{per}}^3([0,2\pi])\) and \(\psi\in\bar{\mathcal M}\),
\begin{equation}
\|u\|_{\mathcal D,\psi}\ \le\ C_{b,1}\,\big\|(\,u\circ \psi^{-1}\,)^{(3)}\big\|_{L^2([0,2\pi])}
\ +\ C_{b,2}\,\big\|u\circ \psi^{-1}\big\|_{C([0,2\pi])}.
\end{equation}
\end{proposition}

\noindent
While a closed form of the geodesic distance on \(\bar{\mathcal M}\) is still unavailable, we can estimate \(d_{\mathcal F_b}\) by integrating the above bound along an explicit (not necessarily optimal) curve whose velocity always lies in the subbundle \(\widetilde{\mathcal D} \), consisting of $C^3$ diffeomorhisms over $[0,2\pi]$. Consider the linear interpolation of lifts
\begin{equation}
\bar\psi_t:=(1-t)\,\bar\psi_1+t\,\bar\psi_2,\qquad t\in[0,1],
\end{equation}
which connects \(\psi_1,\psi_2\in\bar{\mathcal M}\). Since \(\dot{\bar\psi}_t\in C_{\operatorname{per}}^3([0,2\pi])=\widetilde{\mathcal D}_{\psi_t}\), integrating the estimate inProposition~\ref{prop:S1_local_bound} along \(\{\bar\psi_t\}\) yields:

\begin{proposition}[A global upper bound via curves tangent to the subbundle]
\label{prop:S1_global_bound}
Suppose \(b=\cos\beta\) with \(\beta/\pi\notin\mathbb Q\). Then there exist constants \(C_{b,3},C_{b,4}>0\) such that for all \(\psi_1,\psi_2\in\bar{\mathcal M}\),
\begin{equation}
\begin{aligned}
    d_{\mathcal F_b}(\psi_1,\psi_2)
    \ \le\ &\ C_{b,3}\int_0^{2\pi} \Big[ A(\rho)\,(D^2\rho)^2 + B(\rho)\,(D\rho)^2(D^2\rho) + C(\rho)\,(D\rho)^4 \Big]\,dx \\
    &\quad +\ C_{b,4}\ \|\psi_1-\psi_2\|_{C([0,2\pi])},
\end{aligned}
\end{equation}
where
\begin{equation}
\rho(x):=\frac{\psi_1'}{\psi_2'}\big(\psi_2^{-1}(x)\big)>0,\qquad D:=\frac{d}{dz},
\end{equation}
and
\[
A(\rho)=\frac{67\rho^6+67\rho^5+67\rho^4+67\rho^3+172\rho^2-80\rho+60}{420\,\rho^7},
\]
\[
B(\rho)=-\,\frac{17\rho^6+34\rho^5+51\rho^4+68\rho^3+295\rho^2-150\rho+105}{140\,\rho^8},
\]
\[
C(\rho)=\frac{3\big(\rho^5+3\rho^4+6\rho^3+10\rho^2+15\rho+21\big)}{56\,\rho^8}.
\]
\end{proposition}

\noindent
Although the bound appears complicated, it depends primarily on the geometric quantity \(\rho\), i.e., the ratio between the first derivatives of \(\psi_1\) and \(\psi_2\). When higher-order derivatives of \(\rho\) are small and \(\rho\) is bounded away from zero, the length of the above admissible curve (hence \(d_{\mathcal F_b}(\psi_1,\psi_2)\)) is small, indicating that the approximation from \(\psi_1\) to \(\psi_2\) is easy.

By identifying \(\psi\) with an increasing function on \([0,2\pi]\), this setting resembles the 1D ReLU case. As there, the local norm admits an upper bound in terms of Sobolev-type quantities; here, however, higher-order terms emerge due to composition on the circle. Consequently, unlike the 1D ReLU case, there is no global reparametrization (such as~\eqref{eq:map_alpha}) that flattens the local norm into a classical norm; the global complexity remains governed by \(\rho\) and its higher derivatives, as reflected in the estimate above. The qualitative message is similar: if \(\rho\) stays uniformly away from \(0\) and its derivatives up to order \(4\) are small, then the approximation (in the sub-Finsler, minimal-time sense) from \(\psi_1\) to \(\psi_2\) is easy.

\subsection*{Computations for the layer-normalization example}

Recall the control family $\mathcal G_b$ and its projected family $\mathcal F_b$ introduced in the previous subsection.
For completeness, we keep the explicit parametrizations needed for the estimates below.

For a given (scalar) bias parameter $b\in \mathbb R$ (we write $b$ for $b\mathbf 1\in\mathbb R^2$), we consider functions of the form
\begin{equation}
    x\to W\sigma(Ax+b), \quad x\in \mathbb R^2, \ W, A\in \mathbb R^{2\times 2}, \ W \text{ is diagonal }, \ \|W\|_F\le 1, \ \|A_i\|_2=1, \ i=1,2,
\end{equation}
where $A_i$ is the $i$-th row of $A$, and denote by $\mathcal G_b$ the collection of all such vector fields.

We then introduce an associated control family on the circle $S^1\subset \mathbb R^2$ by projection:
\begin{equation}
    \mathcal F_b:=\{\operatorname{Proj}(g)\mid g\in \mathcal G_b\} \subset \operatorname{Vec}(S^1),
\end{equation}
where the projection operator is defined as:
$$
\operatorname{Proj}(V)(x):=(I-xx^T)V(x), \text{ for all } x\in S^1.
$$
Parameterizing the circle by $\theta\in [0,2\pi)$, any tangent vector field $V$ on $S^1$ can be identified with a $2\pi$-periodic continuous function $f(\theta): \mathbb R\to \mathbb R $ by setting
$f(\theta)= V((\cos\theta, \sin \theta))\cdot (-\sin \theta, \cos \theta)$.
Under this identification, the control family $\mathcal F_b$ can be written as
\begin{equation}
    \bar{\mathcal F}_b:=\{\theta \to - w_1\sigma(a_1\cdot (\sin\theta, \cos \theta) +b_1)\sin\theta+w_2\sigma(a_2\cdot (\sin \theta, \cos \theta)+ b_2)\cos\theta\mid |w_1|, |w_2|, |a_1|, |a_2|\le 1\}.
\end{equation}
In what follows we work with the induced local norm $\|\cdot\|_{\operatorname{Id}}$ (defined in the general theory) on the tangent space of the smooth submanifold $\bar{\mathcal M}$ introduced earlier; the corresponding estimate at a general base point $\varphi\in\bar{\mathcal M}$ is then obtained by right-translation (composition with $\varphi^{-1}$), as summarized before.

For each $\phi\in \bar{\mathcal M}$, the tangent space $T_\phi \bar{\mathcal M}$ can be identified as the set of all $C^3$-vector fields on $S^1$, which can be identified as
\begin{equation}
    C_{\text{per}}^3([0,2\pi]):=\{f\in C^3(\mathbb R)\mid f(\theta+2\pi)=f(\theta), \forall \theta\in \mathbb R\}.
\end{equation}

Now, we first give an estimate of the local norm $\|\cdot\|_\phi$ at $\phi\in \bar{\mathcal M}$ by estimating the local norm at the identity map $\operatorname{Id}$.

First, we can decompose a given $f\in C_{\text{per}}^3([0,2\pi])$ in the tangent space as the following:
\begin{equation}
    f(\theta)= (f(\theta) \cos\theta)\cos\theta + (f(\theta)\sin\theta)\sin\theta.
\end{equation}
We notice that if $|a|\le 1$, we can write $a =(\cos \varphi, \sin \varphi)$ for some $\varphi\in [0,2\pi)$. Then, we have
\begin{equation}
    a\cdot x = \cos(\theta-\varphi),
\end{equation}
which gives a reparameterization of the control family.
We then approximate $f(\theta)\cos\theta$ by $\sum_{i=1}^N w_i \sigma(\cos(\theta -\varphi_i)+b_1)$ and $f(\theta)\sin\theta$ by $\sum_{i=1}^M w_i \sigma(\cos(\theta -\varphi_i)+b_2)$; our complexity measure is the sum of $|w_i|$.

For the subproblem, notice that any linear combination
\begin{equation}
    \sum_{k=1}^m w_k\sigma(\cos(\theta -\varphi_i)+b)
\end{equation}
can be identified as a convolution of $g_b(\theta):=\sigma(\cos\theta - b)$ with a discrete measure $\rho = \sum_{k=1}^m w_k \delta_{\varphi_k}$, where $\delta_\varphi$ is the Dirac measure at $\varphi$.

Based on this observation, we have the following proposition:
\begin{proposition}
    Suppose that there exists a function $\rho\in L^1([0,2\pi])$ such that
        \begin{equation}
        \int_0^{2\pi} g_b(x+t)\rho(t) dt = u(x).
    \end{equation}
    Then, we have that
    \begin{equation}
        \inf\{s\mid u\in \overline{\mathbf{CH}_s(\mathcal F_{b})}\}\le \int_{[0, 2\pi]} |\rho(t)| dt.
    \end{equation}
\end{proposition}
\begin{proof}
Fix $\varepsilon>0$. Choose a continuous $\rho_\varepsilon$ with $\|\rho-\rho_\varepsilon\|_{L^1([0,2\pi])}\le \varepsilon$.
Then
\begin{equation}
\Big\|\int_0^{2\pi} g_b(\cdot+t)\rho(t)\,dt-\int_0^{2\pi} g_b(\cdot+t)\rho_\varepsilon(t)\,dt\Big\|_{C([0,2\pi])}
\le \|g_b\|_{C([0,2\pi])}\,\varepsilon.
\end{equation}
Since $g_b$ and $\rho_\varepsilon$ are continuous on $[0,2\pi]$, a Riemann-sum approximation gives points $t_k$ and weights
$w_k:=\rho_\varepsilon(t_k)\Delta t$ such that
\begin{equation}
\Big\|\int_0^{2\pi} g_b(\cdot+t)\rho_\varepsilon(t)\,dt-\sum_{k=1}^m w_k\, g_b(\cdot+t_k)\Big\|_{C([0,2\pi])}
\le \varepsilon,
\qquad
\sum_{k=1}^m |w_k|
\le \int_0^{2\pi}|\rho_\varepsilon(t)|\,dt+\varepsilon.
\end{equation}
By construction, $\sum_{k=1}^m w_k\, g_b(\cdot+t_k)\in \mathbf{CH}_s(\mathcal F_b)$ with
$s:=\sum_k|w_k|$, hence $u$ lies in the $C([0,2\pi])$-closure of $\mathbf{CH}_{\|\rho\|_{L^1}+C\varepsilon}(\mathcal F_b)$ for a constant $C$ independent of $\varepsilon$.
Letting $\varepsilon\to 0$ yields the claim.
\end{proof}

Now, we will show that there exists some $b$ such that such $\rho$ exists for any $u\in C_{\text{per}}^3([0,2\pi])$.
\begin{proposition}
    Let $b = \cos \beta$, where $\beta$ is a quadratic irrational multiple of $\pi$, i.e. $\beta = \pi \alpha$ where $\alpha$ is a quadratic irrational number. Then, for any $f\in C_{\text{per}}^3([0,2\pi])$, there exists $\rho\in L^1([0,2\pi])$ such that
\begin{equation}
    \int_0^{2\pi} g_b(x+t)\rho(t) dt = f(x).
\end{equation}
Moreover, there exists a constanc $C_b$ depending only on $b$ such that
\begin{equation}
    \|\rho\|_{L^1([0,2\pi])} \le C_b \|f\|_{W^{3,2}([0,2\pi])}.
\end{equation}
\end{proposition}

\begin{proof}
    Notice that $g_b(x)$ is an even function, we have that its sine coefficients in the Fourier series are all zero.
Its cosine coefficients can be calculated as:
\begin{equation}
\begin{aligned}
a_n &= \frac{1}{2\pi} \int_0^{2\pi} \sigma(\cos x - \cos \beta) \cos(nx)\,dx\\
&= \frac{1}{2\pi} \int_{-\beta}^{\beta} (\cos x - \cos \beta)\cos(nx)\,dx\\
&= \frac{1}{2\pi}\left(\int_{-\beta}^{\beta}\cos x\cos(nx)\,dx-\cos\beta\int_{-\beta}^{\beta}\cos(nx)\,dx\right)\\
&= \frac{1}{2\pi}\left(\frac{1}{2}\int_{-\beta}^{\beta}\big(\cos((n-1)x)+\cos((n+1)x)\big)\,dx-\cos\beta\int_{-\beta}^{\beta}\cos(nx)\,dx\right)\\
&= \frac{1}{2\pi}\left(\frac{1}{2}\left[\frac{\sin((n-1)x)}{n-1}+\frac{\sin((n+1)x)}{n+1}\right]_{-\beta}^{\beta}-\cos\beta\left[\frac{\sin(nx)}{n}\right]_{-\beta}^{\beta}\right)\\
&= \frac{1}{2\pi}\left(\frac{\sin((n-1)\beta)}{n-1}+\frac{\sin((n+1)\beta)}{n+1}-\frac{2\cos\beta\,\sin(n\beta)}{n}\right)\\
&= \frac{1}{2\pi}\left(\frac{\sin((n-1)\beta)}{n-1}+\frac{\sin((n+1)\beta)}{n+1}-\frac{\sin((n+1)\beta)+\sin((n-1)\beta)}{n}\right)\\
&= \frac{1}{2\pi}\left(\sin((n-1)\beta)\!\left(\frac{1}{n-1}-\frac{1}{n}\right)+\sin((n+1)\beta)\!\left(\frac{1}{n+1}-\frac{1}{n}\right)\right)\\
&=\begin{cases}
\displaystyle \frac{1}{2\pi}\!\left(\frac{\sin((n-1)\beta)}{n(n-1)}-\frac{\sin((n+1)\beta)}{n(n+1)}\right), & n\ge 2,\\[8pt]
\displaystyle \frac{1}{2\pi}\!\left(\beta-\frac{\sin(2\beta)}{2}\right), & n=1.
\end{cases}
\end{aligned}
\end{equation}
Notice that
\begin{equation}
    \frac{1}{2\pi}\!\left(\frac{\sin((n-1)\beta)}{n(n-1)}-\frac{\sin((n+1)\beta)}{n(n+1)}\right) = -\frac{2}{n^2}\sin\beta \cos n\beta +\mathcal O(\frac{1}{n^3}).
\end{equation}
By the continued fraction approximation of quadratic irrational numbers, we have that
\begin{equation}
    \inf_n{|n\cos n\beta|}>0.
\end{equation}
This indicates that there exists some constant $C_{1,b}>0$, only depending on $b$, such that
\begin{equation}
    |\hat g_b(n)|\ge \frac{C_{1,b}}{n^3}, \text{ for all integer } n\neq 0.
\end{equation}

If we assume $f$ to be $C^5$-smooth(can be weaker, say $f\in H^4([0,2\pi]))$, we have that
\begin{equation}
    |\hat f(n)|\le \frac{C}{n^5}, \text{ for some } C>0, \text{ for all integer } n \neq 0.
\end{equation}
Therefore, the Fourier coefficients of $\rho$ satisfies
\begin{equation}
    \frac{\hat f(n)}{\hat g_b(n)} \le \frac{C}{C_{1,b} n^2}, \text{ for all integer } n.
\end{equation}
This indicates that there exists a function $\rho\in H^1([0,2\pi])$ with Fourier coefficients $\hat \rho(n) = \frac{\hat f(n)}{\hat g_b(n)}$ such that:
\begin{equation}
    f(x) = \int_0^{2\pi} g_b(x+t)\rho(t) dt.
\end{equation}
Moreover, we have that
\begin{equation}
    \begin{aligned}
           \|\rho\|_{L^1([0,2\pi])} \le \sqrt{2\pi}\|\rho\|_{L^2([0,2\pi])} &\le {\sqrt{2\pi}}\left((\frac{C}{C_{1,b}})^2\sum_{n\in \mathbb Z\setminus \{0\}} (n^3|\hat f(n)|)^2+ (\frac{\hat f(0)}{\hat g(0)})^2\right)^{\frac{1}{2}} \\
           &\le C_{b,2} \|f^{(3)}\|_{L^2([0,2\pi])} + C_{b,3}\|f\|_{C([0,2\pi])}\\,
    \end{aligned}
\end{equation}
here $C_{b,2}$ and $C_{b,3}$ are two constants that only depend on $b$.
\end{proof}

Combining the above two propositions, we have an estimate of the local norm:
\begin{equation}
    \|u\|_{\varphi} \le C_{b,2}\|(u\circ \varphi^{-1})^{(3)}\|_{L^2{[0,2\pi]}}+C_{b,3}\|u\circ \varphi^{-1}\|_{C([0,2\pi])}.
\end{equation}

Now, we provide a global distance estimation between two diffeomorphisms $\phi_1, \phi_2\in \mathcal M$.

For any $\psi\in \bar{\mathcal M}$, there exists a natural path (homotopy) from $\operatorname{Id}$ to $\psi$. Specifically,
consider the cover
\begin{equation}
    p:\mathbb R \to S^1, \theta\to e^{i\theta}.
\end{equation}
For any $\varphi\in \bar{\mathcal M}$,
there exists a lift $\tilde \psi\in \operatorname{Diff}(\mathbb R)$ which is orientation preserving such that $\tilde \psi(x+2\pi) = \tilde \psi(x)+2\pi$ and $p\circ \tilde \psi = \psi\circ p$.
Therefore, we identify the $C^5$-diffeomorphism $\psi$ on $S^1$ as the set of functions:
\begin{equation}
    \begin{aligned}
    \mathcal N: =\{\psi \in C^5([0,2\pi]) &\mid \psi^\prime(x)> 0 \text{ for all } x\in (0,2\pi), \\
    &\psi(0)+2\pi = \psi(2\pi), \psi^{(i)}(0) = \psi^{(i)}(2\pi) \text{ for } i=1,2,3,4, 5\}
    \end{aligned}
\end{equation}

For any $\psi_1, \psi_2 \in \mathcal N$,
\begin{equation}
    \psi_t :=(1-t)\psi_1 + t \psi_2
\end{equation}
gives a homotopy from $\psi_1$ to $\psi_2$ on $\mathcal N$. Therefore, we can estimate the global distance $d_{\mathcal F}(\psi_1, \psi_2)$ by estimating the length of this path:
\begin{equation}
    \begin{aligned}
    \int_0^{1} \|\frac{\partial}{\partial t}\psi_t\|_{\psi_t} dt &= \int_0^1 \|(\psi_2-\psi_1)\circ \psi_t^{-1}\|_{\operatorname{Id}} dt\\
    &\le \left(\int_0^1 C_{b,2}\|((\psi_2-\psi_1)\circ \psi_t^{-1})^{(3)}\|_{L^2{[0,2\pi]}}+C_{b,3}\|(\psi_2-\psi_1)\circ \psi_t^{-1}\|_{C([0,2\pi])}dt. \right)
    \end{aligned}
\end{equation}
Denote $u:=\psi_2-\psi_1$. For the first term, we have
\begin{equation}
\int_0^1 C_{b,2}\|(u\circ \psi_t^{-1})^{(3)}\|_{L^2{[0,2\pi]}}dt \le C_{b,2}\left(\int_{0}^{1}\|(u\circ \psi_t^{-1})^{(3)}\|_{L^2{[0,2\pi]}}^2 dt\right)^{1/2}
\end{equation}

Now, for given $t\in [0,1]$, denote
\begin{equation}
    y = \psi_t^{-1}(x), z := \psi_2(y).
\end{equation}

Define
\begin{equation}
\rho(z)\ :=\ \frac{\psi_1'}{\psi_2'}(\psi_2^{-1}(z))\ >0,\qquad
D:=\frac{d}{dz}, \qquad \rho_t = t+(1-t)\rho.
\end{equation}
Then, we have
\begin{equation}
   \frac{d}{dx} = \frac{1}{\psi_t'(y)}\frac{d}{dy} = \frac{\psi_2^\prime}{\psi_t}\frac{d}{d z} = \frac{1}{h} D
\end{equation}

We also have:
\begin{equation}
Du = 1-\rho,\qquad D^2 u = -D\rho,\qquad D^3 u = -D^2\rho.
\end{equation}
Then, we have

\begin{equation}
    \frac{d}{dx}(u\circ\psi_t^{-1}) =\frac{1}{h} Du = \frac{1-\rho}{h}
\end{equation}
\begin{equation}
    \frac{d^2}{dx^2}(u\circ\psi_t^{-1})''=\frac{1}{h}D\left(\frac{1-\rho}{h}\right)
=\frac{1}{h}\left(\frac{-D\rho}{h}+\frac{-(1-\rho)Dh}{h^2}\right).
\end{equation}
Notice that
\begin{equation}
    h+(1-t)(1-\rho)=\big(t+(1-t)\rho\big)+(1-t)(1-\rho)=1,
\end{equation}
we have
\begin{equation}
    \frac{d^2}{dx^2}(u\circ\psi_t^{-1})=-\frac{D\rho}{h^3}
\end{equation}
Finally, we have
\begin{equation}
    \begin{aligned}
\frac{d^3}{dx^3}(u\circ\psi_t^{-1})&=\frac{1}{h}D\left(-\frac{D\rho}{h^3}\right)
=-\frac{1}{h}\left(\frac{D^2\rho}{h^3}-\frac{3,Dh}{h^4}D\rho\right)\\
&= -\frac{D^2\rho}{h^4}+3(1-t)\frac{(D\rho)^2}{h^5}\\
& = -\frac{D^2\rho}{(t+(1-t)\rho)^4}+3(1-t)\frac{(D\rho)^2}{(t+(1-t)\rho)^5}
    \end{aligned}
\end{equation}

Therefore, by switching the order of integration, we have
\begin{align}
\int_0^1 \big\|(u\circ\psi_t^{-1})'''_x\big\|_{L^2}^2\,dt
&= \int_0^{2\pi} \!\Big[A(\rho)\,(D^2\rho)^2 + B(\rho)\,(D\rho)^2(D^2\rho)
+ C(\rho)\,(D\rho)^4\Big]\,dx, \label{eq:third}
\end{align}
where
\[
A(\rho)=\frac{67\rho^6+67\rho^5+67\rho^4+67\rho^3+172\rho^2-80\rho+60}{420\,\rho^7},
\]
\[
B(\rho)=-\,\frac{17\rho^6+34\rho^5+51\rho^4+68\rho^3+295\rho^2-150\rho+105}{140\,\rho^8},
\]
\[
C(\rho)=\frac{3\big(\rho^5+3\rho^4+6\rho^3+10\rho^2+15\rho+21\big)}{56\,\rho^8}.
\]
For the second term, since $\psi_t$ is a diffeomorphism, it directly follows that
\begin{equation}
    \int_0^1 \|(\psi_2-\psi_1)\circ \psi_t^{-1}\|_{C([0,2\pi])}dt = \int_0^1 \|(\psi_2-\psi_1)\|_{C([0,2\pi])}dt = \|\psi_1-\psi_2\|_{C([0,2\pi])}
\end{equation}
Combine these, we have the estimation:
\begin{equation}
    d_{\mathcal F}(\psi_1, \psi_2) \le C_{b,2}\int_0^{2\pi} \!\Big[A(\rho)\,(D^2\rho)^2 + B(\rho)\,(D\rho)^2(D^2\rho)
+ C(\rho)\,(D\rho)^4\Big]\,dx + C_{b,3} \|\psi_1-\psi_2\|_{C([0,2\pi])}
\end{equation}
That is, the complexity is mainly related to the regularity of $\rho$.

% \section{Conclusion}
% In this paper, we introduced a geometric framework to analyze
% the approximation rate problem for deep learning in a dynamical systems setting.
% In particular, we showed that the time-cost (corresponding to depth requirement) of
% approximating maps can be viewed as geodesic distances over infinite dimensional
% sub-Finsler manifolds.
% This viewpoint departs in an interesting way from classical
% approximation theory over linear spaces.
% The framework provides a unified perspective and can be potentially used to study a wide
% range of problems involving function (or distribution) approximation
% by flows, which has applications both in modern deep learning and
% other areas of mathematical sciences.
% Important future directions include the investigation of learning architectures
% inspired by the current theoretical analysis, such as the choice of reference maps
% different from the identity.